\documentclass{article}

\usepackage{PRIMEarxiv}

\usepackage[utf8]{inputenc} 
\usepackage[T1]{fontenc}    
\usepackage{hyperref}       
\usepackage{url}            
\usepackage{booktabs}       
\usepackage{amsfonts}       
\usepackage{nicefrac}       
\usepackage{microtype}      
\usepackage{lipsum}
\usepackage{fancyhdr}       
\usepackage{graphicx}       
\graphicspath{{media/}}

\usepackage{graphicx}
\usepackage{subfigure}
\usepackage{booktabs} 

\usepackage{hyperref}

\usepackage{amsmath}
\usepackage{amssymb}
\usepackage{mathtools}
\usepackage{amsthm}

\usepackage{mathrsfs} 
\usepackage{multirow}
\usepackage{arydshln}
\usepackage{bm}

\usepackage{microtype}
\usepackage{graphicx}
\usepackage{subfigure}
\usepackage{booktabs} 

\usepackage{amsmath}
\usepackage{amssymb}
\usepackage{mathtools}
\usepackage{amsthm}

\usepackage{amsfonts}
\usepackage{mathrsfs} 
\usepackage{multirow}
\usepackage{arydshln}
\usepackage{bm}

\usepackage{natbib}
\theoremstyle{plain}
\newtheorem{theorem}{Theorem}[section]

\newtheorem{lemma}[theorem]{Lemma}

\theoremstyle{definition}
\newtheorem{definition}[theorem]{Definition}

\theoremstyle{remark}
\newtheorem{remark}[theorem]{Remark}

\newtheorem{restatetheoreminner}{\bf Theorem}
\newenvironment{restatetheorem}[1]{%
  \renewcommand\therestatetheoreminner{#1}%
  \restatetheoreminner
}{\endrestatetheoreminner}

\title{Towards the Generalization of Multi-view Learning: An Information-theoretical Analysis}

\author{
  Wen Wen, Tieliang Gong\thanks{\textit{Corresponding to}: Tieliang Gong} , Yuxin Dong, Weizhan Zhang\\
  School of Computer Science of Technology \\
  Xi'an Jiaotong University\\
  \texttt{\{wen190329,adidasgtl,yxdong9805\}@gmail.com}  \\
  \texttt{zhangwzh@xjtu.edu.cn}
   \And
  Shujian  Yu \\
  Department of Artificial Intelligence \\
  Vrije Universiteit Amsterdam  \\
  \texttt{yusj9011@gmail.com} \\
}

\begin{document}
\maketitle

\begin{abstract}
    Multiview learning has drawn widespread attention for its efficacy in leveraging cross-view consensus and complementarity information to achieve a comprehensive representation of data. While multi-view learning has undergone vigorous development and achieved remarkable success, the theoretical understanding of its generalization behavior remains elusive. This paper aims to bridge this gap by developing information-theoretic generalization bounds for multi-view learning, with a particular focus on multi-view reconstruction and classification tasks. Our bounds underscore the importance of capturing both consensus and complementary information from multiple different views to achieve maximally disentangled representations. These results also indicate that applying the multi-view information bottleneck regularizer is beneficial for satisfactory generalization performance. Additionally, we derive novel data-dependent bounds under both leave-one-out and supersample settings, yielding computational tractable and tighter bounds. In the interpolating regime, we further establish the fast-rate bound for multi-view learning, exhibiting a faster convergence rate compared to conventional square-root bounds. Numerical results indicate a strong correlation between the true generalization gap and the derived bounds across various learning scenarios.
    
\end{abstract}

\section{Introduction}
\label{submission}
In most scientific data analysis scenarios, data collected from diverse domains and different sensors exhibit heterogeneous properties while preserving underlying connections. For example, (1) a piece of text can express the same semantics and sentiment in multiple different languages; (2) the user's interest can be reflected in the text posted, images uploaded, and videos viewed; (3) animals perceive potential dangers in their surroundings through various senses such as sight, hearing, and smell. All of these reflect different perspectives of the data, collectively referred to as multi-view data. Extracting consensus and complementarity information from multiple views to achieve a comprehensive representation of multi-view data, has stimulated research interest across various fields and led to the development of multi-view learning \cite{hamdi2021mvtn,fan2022mv,fu2022geo,hong20233d}. 

While various methodologies have emerged in multi-view learning, predominantly encompassing canonical correlation analysis (CCA)-based approaches \cite{gao2020cross,chen2022k,shu2022d} and engineering-driven techniques \cite{xu2021multi,bai2023hvae}, these methods suffer from a critical limitation. Specifically, their emphasis on maximizing cross-view consensus information often comes at the expense of view-specific, task-relevant information, thereby potentially compromising downstream performance \cite{liang2024factorized}. Recent significant efforts have been dedicated to leveraging diverse information-theoretic techniques to precisely capture both view-common and view-unique components from multiple views \cite{wang2019deep,federici2020learning,wang2023self,cui2024novel,zhang2024discovering}, thereby yielding maximally disentangled representation and improving generalization ability. For instance, \cite{kleinman2024gacs} and \cite{zhang2024discovering} introduce the notion of Gács-Körner common information \citep{gacs1973common} and utilize total correlation between consensus and complementarity information to extract mutually independent cross-view common and unique components. The works of \cite{federici2020learning} and \cite{cui2024novel} generalizes the information bottleneck (IB) approach \cite{tishby2000information} to multi-view learning, achieving a favorable trade-off between prediction performance and complexity in learning compact multi-view representation. Despite empirical success, the generalization properties of multi-view learning, particularly regarding consensus and complementarity information across views, still remain poorly understood.

This work provides a comprehensive information-theoretic generalization analysis for multi-view learning by establishing high-probability generalization bounds. Our results offer theoretical explanations for the effectiveness of existing multi-view learning methods in achieving sufficiently compact representations and excellent generalization performance, as enumerated in Table \ref{datasettable}. The main contributions of this paper are summarized as follows.

\begin{itemize}
    \item We establish information-theoretic bounds for multi-view learning from the perspective of both reconstruction and classification tasks. Our bounds, expressed in terms of information measures that integrate both view-common and view-unique components, demonstrate that extracting consensus and complementarity information across multiple views within an information-theoretic framework facilitates the maximally disentangled representations and leads to smaller generalization errors. Furthermore, the derived bounds justify the effectiveness of the multi-view information bottleneck regularizer in balancing representation capabilities and generalization performance.

    \item We develop novel data-dependent bounds by employing one-dimensional variables in both leave-one-out and supersample settings, facilitating the attainment of computational feasibility. Additionally, we derive the fast-rate bound for multi-view learning by leveraging the weighted generalization error, improving the scaling rate from $1/\sqrt{nm}$ to $1/nm$ under the interpolating regime, where $n$ and $m$ denote the number of multi-view samples and the number of views, respectively.
    \item Empirical observations on both synthetic and real-world datasets validate the close agreement between the true generalization error and the derived bounds, indicating the effectiveness of our results in capturing the generalization of multi-view learning.
    
\end{itemize}

\begin{table*}[t]
    \centering
    \caption{Theoretical supports for diverse multi-view learning approaches based on information theory.}
    \vskip 0.15in
    \renewcommand\arraystretch{1.2}
    \resizebox{\textwidth}{!}{
    \begin{tabular}{ccccc}
    \toprule
    \textbf{Learning Task}  & \textbf{Reference} & \textbf{Information-theoretic Tool}  & \textbf{Information Measure}   & \textbf{Generalization Analysis} \\
    \midrule
    \multirow{3}{*}{Multi-view Reconstruction}  &\cite{kleinman2024gacs} & Mutual Information & $\max \frac{1}{2} \sum_{j=1}^{2}I(X^{(j)};C)$   &  Theorem \ref{theorem2}\\
    & \multirow{2}{*}{\cite{zhang2024discovering}} & GK Common information & \multirow{2}{*}{$H(C)-TC(C,U^{(1)},\ldots,U^{(m)})$}    & \multirow{2}{*}{Theorem \ref{theorem1}} \\
    & &  Total Correlation&  &  \\
    \cdashline{1-5}
    \multirow{3}{*}{Multi-view Classification} &\cite{wang2019deep} & \multirow{3}{*}{Multi-view IB}  & $I(Y;Z)-I(X;Z)$   & \multirow{3}{*}{Theorems \ref{theorem3}-\ref{theorem7}}   \\
    &\cite{federici2020learning}  & & $I(X;Z|Y) + I(Y;Z)$   &  \\
    &\cite{wang2023self} & & $I(X;Y)+I(X;Y|Z)$   &  \\
    \bottomrule
    \end{tabular}}
    \vskip 0.15in
    
    \label{datasettable}
    \end{table*}

\section{Related work}
\paragraph{Multi-view Learning with Information Theory.} Recently, multi-view learning methods driven by information theory have achieved remarkable success in various application areas \cite{oord2018representation,wang2019deep,federici2020learning,wang2023self}. \cite{oord2018representation} propose a universal unsupervised learning approach based on mutual information to learn compact representations from high-dimensional multi-view data. \cite{wang2019deep} propose a deep multi-view information bottleneck framework to disentangle complex dependencies across different views and extract cross-view common and unique information. Recent significant advancements have been made by \cite{federici2020learning} and \cite{wang2023self}, who develop information bottleneck approaches for multi-view learning under unsupervised and self-supervised settings. Their methods effectively capture compact representations by retaining sufficient information from all views while removing superfluous information from individual views. To obtain maximally disentangled representations, \cite{kleinman2024gacs} generalize the Gács-Körner common information \cite{gacs1973common} to the stochastic functions, efficiently inferring the common and unique components of the cross-representation via parameterized optimization. \cite{zhang2024discovering} introduce the multi-view Gács-Körner common information and propose novel multi-view learning framework capable of precisely extracting common and unique features from multiple views, achieving smaller reconstruction error and improving performance on downstream classification tasks. Despite the superior performance, there is a lack of theoretical understanding of the generalization behavior of multi-view learning.
\paragraph{Information-Theoretic Generalization Analysis.} Information theory has become a powerful tool for depicting and analyzing the generalization properties of learning algorithms in recent years \citep{xu2017information,steinke2020reasoning,wang2023tighter,dongrethinking}. \cite{xu2017information} initially provide an information-theoretic understanding of the generalization ability by relating the generalization error to the input-output mutual information, implying the balance between the data fit and the generalization. This framework has been subsequently enhanced and expanded through diverse approaches, such as the random subsets \cite{bu2020tightening,dongrethinking}, conditional information measures \cite{steinke2020reasoning}, gradient variance analysis \cite{negrea2019information,wang2023tighter}, etc.  Along the lines of previous work, significant effort has recently been devoted to revealing the relationship between information bottlenecks and generalization \citep{shwartz2018representation,saxe2019information,federici2020learning,tezuka2021information}.  \cite{shwartz2018representation} provide an analytic bound on the mutual information between representation compression. However, their bounds do not explain the success of the information bottleneck principle and are invalid for data-dependent algorithms \cite{hafez2020sample}. While \cite{hafez2020sample} move towards this by proving an input compression bound, this is still far from being a valid sample complexity bound. \cite{kawaguchi2023does} establish the information-theoretic bounds for representation learning through mutual information, demenstrating that optimizing the information bottleneck is an effective way to control generalization errors. Remarkably, their derived bounds only apply to the case of a single-view data as input. While the aforementioned work bears its own limitations, it still provides valuable insights for analyzing the generalization of multi-view learning.


\section{Preliminaries}
We denote random variables and their specific realizations by capital letters and lower-case letters, respectively. Denote by $H(X)=\mathbb{E}_{p(X)}[-\log p(X)]$ the Shannon's entropy of random variable $X$ and $H_{\alpha}(X)=\frac{1}{1-\alpha}\log \int_{\mathcal{X}} p^{\alpha}(X)dX$ the R\'{e}nyi's $\alpha$-order entropy of $X$, where the limit case $\alpha\rightarrow 1$ recovers Shannon's entropy. The conditional entropy of random variable $X$ given $Y$ is denoted by $H(X|Y)$. Let $I(X;Z)$ be the mutual information between random variables $X$ and $Z$. The mutual information $I(X;Z)$ can also be expressed by $I(X;Z) = \mathbb{E}_{p(X)}\big[KL\big(p(Z|X)\Vert p(Z)\big)\big]$, where $KL$ denotes Kullback Leibler divergence. The corresponding conditional mutual information given $Y$ is denoted by $I(X;Z|Y)$. For simplicity, we write $X_y$ to denote $X$ conditional on $Y=y$. We define the multivariate mutual information, known as total correlation, $\mathrm{TC}(X_1,\ldots,X_n) = KL(p(X_1,\ldots,X_n)\Vert \prod_{i=1}^n p(X_i))$, where $p(X_i)$ is the marginal distributions of the components of $X$. We use $\circ$ to represent the composition of functions.


\paragraph{Multiview Learning.} Multi-view learning aims to leverage the consensus and complementarity information extracted from multiple views to improve performance. Let multi-view data $X$ comprise $m$ different views, denoted as $X=(X^{(1)},\ldots,X^{(m)})\in\mathcal{X}\subseteq \mathbb{R}^{m\times d}$, where $\mathcal{X}$ is the feature space and each view $X^{(j)}$ is the $d$-dimensional feature vector. Let $\Phi:\mathbb{R}^{m\times d}\mapsto \mathbb{R}^{m\times d}$ represent the hypothesis space of representation functions. Given the hypothesis function $ \phi^{(j)}=\{(\phi_c^{(j)}, \phi_u^{(j)})\}:\mathbb{R}^d \mapsto \mathbb{R}^d$ of the $j$-th view, the extracted view-common and view-unique features of single views $X^{(j)}$, denoted $C$ and $U^{(j)}$, can be expressed by
\begin{equation*}
    (C,U^{(j)}):= (\phi_c^{(j)}(X^{(j)}),\phi_u^{(j)}(X^{(j)}))= \phi^{(j)}(X^{(j)}).
\end{equation*}

To precisely quantify common information across different views, extensive work \citep{salehifar2021quantizer,tsur2024several} has introduced the concept of Gács-Körner common information between two views \cite{gacs1973common}:
\begin{definition}\label{defi1}\cite{gacs1973common}
    Let $X=(X^{(1)},X^{(2)})$ be a two-view data. Given deterministic functions $g_1$ and $g_2$, the Gács-Körner common information is defined by
\begin{align}
    C_{GK}(X):=\max_{C} H(C), \label{gk1}
\end{align}
subjected to $C=g_{1}(X^{(1)}) = g_{2}(X^{(2)})$.
\end{definition}
Definition \ref{defi1} quantifies the maximum amount of information retained in common parts between two views through the entropy $H(C)$ of the variable $C$. This measurement has been expanded in \citep{zhang2024discovering} from two views to $m$ different views ($m\geq 2$). It is noteworthy that such a common information measure typically encounters computational infeasibility when dealing with high-dimensional data. To address this issue, we introduce a universal measure of common information among multiple views:
\begin{definition}\label{defi2}
    Let $X=\{X^{(j)}\}_{j=1}^{m}$ be a multi-view data with $m$ views. Given stochastic functions $f_1,\ldots,f_m$, the multi-view common information is defined by
\begin{align*}
    \widetilde{C}_{GK}(X^{(1)},\ldots,X^{(m)}):=\max_{C} I(X^{(j)};C) \quad j\in[m],
\end{align*}
subjected to $C=f_1(X^{(1)}) = \cdots = f_m(X^{(m)})$.
\end{definition}
It can be easily proven that for arbitrary $i,j\in[m]$, $I(X^{(i)};C)=I(X^{(j)};C)$, thereby leading to an equivalent formula of Definition \ref{defi2}: $\max_C \frac{1}{m} \sum_{j\in[m]} I(X^{(j)};C) = \max_C I(X^{(j)};C)$. If functions $f_1,\ldots,f_m$ are deterministic, where $H(C|X^{(j)})=0$ and $I(X^{(j)};C)=H(C)$, the measure defined above would degenerate into the multi-view Gács-Körner common information. In contrast, Definition \ref{defi2} not only applies broadly to various multi-view learning methods but also circumvents the infeasibility of computing the Gács-Körner common information on high-dimensional data \cite{kleinman2024gacs}.

Building upon the above definitions and exploiting multi-view representation function $\phi=\{\phi^{(j)}\}_{j=1}^m$, one could learn a disentangled representation $Z$ from the multi-view data $X=\{X^{(j)}\}_{j=1}^{m}$, expressed by
\begin{equation*}
    Z:=(C,U^{(1)},\ldots, U^{(m)})=\phi(X).
\end{equation*}

\paragraph{Generalization Error.} The generalization error serves as an effective indication of the generalization capability of hypothesis functions, defined as the difference between the population risk and the empirical risk. In this paper, we aim to analyze the generalization of multi-view learning by establishing information-theoretic generalization error bounds. In particular, we focus on the generalization error of the output hypothesis for both multi-view reconstruction and classification tasks, which evaluates the invariance of the learned representation to the original data and its sufficiency for class prediction.

Assume that we have $n$ labeled multi-view samples, denoted as $S=\big\{(x_i,y_i)\big\}^n_{i=1}$, i.i.d. drawn from a probability distribution over $\mathcal{X}\times\mathcal{Y}$, where $\mathcal{Y}$ denotes the class label space, each $x_i=(x_i^{(1)},\ldots,x_i^{(m)})$ consists of $m$ views and $\mathcal{Y}$, and $y_i$ denotes its corresponding label. Given a loss function $\ell:\mathbb{R}^d \times \mathbb{R}^d \mapsto\mathbb{R}_+$ and a decoding function $\psi:\mathbb{R}^d \mapsto \mathbb{R}^d$, the empirical risk of the hypothesis $\phi=\{\phi^{(j)}\}_{j=1}^m$ for multi-view reconstruction is defined by 
\begin{equation*}
    \widehat{L}_{rec} := \frac{1}{n}\sum^n_{i=1} \ell_{\mathrm{avg}}(x_i ; \psi,\phi),
\end{equation*}
where $\ell_{\mathrm{avg}}(x_i ; \psi,\phi) = \frac{1}{m}\sum_{j=1}^{m}\ell(\psi\circ\phi^{(j)}(x_i^{(j)}), x_i^{(j)})$ represents the average of the sum of the loss between each single view $x_i^{(j)}$ and its reconstructed view $\psi\circ\phi^{(j)}(x_i^{(j)})$ for $i$-th multi-view sample $x_i$. The empirical risk $\widehat{L}_{rec}$ measures the quality of the learning hypothesis $\phi$ for capturing latent multi-view representations to precisely reconstruct the original multi-view data. The corresponding population risk is defined by 
\begin{equation*}
    L_{rec} := \mathbb{E}_{X\sim\mathcal{X}} [\ell_{\mathrm{avg}}(X ; \psi,\phi)],
\end{equation*}
which evaluates the generalization ability of $\phi$ on unseen multi-view data $X$. We define the generalization error for multi-view reconstruction by 
\begin{equation*}
    \overline{\mathrm{gen}}_{rec} = L_{rec} - \widehat{L}_{rec}. 
\end{equation*}

For classification tasks, let the decoding function $\hat{\psi}:\mathbb{R}^d\mapsto\mathbb{R}$. Given the loss function $\hat{\ell}:\mathbb{R}\times \mathbb{R}\mapsto \mathbb{R}_+$, the empirical risk of $\phi=\{\phi^{(j)}\}_{j=1}^m$ for the class prediction is defined by 
\begin{equation*}
    \widehat{L}_{cls} := \frac{1}{n}\sum^n_{i=1} \hat{\ell}_{\mathrm{avg}}\big( (x_i,y_i) ; \hat{\psi},\phi \big),
\end{equation*}
where $\hat{\ell}_{\mathrm{avg}}((x_i,y_i) ; \hat{\psi},\phi)= \frac{1}{m}\sum_{j=1}^{m}\hat{\ell}(\hat{\psi}\circ\phi^{(j)}(x_i^{(j)}), y_i) $ represents the average of the sum of the loss between the predicted label $\hat{\psi}\circ\phi^{(j)}(x_i^{(j)})$ of each single view and its true label $y_i$ for multi-view data $x_i$. Similarly, the population risk is defined by
\begin{equation*}
    L_{cls} := \mathbb{E}_{(X,Y)\sim\mathcal{X}\times\mathcal{Y}} [\hat{\ell}_{\mathrm{avg}}((X,Y) ; \hat{\psi},\phi)].
\end{equation*}
The corresponding generalization error for the multi-view classification  is defined by 
\begin{equation*}
    \overline{\mathrm{gen}}_{cls} = L_{cls} - \widehat{L}_{cls}. 
\end{equation*}
Additionally, let $R_x=\sup_{X\in\mathcal{X}}\ell_{\mathrm{avg}}(X ; \psi,\phi)$ and $R_{x,y}=\sup_{(X,Y)\in\mathcal{X}\times\mathcal{Y}}\hat{\ell}_{\mathrm{avg}}((X,Y) ; \hat{\psi},\phi)$ represent the maximum attainable losses, and $R^s_x=\sup_{i\in[n]}\ell_{\mathrm{avg}}(x_i ; \psi,\phi)$ and $R^{s}_{x,y}=\sup_{i\in[n]}\hat{\ell}_{\mathrm{avg}}\big( (x_i,y_i) ; \hat{\psi},\phi \big)$ represent the maximum samplewise losses.

Drawing on the significant advancements made in \cite{steinke2020reasoning,harutyunyan2021information,rammal2022leave}, which establish computationally tractable bounds under various sampling regimes, we extend these methodologies to multi-view learning scenarios:

\paragraph{Leave-One-Out (LOO) Setting.} The LOO setting is introduced by \cite{rammal2022leave} for generalization analysis. Let $\tilde{S}_l=\{(x_i,y_i)\}_{i=1}^{n+1}$ be a dataset of $n+1$ i.i.d. samples. We denote $U\sim\mathrm{Unif}([n+1])$ as a uniform random variable for selecting the single test sample from $\tilde{S}_l$. Then, the training and test datasets are constructed as $\tilde{S}^l_{\mathrm{train}}=\tilde{S}_l\backslash (x_U,y_U)$ and $\tilde{S}^l_{\mathrm{test}} = \{(x_U,y_U)\}$. The generalization ability of the hypothesis $\phi$ is measured by the LOO validation error $\Delta_{\mathrm{loo}}=\hat{\ell}_{\mathrm{avg}}\big( (x_U,y_U) ; \hat{\psi},\phi \big)-\frac{1}{n}\sum_{i\neq U}\hat{\ell}_{\mathrm{avg}}\big( (x_i,y_i) ; \hat{\psi},\phi \big)$.

\paragraph{Supersample Setting.} The supersample setting is initially explored by \cite{steinke2020reasoning} for generalization analysis. Let  $\tilde{S}_{s}=\{(x_{i,0},y_{i,0}),(x_{i,1},y_{i,1})\}_{i=1}^n$ be the supersample dataset consisting of $2\times n$ i.i.d. samples. The random variable $\tilde{U}=\{\tilde{U}_i\}_{i=1}^n\sim\mathrm{Unif}(\{0,1\}^n)$ is used to split training and test samples. The training and test datasets are then constructed as $\tilde{S}^{s}_{\mathrm{train}}=\{(x_{i,\tilde{U}_i},y_{i,\tilde{U}_i})\}_{i=1}^n$ and $\tilde{S}^{s}_{\mathrm{test}}=\{(x_{i,1-\tilde{U}_i},y_{i,1-\tilde{U}_i})\}_{i=1}^n$, respectively. The generalization ability of $\phi$ is measured by $\Delta_{\mathrm{sup}}=\frac{1}{n}\sum_{i=1}^n\hat{\ell}_{\mathrm{avg}}\big( (x_{i,1-\tilde{U}_i},y_{i,1-\tilde{U}_i}) ; \hat{\psi},\phi \big)-\frac{1}{n}\sum_{i=1}^n \hat{\ell}_{\mathrm{avg}}\big( (x_{i,\tilde{U}_i},y_{i,\tilde{U}_i}) ; \hat{\psi},\phi \big)$.

\paragraph{Presumption.} Following the previous work \cite{kawaguchi2023does}, we focus on hypothesis spaces with finite cardinality. This allows for finite information measures, thereby yielding non-vacuous and non-trivial generalization guarantees. It is noteworthy that such a constraint is natural for digital computers and has been widely used in information-theoretic generalization analysis \cite{xu2017information,kawaguchi2023does}.

\section{Information-theoretic Bounds for Multi-view Learning}
In this section, we develop the high-probability generalization bounds of representation functions for multi-view reconstruction and classification tasks through the novel usage of the typical subset. We separately analyze the generalization properties of deterministic and stochastic representation functions, thereby providing generalization guarantees for various multi-view learning paradigms \cite{wang2023self,kleinman2024gacs,zhang2024discovering}. Additionally, we develop novel data-dependent bounds for multi-view learning under the LOO and supersample settings, and the fast-rate bound in the interpolating regime. Please refer to the Appendix for complete proofs.

\subsection{Generalization Bounds for Multi-view Reconstruction} \label{Section4.1}

We start with deterministic the representation function $\phi=\{\phi^{(j)}\}_{j=1}^m$, and establish the following generalization bound via information measures incorporating view-common and view-unique information:
\begin{theorem}\label{theorem1}
    For any $\gamma>0$ and $\delta>0$, with probability at least $1-\delta$:
    \begin{align*}  
        \overline{\mathrm{gen}}_{rec}  \leq \mathcal{K}_1\sqrt{\frac{H(C)+\sum_{j=1}^m H(U^{(j)})  +  \mathcal{K}_2}{nm} }   +\frac{\mathcal{K}_3}{\sqrt{nm}},
    \end{align*}
where $\mathcal{K}_1, \mathcal{K}_2, \mathcal{K}_3$ are constants of order $\widetilde{\mathcal{O}}(1)$ as $n,m\rightarrow \infty$, specifically defined in Appendix \ref{Proof-Thm1}. 
\end{theorem}
\begin{remark}
    Theorem \ref{theorem1} characterizes the generalization behavior of deterministic multi-view representation functions through the entropy $H(C)$ of the view-common component and $\sum_{j=1}^m H(U^{(j)})$ of the view-unique components. It elucidates how maximally disentangled representation, captured from both common and unique information across multiple views, faciliates excellent generalization performance, as evidenced by recent studies \cite{zhang2024discovering,kleinman2024gacs}. Concretely, an ideal representation can be obtained through optitimizing two quantities: maximizing $H(C)$ to capture the Gács-Körner common information (as defined in Definition \ref{defi1}), while minimizing $H(C)+\sum_{j=1}^m H(U^{(j)})$ such that $H(Z) = H(C)+\sum_{j=1}^m H(U^{(j)})$, leading to statistically independent common and unique components satisfying $\mathrm{TC}(C,U^{(1)},\ldots,U^{(m)})=0$. Therefore, this could yield a sufficiently decoupled representation $Z_*=(C_*,U^{(1)}_*,\ldots,U^{(m)}_*)$, while achieving smaller generalization error for multi-view reconstruction.
    
    
\end{remark}
\begin{remark}
    The derived bound exhibits a convergence rate of $\widetilde{\mathcal{O}}(1/\sqrt{nm})$ with respect to the number $n$ of multi-view samples and the number $m$ of views. In particular, when setting $m=1$, Theorem \ref{theorem1} yields the generalization bound of single-view learning, tightening the existing bounds in \cite{shwartz2018representation,hafez2020sample} by replacing $2^{H(Z)}$ with $H(Z)$ (i.e., $H(C,U^{(1)})$).
\end{remark}

In the following, we connect the generalization error to the mutual information between each view and the common and unique components, as well as to the R\'{e}nyi's entropy of the representation function:

\begin{theorem} \label{theorem2}
    For any $\gamma>0$ and $\delta>0$, with probability at least $1-\delta$ for all $\phi=\{\phi^{(j)}\}_{j=1}^m \in\Phi$:
     \begin{align*}
         \overline{\mathrm{gen}}_{rec}  \leq \frac{\mathcal{K}_{3,\phi} + \mathcal{K}_1 \sqrt{\mathcal{K}_{2,\lambda}} }{\sqrt{nm}} 
         +\mathcal{K}_1\sqrt{\frac{ \sum_{j=1}^m \big(I(X^{(j)};C)+ I(X^{(j)};U^{(j)})\big)  + H_{1-\lambda}(\phi) }{nm} }
     \end{align*}
     where $\mathcal{K}_1, \mathcal{K}_{2,\lambda}$ are constants of order $\widetilde{\mathcal{O}}(1)$, and $\mathcal{K}_{3,\phi} = \mathcal{O}(\sqrt{H_{1-\lambda}(\phi)+1})$, specifically defined in Appendix \ref{Proof-Thm2}.

\end{theorem}
\begin{remark}
    Theorem \ref{theorem2} generalizes the theoretical understanding of stochastic multi-view representations, including Theorem \ref{theorem1} as a special case. A key distinction between them lies in the information theoretical characterization: independent cross-view common and unique information can be captured by maximizing the mutual information $I(X^{(j)};C)$ in Definition \ref{defi2}, while simultaneously minimizing $I(X^{(j)};C) + I(X^{(j)};U)$ for all $j\in[m]$. This approach not only reduces the computational complexity associated with the Gács-Körner common information on high-dimensional data, but also flexibly adapts to a variety of learning models. Additionally, the upper bound in Theorem \ref{theorem2} includes an extra R\'{e}nyi's entropy term $H_{1-\lambda}(\phi)\approx I(\phi;S)$, quantifying the information of the representation function $\phi$ retained from the multi-view training data $S$, where $\phi$ is typically used at test time as opposed to training time and $H(\phi|S) = 0$. This provides a novel insight beyond Theorem \ref{theorem1}: a trade-off between the data fit and the generalization needs to account for achieving both the representation capability and the generalization ability. If the learning hypothesis $\phi$ is deterministic, by $\sum_{i=1}^nI(X_i;Y)\leq I(\{X_i\}_{i=1}^n;Y)$ (Lemma A.10 in \cite{dongtowards}), we have $H(\phi) = 0$, $\sum_{j=1}^m I(X^{(j)};C) \leq I(X;C) = H(C)$, $I(X^{(j)};U^{(j)}) = H(U^{(j)})$, and Theorem \ref{theorem2} can recover the upper bound in Theorem \ref{theorem1}
\end{remark}
\begin{remark}\label{remark4.2}
    Theorem \ref{theorem2} scales $\widetilde{\mathcal{O}}(1/\sqrt{nm})$ in the number of multi-view samples and views, providing tighter bounds in terms of mutual information compared to single-view learning. Let $I(X;Z)$ denote the mutual information of single-view learning, we similarly prove that $I(X;Z)\geq  \frac{1}{m}\sum_{j=1}^m I(X^{(j)};Z^{(j)})$$\geq \frac{1}{m}\sum_{j=1}^m I(X^{(j)};C^{(j)}) + I(X^{(j)};U^{(j)})$, which demonstrates the benefits of learning from multiple views over a single view \cite{shwartz2018representation,hafez2020sample}.
\end{remark}

\subsection{Generalization Bounds for Multi-view Classification Tasks}\label{Section4.2}

We further investigate the generalization properties of the representation function for classification tasks. The following theorem provides the generalization bound of deterministic representation functions, illustrating the crucial role of the multi-view information bottleneck regularizer for controlling the generalization error.

\begin{theorem}\label{theorem3}
    For any $\gamma>0$ and $\delta>0$, with probability at least $1-\delta$, we have 
    \begin{align*}  
        \overline{\mathrm{gen}}_{cls} \leq  \widetilde{\mathcal{K}}_1 \sqrt{\frac{\sum_{j=1}^m I(X^{(j)};C,U^{(j)}|Y) +  \widetilde{\mathcal{K}}_2}{nm}}+\frac{ \widetilde{\mathcal{K}}_3}{\sqrt{nm}}
    \end{align*}
where $\widetilde{\mathcal{K}}_1 = \mathcal{O}(\sqrt{\vert \mathcal{Y}\vert})$, and $\widetilde{\mathcal{K}}_2 , \widetilde{\mathcal{K}}_3 $ are constants of order $\widetilde{\mathcal{O}}(1)$ as $n,m\rightarrow \infty$, specifically defined in Appendix \ref{Proof-Thm3}.
\end{theorem}
\begin{remark}
    Theorem \ref{theorem3} establishes an upper bound on the generalization error in terms of the conditional mutual information $I(X^{(j)};C,U^{(j)}|Y)$ between single views $X^{(j)}$ and its representations $(C,U^{(j)})$, conditioned on the label $Y$. It is noteworthy that $I(X^{(j)};C,U^{(j)}|Y)$ could serve as an information bottleneck regularizer to achieve the compressed representation and reduce the generalization error, which has been empirically proven to be successful in recent studies \cite{federici2020learning,lee2021compressive}. In theory, the spirit of the information bottleneck is to regularize $I(X^{(j)};C,U^{(j)})$ while the maximizing $I(C,U^{(j)};Y)$. By $I(X;Z)=I(X;Z|Y)+I(Z;Y)$ (Eq. (2) in \cite{federici2020learning}), it is obvious that regularizing $I(X^{(j)};C,U^{(j)}|Y)$ instead of $I(X^{(j)};C,U^{(j)})$ is a more intuitive and direct way to learn sufficient compact and highly label-relevant representations. Furthermore, $I(X^{(j)};C,U^{(j)}|Y)$ can be reduced to zero, leading to a tighter upper bound.
\end{remark}
\begin{remark}
    Notably, Theorem \ref{theorem3} yields a more rigorous upper bound than that of single-view learning \cite{kawaguchi2023does}, demenstrating the the advantages of learning representations from multiple views. Let $I(X;Z|Y)$ denote the conditional mutual information of single-view learning. Analogous to the analysis of Remark \ref{remark4.2}, it is evident that $I(X;Z|Y)\geq \frac{1}{m}\sum_{j=1}^m I(X^{(j)};Z^{(j)}|Y) = \frac{1}{m}\sum_{j=1}^m I(X^{(j)};C,U^{(j)}|Y)$. When setting $m=1$, Theorem \ref{theorem3} recover the bound of single-view learning, exhibiting comparable convergence order of $\widetilde{\mathcal{O}}(1/\sqrt{n})$ to \cite{kawaguchi2023does}. In contrast to previous work, our bound provides clear interpretation in terms of decomposing random variables into unique and common components.
\end{remark}
\begin{remark}
    Compared to Theorems \ref{theorem1} and \ref{theorem2}, Theorem \ref{theorem3} inherently depends on the label space $\mathcal{Y}$, implying a trade-off associated with its space complexity. On the one hand, the bound scales as $\mathcal{O}(\sqrt{\vert \mathcal{Y}\vert})$ in the cardinality of $\mathcal{Y}$, potentially leading to trivial bounds when $\vert \mathcal{Y}\vert$ is large. Theorem \ref{theorem2} is preferred for obtaining meaningful generalization guarantees, since each view is a more desirable supervisory signal than the label in this setting. On the other hand, when $I(X^{(j)};Y)$ is large, implying that each single view contains highly label-relevant information, it is advisable to apply Theorem \ref{theorem3} for tighter upper bounds. 
\end{remark}

By extending the analysis to stochastic representation functions, we derive the following upper bound that reconciles the information bottleneck regularizer $I(X^{(j)};C,U^{(j)}|Y)$ with the Rényi entropy $H_{1-\lambda}(\phi)$ of the representation function $\phi$:
\begin{theorem}\label{theorem4}
    For any $\gamma>0$ and $\delta>0$, with probability at least $1-\delta$, we have for all $\phi=\{\phi^{(j)}\}_{j=1}^m \in\Phi$:
    \begin{align*}  
        \overline{\mathrm{gen}}_{cls} \leq \frac{ \widetilde{\mathcal{K}}_{3,\phi}}{\sqrt{nm}}  
        + \widetilde{\mathcal{K}}_1 \sqrt{\frac{\sum_{j=1}^m I(X^{(j)};C,U^{(j)}|Y) + H_{1-\lambda}(\phi) +  \widetilde{\mathcal{K}}_{2,\lambda}}{nm}}, 
    \end{align*}
where $\widetilde{\mathcal{K}}_1 = \mathcal{O}(\sqrt{\vert \mathcal{Y}\vert})$, $\widetilde{\mathcal{K}}_{2,\lambda} = \widetilde{\mathcal{O}}(1)$, and $\widetilde{\mathcal{K}}_{3,\phi} = \mathcal{O}(\sqrt{H_{1-\lambda}(\phi) + 1})$ as $n,m\rightarrow \infty$, specifically defined in Appendix \ref{Proof-Thm4}.
\end{theorem}
\begin{remark}
    Theorem \ref{theorem4} provides additional insights beyond Theorem \ref{theorem3} by capturing a novel relationship between the amount of information in the representation learned from multi-view inputs, i.e., $I(X^{(j)};C,U^{(j)}|Y)$ and the information retained by the representation function from the training data, i.e., $H_{1-\lambda}(\phi)$. It is noteworthy that $H_{1-\lambda}(\phi)$ monotonically decreases as $\lambda$ decreases, implying a trade-off between  the Rényi's $(1-\lambda)$-order entropy of representation functions $\phi$ and the scaling rate of generalization. As a special case of Theorem \ref{theorem4} with $m=1$, this enables us to tighten the existing bounds of single-view learning \cite{kawaguchi2023does} by applying a smaller $\lambda$. Moreover, the inherent robustness of Rényi entropy contributes to enhancing the robustness of learned representations \cite{hanel2009robustness,lee2022renyicl}.
\end{remark}

\subsection{Data-dependent Generalization Bounds}\label{Section4.3}
The following theorems present novel sample complexity bounds by bounding the validation error under both the LOO and supersample settings:
\begin{theorem}\label{theorem5}
    If $\lambda \rightarrow 0$, for any $\delta>0$ and all $\phi=\{\phi^{(j)}\}_{j=1}^m \in\Phi$, with probability at least $1-\delta$, we have
        \begin{align*}
            \Delta_{\mathrm{loo}}  \leq \mathcal{K}^{u}_{1} \sqrt{\sum_{j=1}^{m} I(X^{(j)};C,U^{(j)}|Y)  +   I(\phi; U) + \mathcal{K}^{u}_{2,\lambda} },
        \end{align*}
    where $\mathcal{K}^{u}_{1} ,\mathcal{K}^{u}_{2,\lambda}$ are constants of order $\tilde{\mathcal{O}}(1)$, specifically defined in Appendix  \ref{Proof-Thm5}.
    \end{theorem}

    \begin{theorem}\label{theorem6}
        If $\lambda \rightarrow 0$, for any $\delta>0$, and all $\phi=\{\phi^{(j)}\}_{j=1}^m \in\Phi$, with probability at least $1-\delta$, we have
            \begin{align*}
                \Delta_{\mathrm{sup}} \leq \frac{\mathcal{K}^{\tilde{u}}_{1} \sqrt{\mathcal{K}^{\tilde{u}}_{2,\lambda}} }{\sqrt{nm}}
                 + \mathcal{K}^{\tilde{u}}_{1}\sqrt{\frac{\sum_{j=1}^{m} I(X^{(j)};C,U^{(j)}|Y)  +   I(\phi; \tilde{U})}{nm}},
        \end{align*}
    where $\mathcal{K}^{\tilde{u}}_{1}, \mathcal{K}^{\tilde{u}}_{2,\lambda}$ are constants of order $\tilde{\mathcal{O}}(1)$ as $n,m\rightarrow \infty$, specifically defined in Appendix \ref{Proof-Thm6}.
    \end{theorem}
\begin{remark}
    Notably, the upper bound of Theorem \ref{theorem6} scales proportionally with $\tilde{\mathcal{O}}(1/\sqrt{nm})$ in the supersample setting, whereas Theorem \ref{theorem5} does not, as only one sample is chosen to evaluate the test loss each time under the LOO setting. Nevertheless, both Theorems \ref{theorem5} and \ref{theorem6} can improve upon existing results by exploiting one-dimensional random variables ($U$ or $\tilde{U}$), achieving computational tractable and tighter bounds. Concretely, it can be proven that $I(\phi; \tilde{U})$ is strictly smaller than $H_{1-\lambda}(\phi)\approx I(\phi;S)$. By applying data-processing inequality on the Markov chain: $U\rightarrow \tilde{S}^s_{\mathrm{train}} \rightarrow \phi$ conditioned on $\tilde{S}_s$, we have $I(\phi;\tilde{U}|\tilde{S}_s) \leq (\phi;\tilde{S}^s_{\mathrm{train}}|\tilde{S}_s)$. We further utilize the independence between $\tilde{U}$ and $\tilde{S}_s$ and obtain $I(\phi;\tilde{U})\leq I(\phi;\tilde{U}) + I(\tilde{U};\tilde{S}_s|\phi) = I(\phi;\tilde{U}|\tilde{S}_s) + I(\tilde{U};\tilde{S}_s) = I(\phi;\tilde{U}|\tilde{S}_s)$. Similarly, the conditional independence between $\tilde{S}_s$ and $\phi$ given $\tilde{S}^s_{\mathrm{train}}$ indicates that $I(\phi;\tilde{S}^s_{\mathrm{train}}|\tilde{S}_s) \leq I(\phi;\tilde{S}^s_{\mathrm{train}}|\tilde{S}_s) +I(\phi;\tilde{S}_s) = I(\phi;\tilde{S}_s|\tilde{S}^s_{\mathrm{train}}) +I(\phi;\tilde{S}^s_{\mathrm{train}}) = I(\phi;\tilde{S}^s_{\mathrm{train}}) $. Since the training process $\tilde{S}^s_{\mathrm{train}}\mapsto \phi$ is deterministic, the randomness of  $\phi$ is mainly induced by the variable $\tilde{U}$, which implies that $H(\phi|\tilde{U})\approx 0$. With these in mind, we obtain that
    \begin{equation*}
        I(\phi;\tilde{U})\leq I(\phi;\tilde{U}|\tilde{S}_s)\leq I(\phi;\tilde{S}^s_{\mathrm{train}}|\tilde{S}_s)\leq I(\phi;S).
    \end{equation*}
    Therefore, the mutual information $I(\phi; \tilde{U})$ serves as a more stringent estimate of $I(\phi;S)$, thereby tightening the upper bounds. The above analysis can be similarly applied to $I(\phi;U)$. It is worth noting that the supersample validation error can approximate the generalization error when the training dataset is sufficiently large, as $\lim_{n,m\rightarrow\infty} \Delta_{\mathrm{sup}} = \overline{\mathrm{gen}}_{cls}$. In this case, Theorem \ref{theorem6} achieves a substantial improvement over Theorem \ref{theorem4}.
\end{remark}

\subsection{Fast-rate Generalization Bound}\label{Section4.4}

We further develop the fast-rate bound in the context of multi-view classification by leveraging the weighted generalization error: $L_{cls}-(1+\xi)\widehat{L}_{cls}$, where $\xi$ is a predefined positive constant. This framework facilitates the attainment of a faster convergence rate of $1/nm$ instead of the conventional $1/\sqrt{nm}$.

\begin{theorem}\label{theorem7}
    For any $\lambda\in(0,1)$, $\beta\in(0,\log 2)$, $ \xi\geq\frac{\log (2-e^{2\beta R^{\tilde{s}}_{x,y}})}{2\beta R^{\tilde{s}}_{x,y}}-1$, and $\delta>0$, with probability at least $1-\delta$, for all $\phi=\{\phi^{(j)}\}_{j=1}^m \in\Phi$, we have
    \begin{align*}
      \overline{\mathrm{gen}}_{cls} \leq  \xi\widehat{L}_{cls}   + \frac{\sum_{j=1}^{m} I(X^{(j)};C,U^{(j)}|Y)   + H_{1-\lambda}(\phi) +\hat{\mathcal{K}}}{nm\beta},
    \end{align*}
where $\hat{\mathcal{K}} = \tilde{\mathcal{O}}(1)$ is given in Appendix \ref{Proof-Thm7}. In the interpolating setting, i.e., $\widehat{L}_{cls} = 0$, we further have 
\begin{equation*}
    \overline{\mathrm{gen}}_{cls}  \leq \frac{\sum_{j=1}^{m} I(X^{(j)};C,U^{(j)}|Y)   + H_{1-\lambda}(\phi) +\hat{\mathcal{K}}}{nm\beta}.
\end{equation*}
\end{theorem}
\begin{remark}
    In the interpolating regime, where the empirical risk approaches zero, the weighted generalization error simplifies to the unweighted generalization error, leading to the bound scaling of $1/nm$. This characteristic renders the derived fast-rate bounds especially useful when the training error is small or even zero. In particular, Theorem \ref{theorem7} with $m=1$ enhances the generalization bounds of single-view learning \cite{shwartz2018representation,hafez2020sample,kawaguchi2023does} by achieving a faster scaling rate of $1/n$ as opposed to $1/\sqrt{n}$.
\end{remark}

\begin{figure*}[t]
    \centering
    \subfigure[Multi-view Reconstruction]{
        \label{fig:corr_loss} 
        \includegraphics[height=45mm]{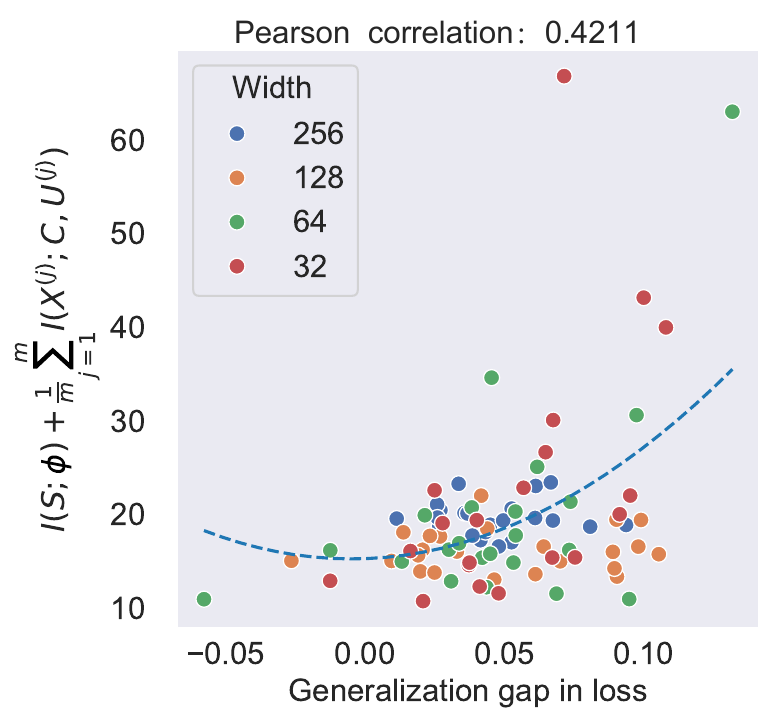} 
    }
    \hfill
    \subfigure[Multi-view Classification]{
        \label{fig:corr_cond} 
        \includegraphics[height=45mm]{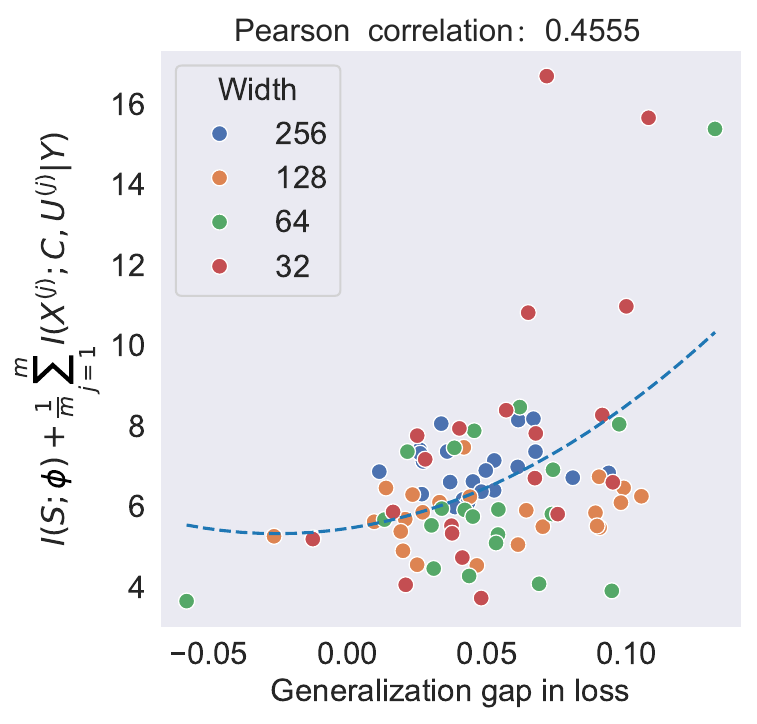} 
    }
    \hfill
    \subfigure[Pearson correlation analysis]{
        \raisebox{19mm}{ 
        \scalebox{0.75}{ 
            \begin{tabular}{lc}
                \toprule
                Metric & Correlation \\
                \midrule
                Num. params. & -0.0188 \\
                $\Vert W\Vert_F$ & -0.0781 \\
                $\frac{1}{m}\sum_{j=1}^{m}I(X^{(j)};C,U^{(j)})$ & 0.3712 \\
                $\frac{1}{m}\sum_{j=1}^{m}I(X^{(j)};C,U^{(j)}|Y)$ & 0.3842 \\
                $I(\phi;S)$ & 0.0508 \\
                $I(\phi;S) + I(X;Z)$ &  0.3928\\
                $I(\phi;S) + I(X;Z|Y)$ &  0.4130 \\
                $I(\phi;S) + \frac{1}{m}\sum_{j=1}^{m}I(X^{(j)};C,U^{(j)})$ &  \underline{0.4211}\\
                $I(\phi;S) + \frac{1}{m}\sum_{j=1}^{m}I(X^{(j)};C,U^{(j)}|Y)$ &  \textbf{0.4555} \\
                \bottomrule
            \end{tabular}
        }}
        \label{tbl:corr}
    }
    \caption{Pearson correlation analysis between the generalization error and information measures in the derived bounds for a five-layer MLP trained on synthetic Gaussian datasets. (a),(b): The correlations of $\frac{1}{m}\sum_{j=1}^{m}I(X^{(j)};C,U^{(j)})$ and $\frac{1}{m}\sum_{j=1}^{m}I(X^{(j)};C,U^{(j)}|Y)$ with the generalization error for both reconstruction and classification. (c): Comparison of Pearson correlation coefficients for different factors and the generalization error.}
    \label{fig1}
    \vskip -0.2in 
\end{figure*}

\begin{figure*}[t]
    \vskip 0.2in
    \centering
    \subfigure[MNIST (CNN, Adam)]{
      \label{aa} 
      \includegraphics[height=35mm]{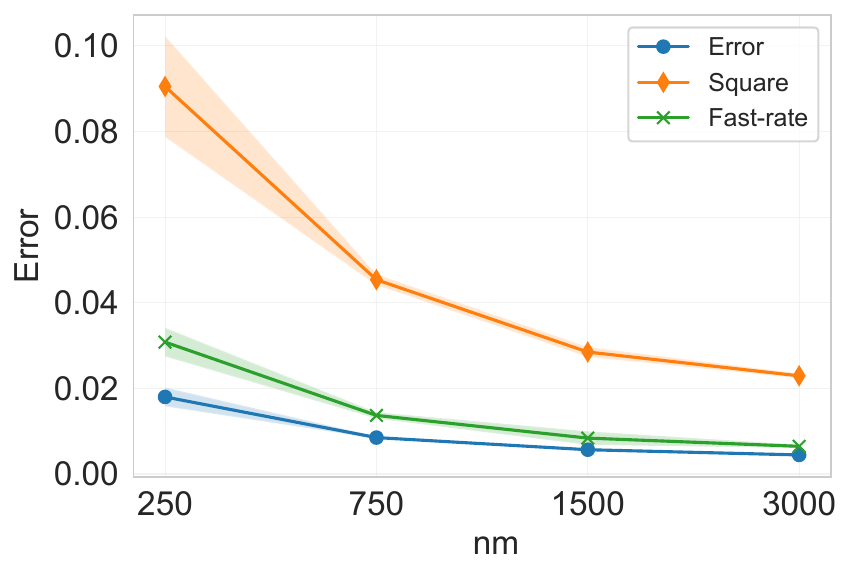}}
    \subfigure[CIFAR-10 (ResNet, SGD)]{
      \label{bb} 
      \includegraphics[height=35mm]{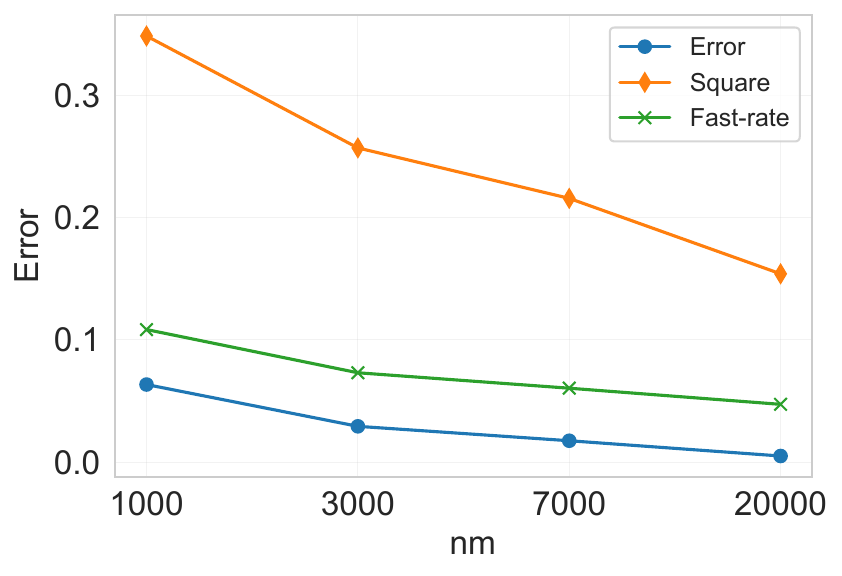}}
      \subfigure[CIFAR-10 (ResNet, SGLD)]{
      \label{bb} 
      \includegraphics[height=35mm]{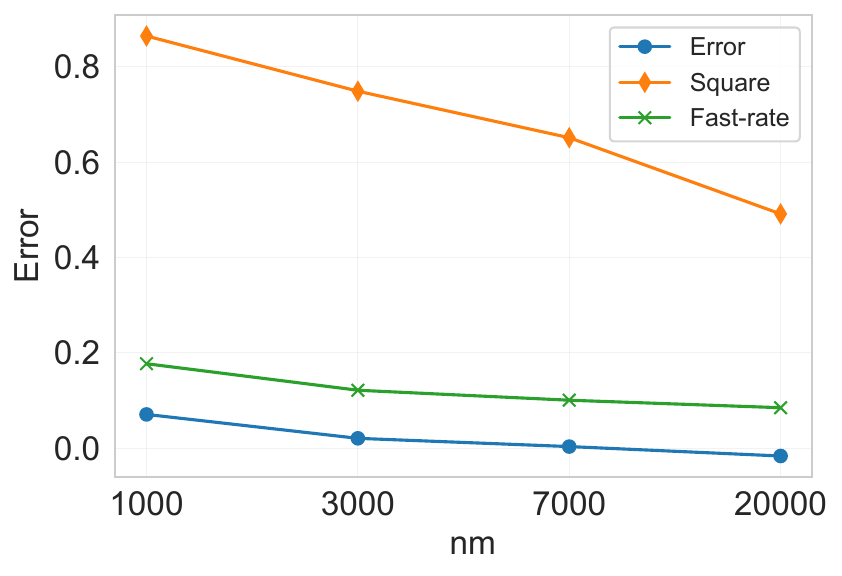}}
    \caption{Comparison of the generalization bounds for classification tasks on real-world datasets with different optimizers. (a) CNN model trained with binary MNIST using Adam, (b), (c): pretrained ResNet-50 model fine-tuned on CIFAR-10 using SGD and SGLD, respectively.
  }
    \label{fig2} 
    \vskip -0.2in
  \end{figure*}

\section{Numerical Results}
In this section, we empirically compare the derived generalization bounds for downstream classification tasks and the upper bounds of single-view learning \cite{kawaguchi2023does}. Our empirical evaluation consists of two parts: firstly, we investigate the effectiveness of learning from multiple views versus a single view by comparing the upper bounds on synthetic Gaussian datasets, which follows the same settings as \cite{kawaguchi2023does}. Secondly, we evaluate the tightness between the derived upper bounds and the generalization error by employing complex neural networks on real-world image classification datasets (ResNet-50 \cite{he2016deep} on CIFAR-10 dataset \cite{krizhevsky2009learning} and four-layer CNN on MNIST dataset \cite{deng2012mnist}), which follows the deep learning settings as in \cite{harutyunyan2021information,wang2023tighter}. In experiments, we regard all data of each category in the adopted dataset as a multi-view sample, with each data representing a distinct view of the multi-view data. We utilize the binary loss to quantify the empirical and population risks.

\subsection{Synthetic Datasets}
We train a $5$-layer MLP network consisting of a variational encoder and a classifier on synthetic Gaussian datasets. We follow the experimental settings adopted in \cite{kawaguchi2023does}, where 216 models are trained with varying model architectures, weight decay rates, random seeds, and dataset draws. To effectively evaluate the predictive power of different information measures on the generalization ability, , we extend the method from previous work \cite{galloway2022bounding,kawaguchi2023does} to multi-view learning scenarios and apply Pearson correlation analysis. As shown in Figure \ref{fig1}, both $\frac{1}{m}\sum_{j=1}^{m}I(X^{(j)};C,U^{(j)})$ and $\frac{1}{m}\sum_{j=1}^{m}I(X^{(j)};C,U^{(j)}|Y)$ exhibit a greater degree of correlation with the generalization error in comparison to other metrics, encompassing the number of the weight parameters, the F-norm of the learning hypothesis, and information bottleneck regularizer $I(X;Z|Y)$ of single-view learning. This empirical observation underscores the superiority of information measures involving multiple views over single views in capturing generalization dynamics.


\subsection{Real-world Datasets}
To precisely quantify information-theoretic generalization bounds in deep learning on real-world datasets, we utilize the experimental configuration provided by \cite{harutyunyan2021information}. Specifically, we train a $4$-layer CNN on the binary MNIST dataset with digits $4$ and $9$ and fine-tune a pre-trained ResNet-50 model on the CIFAR-10 dataset. Figure \ref{fig2} presents a comparison of the generalization error with the square-root bound and the fast-rate bound. This visualization result indicates that the derived upper bounds closely align with the trend of the generalization error, showing a decrease as the number of multi-view samples increases. Moreover, the fast-rate bound stands as the most stringent estimation of the generalization error among the comparisons, demonstrating its effectiveness with small training errors as elaborated in our analysis.

\section{Conclusion}

In this paper, we provide a comprehensive information-theoretic generalization analysis for multi-view learning. Specifically, we establish high-probability generalization bounds for both multi-view reconstruction and multi-view classification tasks. Our results demonstrate that capturing consensus and complementarity information from multiple views within an information-theoretic framework facilitates compact and maximally disentangled representations. In addition, the derived bounds reveal the critical role of the multi-view information bottleneck regularizer in improving the generalization performance for downstream classification tasks. In the LOO and supersample settings, we further derive novel data-dependent generalization bounds, achieving computational tractability and tightening existing results. Furthermore, we provide the first fast-rate bounds for multi-view learning to the best of our knowledge, yielding a faster converge rate in terms of the number of multi-view samples and the number of views within multi-view data. Numerical results validate the effectiveness of the derived bounds in capturing the generalization dynamics of multi-view learning. In future work, we will design theory-driven multi-view learning algorithms valid for various tasks to achieve excellent representation power and generalization performance.

\section*{Impact Statement}
This paper provides the information-theoretic generalization analysis for meta-learning with the goal of advancing the field of Machine Learning. There are many potential societal consequences of our work, none of which we feel must be specifically highlighted here.

\bibliography{references}
\bibliographystyle{icml2025}

\newpage
\appendix
\onecolumn
\section{Notations}
The main notations for proofs are summarized in Table \ref{Notations}.

\begin{table}[!ht]
    \centering
    \renewcommand\arraystretch{1.5}
    \caption{Summary of main notations involved  in this paper}
    \resizebox{\textwidth}{!}{
    \begin{tabular}{l|l}
        \hline
        Notations & Descriptions\\
        \hline
        $\mathcal{X}$ & the feature space. \\
        $\mathcal{Y}$ & the label space. \\
        $\Phi$ & the hypothesis space of functions mapping from $\mathbb{R}^{m\times d}$ to $\mathbb{R}^{m\times d}$. \\
        $X$ & the multi-view data variable. \\ 
        $Y$ & the label variable. \\ 
        $X_y$ & the random variable $X$ conditional on $Y=y$.\\
        $C$ & the common information across multiple views. \\
        $U^{(j)}$ & the unique information of the $j$-th view. \\
        $Z$ & the learned representation, denoted by $Z = (C, U^{(1)},\ldots, U^{(m)}).$\\
        $x $& the specific multi-view data consisting of $m$ views, i.e., $x=(x^{(1)},\ldots,x^{(m)})\in\mathbb{R}^d$. \\
        $y $& the specific label. \\
        $s$ & the training dataset for unsupervised learning, defined by $s=\big\{x_i=(x_i^{(1)},\ldots,x_i^{(m)})\big\}^n_{i=1}$.\\
        $S$ & the training dataset of multi-view learning, defined by $\tilde{S}=\{(x_i,y_i)\}$. \\
        $\tilde{S}_l$ & the dataset for leave-one-out setting, defined by  $\tilde{S}_l=\{(x_i,y_i)\}_{i=1}^{n+1}$. \\
        $\tilde{S}_{s}$ & the dataset for supersample setting, defined by  $\tilde{S}_{s} =\{(x_{i,0},y_{i,0}),(x_{i,1},y_{i,1})\}_{i=1}^n$. \\
        $U$ & the uniform random variable used to select the single test sample from $\tilde{S}_l$, $U\sim\mathrm{Unif}([n+1])$. \\
        $\tilde{U}$ & the uniform random variable used to split training and test samples from $\tilde{S}_s$, $\tilde{U}=\{\tilde{U}_i\}_{i=1}^n\sim\mathrm{Unif}(\{0,1\}^n)$. \\
        $\phi^{(j)}$ & the representation function for the $j$-th view $\phi^{(j)}=:\{(\phi_c^{(j)}, \phi_u^{(j)})\}:\mathbb{R}^{d}\mapsto \mathbb{R}^{d}$. \\
        $\phi$ & the representation function of multi-view data, defined by $\phi=:\{\phi^{(j)}\}_{j=1}^m$. \\
        $\psi$ & the decoder function for multi-view reconstruction mapping from $\mathbb{R}^d$ to $\mathbb{R}^d$. \\
        $\hat{\psi}$ & the decoder function for multi-view classification mapping from $\mathbb{R}^d$ to $\mathbb{R}$. \\
        $\ell(\psi\circ\phi(X^{(j)}),X^{(j)})$ & the loss function for single-view reconstruction, where $\ell:\mathbb{R}^d \times \mathbb{R}^d\mapsto\mathbb{R}_+$. \\
        $\hat{\ell}(\hat{\psi}\circ\phi(X^{(j)}),Y)$ & the loss function for single-view classification, where $\hat{\ell}:\mathbb{R} \times \mathbb{R}\mapsto\mathbb{R}_+$.\\
        $\ell_{\mathrm{avg}}(X;\psi,\phi)$ & the average loss of multi-view reconstruction, defined by $\ell_{\mathrm{avg}}(X;\psi,\phi) = \frac{1}{m}\sum_{j=1}^{m}\ell(\psi\circ\phi^{(j)}(X^{(j)}), X^{(j)})$ \\
        $\hat{\ell}_{\mathrm{avg}}((X,Y);\hat{\psi},\phi)$ & the average loss of multi-view classification, defined by $\hat{\ell}_{\mathrm{avg}}((X,Y);\hat{\psi},\phi) = \frac{1}{m}\sum_{j=1}^{m}\ell(\hat{\psi}\circ\phi^{(j)}(X^{(j)}), Y)$ \\
        $\overline{\mathrm{gen}}_{rec}$ & the generalization error for multi-view reconstruction.\\
        $\overline{\mathrm{gen}}_{cls}$ & the generalization error for multi-view classification. \\
        $\Delta_{\mathrm{loo}}$ & the validation error for leave-one-out (LOO) setting, $\Delta_{\mathrm{loo}}=\hat{\ell}_{\mathrm{avg}}((x_U,y_U);\hat{\psi},\phi)-\frac{1}{n}\sum_{i\neq U}\hat{\ell}_{\mathrm{avg}}((x_i,y_i);\hat{\psi},\phi)$.\\
        $\Delta_{\mathrm{sup}}$ & the validation error for supersample setting, $\Delta_{\mathrm{sup}}=\frac{1}{n}\sum_{i=1}^{n} \hat{\ell}_{\mathrm{avg}}((x_{i,1-\tilde{u}_i},y_{i,1-\tilde{u}_i});\hat{\psi},\phi)  - \frac{1}{n}\sum_{i=1}^{n} \hat{\ell}_{\mathrm{avg}}((x_{i,\tilde{u}_i},y_{i,\tilde{u}_i}); \hat{\psi},\phi)$. \\
        $R_x,R_{x,y}$ & the maximum attainable losses, $R_x=\sup_{x\in\mathcal{X}}\ell_{\mathrm{avg}}(X;\psi,\phi)$ and $R_{x,y}=\sup_{(X,Y)\in\mathcal{X}\times\mathcal{Y}}\hat{\ell}_{\mathrm{avg}}((X,Y);\hat{\psi},\phi)$, respectively. \\
        $R^s_x,R^{\tilde{s}}_{x,y}$ & the maximum samplewise losses, $R^s_x=\sup_{i\in[n]}\ell_{\mathrm{avg}}(x_i;\psi,\phi)$ and $R^{\tilde{s}}_{x,y}=\sup_{i\in[n]}\hat{\ell}_{\mathrm{avg}}((x_i,y_i);\hat{\psi},\phi)$. \\
        $\vert \mathcal{Y}\vert$ & the cardinality of hypothesis space $\mathcal{Y}$. \\
        $e$ & the base of the natural logarithm. \\
        \hline
    \end{tabular}
    }
    \label{Notations}
\end{table}

\section{Additional Lemmas}
\begin{lemma}\label{HZ}\citep{watanabe1960information}
    Given a set of $n$ stochastic variables $\lambda = (\lambda_1,\lambda_2,\ldots,\lambda_n) $, where $\lambda_i$ can take any different discrete values. The set $\lambda$ of $n$ variables is now divided into two subsets $\mu$ and $\upsilon$ respectively containing $l$ and $m$ variables, which satisfies $\mu\bigcup\upsilon=  \lambda$, $\mu\bigcap \upsilon=\emptyset$, and $n=l+m$. Then, $H(\lambda)=H(\mu)+H(\upsilon)-TC(\lambda;\mu,\upsilon)$, where $TC(\lambda;\mu,\upsilon)=H(\mu)+H(\upsilon)-H(\lambda)$ represents information redundancy.
\end{lemma}
\begin{lemma}\label{kawaguchi}\citep{kawaguchi2022robustness}
    Let the vector $X = (X_1,\ldots,X_k)$ follows the multinomial distribution with parameters and $p=(p_1,\ldots,p_k)$. Let $\bar{a}_1,\ldots,\bar{a}_k\geq 0$ be fixed such that $\sum_{i=1}^k \bar{a}_ip_i\neq 0$. Then, for any $\epsilon>0$,
    \begin{equation*}
        \mathbb{P}\Big(\sum_{i=1}^{k}\bar{a}_i\big(p_i-\frac{X_i}{m}\big)>\epsilon\Big)\leq \exp\Big(-\frac{m\epsilon^2}{\beta}\Big),
    \end{equation*} 
    where $\beta = 2\sum_{i=1}^k\bar{a}_i^2p_i$.
\end{lemma}

\section{Proof Sketch}
For clarity, we present a proof sketch for deriving information-theoretic generalization bounds established in this paper. Our proof involves the construction of typical subsets and the tricky usage of probability bounds, including the four steps:

\textbf{Step 1.} Construct the typical subset. we construct the typical subsets associated with the feature representation space and the learning hypothesis space, and prove their properties by leveraging a standard proof used in information theory and the McDiarmid's inequality.

\textbf{Step 2.} Decompose the generalization gap. We decompose into three terms (i.e., $\mathrm{\uppercase\expandafter{\romannumeral1}} + \mathrm{\uppercase\expandafter{\romannumeral2}} + \mathrm{\uppercase\expandafter{\romannumeral3}}$ in Lemma \ref{lemmaC.1}), where the third term corresponds to the case of the extracted representation being in the typical set, while other two terms are for the case of the extracted representation being outside of the typical set.

\textbf{Step 3.} Bound each terms.  We bound each term in the decomposition. To be specific,  we derive an upper bound on the first term by invoking properties of the typical subset, while the other two terms are bounded by recasting the problem into that of multinomial distributions and applying the concentration inequality of multinomial distributions.

\textbf{Step 4.} Combine the upper bounds of each term, we then obtain the desirable bounds.

\section{Proof of Theorems \ref{theorem1}\&\ref{theorem2} [Generalization Bound for Multi-view Reconstruction]} 

Let multi-view data $X=\{X^{(j)}\}_{j=1}^m$ be generated with a hidden function $\theta$ by $X = \theta(Y,V)$, where $Y$ is the randomly generated label, $\theta$ is some hidden deterministic function, and $V=\{V^{(j)}=(V^{(j)}_1,\ldots,V^{(j)}_d)\}_{j=1}^m\in\mathcal{V}\subseteq \mathbb{R}^{m\times d}$ are i.i.d. nuisance variables. Denote the random variables for $X$ and $Z$ conditioned on $Y=y$ by $X_y$ and $Z_y$. For any $y\in\mathcal{Y}$, define the sensitivity $c^y_{\phi}$ of the representation function $\phi=\{\phi^{(j)}\}_{j=1}^m$ w.r.t the nuisance variable $V^{(j)}_i$:
\begin{align*}
    c^y_{\phi} = \sup_{j\in [m]} \sup_{v^{(j)}_1,\ldots, \hat{v}^{(j)}_i,\ldots,v^{(j)}_d}\vert & \log(p_{z}(\phi^{(j)}\circ\theta_y(v^{(j)}_1,\ldots, v^{(j)}_i,\ldots,v^{(j)}_d))) 
    - \log(p_{z}( \phi^{(j)}\circ\theta_y(v^{(j)}_1,\ldots, \hat{v}^{(j)}_i,\ldots,v^{(j)}_d)))  \vert,
\end{align*}
where $p_{z}(z)=\mathbb{P}(Z=z)$ and $\theta_y(v^{(j)}) = \theta (y,v^{(j)})$, and then define the global sensitivity of $\phi$ by
\begin{align*}
    c_{\phi} = \mathbb{E}_Y[c^Y_{\phi}].
\end{align*}
Let the set of all possible multi-view data and its representation by 
\begin{equation*}
    \mathcal{Z}^{x} = \{\theta_y(v),\phi\circ\theta_y(v): v\in\mathcal{V}, y\in\mathcal{Y} \}.
\end{equation*}
For any $\gamma>0$, we define the typical subset $\mathcal{Z}^{x}_{\gamma}$ of $\mathcal{Z}^{x}$ by
\begin{equation}\label{subset1}
    \mathcal{Z}^{x}_{\gamma}  = \Big\{x=(x^{(1)},\ldots,x^{(m)}),z=(c,u^{(1)},\ldots,u^{(m)})\in \mathcal{Z}^{x}: -\log p_{z}(z) - H(Z)\leq c_{\phi} \sqrt{\frac{d\log(\sqrt{nm}/\gamma)}{2}} \Big\}. 
\end{equation}

\subsection{Properties of the Typical Subset}

\begin{lemma}\label{lemma0}
    For any $\gamma>0$ and all $j\in[m]$, we have 
    \begin{equation*}
        \mathbb{P}(X,Z\notin  \mathcal{Z}^{x}_{\gamma})\leq \frac{\gamma}{\sqrt{n}},
    \end{equation*}
and 
   \begin{equation*}
    \vert  \mathcal{Z}^{x}_{\gamma}\vert \leq \exp( H(Z) + c_{\phi} \sqrt{\frac{d\log(\sqrt{nm}/\gamma)}{2}}).
   \end{equation*}
\end{lemma}
\begin{proof}
    Consider the function $f(y,v) = -\log p_{z}(h_y(v))$, where $h_y(v) =\phi\circ\theta_y(v)$. Let $p_y(y)=\mathbb{P}(Y = y)$, $p_v(v)=\mathbb{P}(V = v)$, and $h_y^{-1}(z)=\{v\in\mathcal{V}: h_y(v) = z\}$, we have 
    \begin{align*}
        \mathbb{E}_{Y,V} [f(Y,V)] = & -\sum_{y\in\mathcal{Y}}p_y(y)\sum_{v\in\mathcal{V}} p_v(v)\log  p_{z}(h_y(v))\\
        = & -\sum_{y\in\mathcal{Y}}p_y(y) \sum_{z\in{\mathcal{Z}}}\sum_{v\in h_y^{-1}(z)} p_v(v)\log  p_{z}(h_y(v))  \\
         =& -\sum_{z\in{\mathcal{Z}}}\Big(\sum_{y\in\mathcal{Y}}p_y(y)\sum_{v\in h_y^{-1}(z)} p_v(v) \Big)\log  p_{z}(z)  \\
         = & - \sum_{z\in{\mathcal{Z}}}p_{z}(z)\log p_{z}(z) \\
         =& H(Z).
    \end{align*}
By applying McDiarmid's inequality on $f(V) = -\log p_{z}(Z)$, we have 
\begin{equation*}
    \mathbb{P}(-\log p_{z}(Z)-H(Z)\leq \epsilon) \leq \exp\Big(-\frac{2\epsilon^2}{d (c_{\phi})^2}\Big).
\end{equation*}
Let $\delta = \exp\Big(-\frac{2\epsilon^2}{d (c_{\phi})^2}\Big)$, we have 
\begin{equation*}
    \epsilon = c_{\phi} \sqrt{\frac{d\log(1/\delta)}{2}}.
\end{equation*}
Combining with (\ref{subset1}) and setting $\delta = \gamma/\sqrt{nm}$, having 
\begin{align*}
    \mathbb{P} (X,Z\notin  \mathcal{Z}^{x}_{\gamma} ) \leq \delta = \frac{\gamma}{\sqrt{nm}}.
\end{align*}
and accordingly $\epsilon = c_{\phi} \sqrt{\frac{d\log(\sqrt{nm}/\gamma)}{2}}$.
Further consider the size of the typical subset, for any $z\in \mathcal{Z}^{x}_{\gamma}$, we have 
\begin{align*}
    -\log p_{z}(z)-H(Z)\leq & \epsilon \\
    -H(Z) - \epsilon \leq & \log p_{z}(z) \\
    \exp(-H(Z) - \epsilon) \leq & p_{z}(z).
\end{align*}
This implies that 
\begin{align*}
    1\geq  \sum_{z\in  \mathcal{Z}^{x}_{\gamma}}p_{z}(z)
    \geq   \sum_{z\in  \mathcal{Z}^{x}_{\gamma}} \exp( -H(Z) - \epsilon) = \vert  \mathcal{Z}^{x}_{\gamma}\vert \exp(-H(Z) - \epsilon).
\end{align*}
We thus obtain
\begin{equation}\label{upp}
    \vert  \mathcal{Z}^{x}_{\gamma}\vert \leq \exp( H(Z) + c_{\phi} \sqrt{\frac{d\log(\sqrt{nm}/\gamma)}{2}}).
\end{equation}

\end{proof}

\subsection{Decompose the Generalization Error}
Recall the definition of the generalization error of the learning hypothesis $\phi$ for multi-view reconstruction:
\begin{equation}\label{gen1}
  \overline{\mathrm{gen}}_{rec} =   \mathbb{E}_{X\sim \mathcal{X}} \big[\ell_{\mathrm{avg}}(X;\psi,\phi)\big] - \frac{1}{n}\sum^n_{i=1} \frac{1}{m}\sum_{j=1}^{m} \ell(\psi\circ\phi^{(j)}(x^{(j)}_i),x^{(j)}_i).
\end{equation}
We define $T = \vert  \mathcal{Z}^{x}_{\gamma}\vert$, the elements of the typical subset $ \mathcal{Z}^{x}_{\gamma}$ by $ \mathcal{Z}^{x}_{\gamma}=\{(a^x_1,a_1^c,a^{u}_1),\ldots,(a^x_T,a_T^c,a^{u}_T)\}$, and 
\begin{align*}
    \mathcal{U} = & \{i\in[n],j\in[m]: (x^{(j)}_i,\phi^{(j)}(x^{(j)}_i)) \notin \mathcal{Z}^{x}_{\gamma}\} \\
    \mathcal{U}_k =& \{i\in[n],j\in[m]: x^{(j)}_i=a^x_k ,\phi^{(j)}(x^{(j)}_i) = (a^c_k, a^{u}_k)\},
\end{align*}
which are dependent on the training set $S$. By using the above notations, we decompose the generalization error as follows.
\begin{lemma}\label{lemmaC.1}
    The generalization error satisfies:
    \begin{equation}\label{eq5}
        \mathbb{E}_{X\sim \mathcal{X}} \big[\ell_{\mathrm{avg}}(X;\psi,\phi)\big] - \frac{1}{n}\sum^n_{i=1} \frac{1}{m}\sum_{j=1}^{m} \ell(\psi\circ\phi^{(j)}(x^{(j)}_i),x^{(j)}_i)   = \mathrm{\uppercase\expandafter{\romannumeral1}} + \mathrm{\uppercase\expandafter{\romannumeral2}} + \mathrm{\uppercase\expandafter{\romannumeral3}},
     \end{equation}
     where 
     \begin{align*}
        \mathrm{\uppercase\expandafter{\romannumeral1}} = &  \mathbb{P}\Big(X^{(j)},\phi^{(j)}(X^{(j)})\notin \mathcal{Z}^{x}_{\gamma}\Big)\Big( \mathbb{E}_{X\sim \mathcal{X},X^{(j)}\sim X} \big[\ell(\psi\circ\phi^{(j)}(X^{(j)}),X^{(j)})|X^{(j)},\phi^{(j)}(X^{(j)})\notin \mathcal{Z}^{x}_{\gamma}\big] \\
        & -\frac{1}{\vert \mathcal{U}\vert} \sum_{i,j\in \mathcal{U}} \ell(\psi\circ\phi^{(j)}(x^{(j)}_i),x^{(j)}_i)\Big) \\
        \mathrm{\uppercase\expandafter{\romannumeral2}} = &  \frac{1}{\vert \mathcal{U}\vert} \Big(\mathbb{P}\Big(X^{(j)},\phi^{(j)}(X^{(j)})\notin \mathcal{Z}^{x}_{\gamma}\Big) - \frac{\vert \mathcal{U}\vert}{nm} \Big)  \sum_{i,j\in\mathcal{U}}   \ell(\psi\circ\phi^{(j)}(x^{(j)}_i),x^{(j)}_i), \\
        \mathrm{\uppercase\expandafter{\romannumeral3}} = &\sum_{k=1}^{T} \Big(\mathbb{P}\Big(X^{(j)}=a^{x}_k ,\phi(X^{(j)}) = (a^c_k, a^{u}_k)\Big) - \frac{\vert \mathcal{U}_k\vert}{nm}  \Big) \ell\big(\psi(a^c_k, a^{u}_k),a^{x}_k\big).
    \end{align*}
\end{lemma}

\begin{proof}
Note that $\mathcal{U}\bigcup\mathcal{U}_1\bigcup\cdots\mathcal{U}_T=[n]\cup [m]$. The population risk can be decomposed by
    \begin{align}
        &\mathbb{E}_{X\sim \mathcal{X}} \big[\ell_{\mathrm{avg}}(X;\psi,\phi)\big]  \nonumber\\
        = & \mathbb{P}\big(X,Z\notin \mathcal{Z}^{x}_{\gamma}\big) \mathbb{E}_{X\sim \mathcal{X}}\Big[\ell_{\mathrm{avg}}(X;\psi,\phi)|(X,Z)\notin \mathcal{Z}^{x}_{\gamma}\Big]  + \sum_{k=1}^{T} \mathbb{P}\big(X,Z\in \mathcal{Z}^{x}_{\gamma}\big) \mathbb{E}_{X\sim \mathcal{X}} \big[\ell_{\mathrm{avg}}(X;\psi,\phi)|(X,Z)\in \mathcal{Z}^{x}_{\gamma}\big] \nonumber\\
        =&  \mathbb{P}\big(X^{(j)},\phi^{(j)}(X^{(j)})\notin \mathcal{Z}^{x}_{\gamma}\big) \mathbb{E}_{X\sim \mathcal{X},X^{(j)}\sim X}\Big[\ell\big(\psi\circ\phi^{(j)}(X^{(j)}),X^{(j)}\big)\big|X^{(j)},\phi^{(j)}(X^{(j)})\notin \mathcal{Z}^{x}_{\gamma} \Big]\nonumber\\
        &+ \sum_{k=1}^{T}\mathbb{P}\big(X^{(j)}=a^{x}_k ,\phi(X^{(j)}) = (a^c_k, a^{u}_k)\big) \mathbb{E}_{X\sim \mathcal{X},X^{(j)}\sim X}\Big[\ell\big(\psi\circ\phi^{(j)}(X^{(j)}),X^{(j)}\big)\big|X^{(j)}=a^{x}_k ,\phi^{(j)}(X^{(j)}) = (a^c_k, a^{u}_k) \Big]  \nonumber\\
        =&  \mathbb{P}\big(X^{(j)},\phi^{(j)}(X^{(j)})\notin \mathcal{Z}^{x}_{\gamma}\big) \mathbb{E}_{X\sim \mathcal{X},X^{(j)}\sim X}\Big[\ell\big(\psi\circ\phi^{(j)}(X^{(j)}),X^{(j)}\big) |X^{(j)},\phi^{(j)}(X^{(j)})\notin \mathcal{Z}^{x}_{\gamma} \Big] \nonumber\\
        &+ \sum_{k=1}^{T}\mathbb{P}\big(X^{(j)}=a^{x}_k ,\phi^{(j)}(X^{(j)}) = (a^c_k, a^{u}_k)\big)  \ell\big(\psi(a^c_k, a^{u}_k),a^{x}_k\big)
         \label{de1}
    \end{align}
Similarly, the empirical risk is decomposed by 
\begin{align}
    \frac{1}{n}\sum^n_{i=1} \frac{1}{m}\sum_{j=1}^{m} \ell(\psi\circ\phi^{(j)}(x^{(j)}_i),x^{(j)}_i) = & \frac{1}{nm}\Big(\sum_{i,j\in\mathcal{U}} \ell(\psi\circ\phi^{(j)}(x^{(j)}_i),x^{(j)}_i) + \sum_{k=1}^T\sum_{i,j\in\mathcal{U}_k} \ell(\psi\circ\phi^{(j)}(x^{(j)}_i),x^{(j)}_i) \Big) \nonumber\\
    = & \frac{1}{nm} \sum_{i,j\in\mathcal{U}} \ell(\psi\circ\phi^{(j)}(x^{(j)}_i),x^{(j)}_i) +  \frac{1}{nm} \sum_{k=1}^T \sum_{i,j\in\mathcal{U}_k} \ell(\psi(a^c_k, a^{u}_k),a^{x}_k) \nonumber \\
    = & \frac{1}{nm} \sum_{i,j\in\mathcal{U}} \ell(\psi\circ\phi^{(j)}(x^{(j)}_i),x^{(j)}_i) +  \frac{\vert \mathcal{U}_k\vert}{nm}\sum_{k=1}^T  \ell(\psi(a^c_k, a^{u}_k),a^{x}_k). \label{de2}
\end{align}
Putting (\ref{de1}) and (\ref{de2}) back into (\ref{gen1}), we have 
\begin{align*}
      &\mathbb{E}_{X\sim \mathcal{X}} \big[\ell_{\mathrm{avg}}(X;\psi,\phi)\big] - \frac{1}{n}\sum^n_{i=1} \frac{1}{m}\sum_{j=1}^{m} \ell(\psi\circ\phi^{(j)}(x^{(j)}_i),x^{(j)}_i) \\
      = & \mathbb{P}\big(X^{(j)},\phi^{(j)}(X^{(j)})\notin \mathcal{Z}^{x}_{\gamma}\big) \mathbb{E}_{X\sim \mathcal{X},X^{(j)}\sim X}\Big[\ell(\psi\circ\phi^{(j)}(X^{(j)}),X^{(j)})|X^{(j)},\phi^{(j)}(X^{(j)})\notin \mathcal{Z}^{x}_{\gamma}\Big] \\
      & - \mathbb{P}\big(X^{(j)},\phi^{(j)}(X^{(j)})\notin \mathcal{Z}^{x}_{\gamma}\big) \frac{1}{\vert \mathcal{U}\vert} \sum_{i,j\in \mathcal{U}}\ell(\psi\circ\phi^{(j)}(x^{(j)}_i),x^{(j)}_i) \\
     & +\mathbb{P}\big(X^{(j)},\phi^{(j)}(X^{(j)})\notin \mathcal{Z}^{x}_{\gamma}\big) \frac{1}{\vert \mathcal{U}\vert} \sum_{i,j\in \mathcal{U}}\ell(\psi\circ\phi^{(j)}(x^{(j)}_i),x^{(j)}_i)  - \frac{1}{nm} \sum_{i,j\in\mathcal{U}} \ell(\psi\circ\phi^{(j)}(x^{(j)}_i),x^{(j)}_i) \\
      & + \sum_{k=1}^{T}\mathbb{P}\Big(X^{(j)}=a^{x}_k ,\phi(X^{(j)}) = (a^c_k, a^{u}_k)\Big)  \ell\big(\psi(a^c_k, a^{u}_k),a^{x}_k\big) -  \frac{\vert \mathcal{U}_k\vert}{nm} \sum_{k=1}^T  \ell(\psi(a^c_k, a^{u}_k),a^{x}_k) \\
      = &  \mathbb{P}\big(X^{(j)},\phi^{(j)}(X^{(j)})\notin \mathcal{Z}^{x}_{\gamma}\big)\Big( \mathbb{E}_{X\sim \mathcal{X},X^{(j)}\sim X} \big[\ell(\psi\circ\phi^{(j)}(X^{(j)}),X^{(j)})|X^{(j)},\phi^{(j)}(X^{(j)})\notin \mathcal{Z}^{x}_{\gamma}\big] \\
      & -\frac{1}{\vert \mathcal{U}\vert} \sum_{i,j\in \mathcal{U}} \ell(\psi\circ\phi^{(j)}(x^{(j)}_i),x^{(j)}_i)\Big) \\
      & +  \frac{1}{\vert \mathcal{U}\vert} \Big(\mathbb{P}\big(X^{(j)},\phi^{(j)}(X^{(j)})\notin \mathcal{Z}^{x}_{\gamma}\big) - \frac{\vert \mathcal{U}\vert}{nm} \Big)  \sum_{i,j\in\mathcal{U}}   \ell(\psi\circ\phi^{(j)}(x^{(j)}_i),x^{(j)}_i)\\
      & +  \sum_{k=1}^{T} \Big(\mathbb{P}\big(X^{(j)}=a^{x}_k ,\phi(X^{(j)}) = (a^c_k, a^{u}_k)\big) - \frac{\vert \mathcal{U}_k\vert}{nm}  \Big) \ell\big(\psi(a^c_k, a^{u}_k),a^{x}_k\big).
\end{align*}

To simplify the notation, we denote the above decomposition as
\begin{equation}\label{de3}
    \mathbb{E}_{X\sim \mathcal{X}} \big[\ell_{\mathrm{avg}}(X;\psi,\phi)\big] - \frac{1}{n}\sum^n_{i=1} \frac{1}{m}\sum_{j=1}^{m} \ell(\psi\circ\phi^{(j)}(x^{(j)}_i),x^{(j)}_i)   = \mathrm{\uppercase\expandafter{\romannumeral1}} + \mathrm{\uppercase\expandafter{\romannumeral2}} + \mathrm{\uppercase\expandafter{\romannumeral3}},
\end{equation}
where 
\begin{align*}
    \mathrm{\uppercase\expandafter{\romannumeral1}} = &  \mathbb{P}\Big(X^{(j)},\phi^{(j)}(X^{(j)})\notin \mathcal{Z}^{x}_{\gamma}\Big)\Big( \mathbb{E}_{X\sim \mathcal{X},X^{(j)}\sim X} \big[\ell(\psi\circ\phi^{(j)}(X^{(j)}),X^{(j)})|X^{(j)},\phi^{(j)}(X^{(j)})\notin \mathcal{Z}^{x}_{\gamma}\big] \\
    & -\frac{1}{\vert \mathcal{U}\vert} \sum_{i,j\in \mathcal{U}} \ell(\psi\circ\phi^{(j)}(x^{(j)}_i),x^{(j)}_i)\Big) \\
    \mathrm{\uppercase\expandafter{\romannumeral2}} = &  \frac{1}{\vert \mathcal{U}\vert} \Big(\mathbb{P}\Big(X^{(j)},\phi^{(j)}(X^{(j)})\notin \mathcal{Z}^{x}_{\gamma}\Big) - \frac{\vert \mathcal{U}\vert}{nm} \Big)  \sum_{i,j\in\mathcal{U}}   \ell(\psi\circ\phi^{(j)}(x^{(j)}_i),x^{(j)}_i), \\
    \mathrm{\uppercase\expandafter{\romannumeral3}} = &\sum_{k=1}^{T} \Big(\mathbb{P}\Big(X^{(j)}=a^{x}_k ,\phi(X^{(j)}) = (a^c_k, a^{u}_k)\Big) - \frac{\vert \mathcal{U}_k\vert}{nm}  \Big) \ell\big(\psi(a^c_k, a^{u}_k),a^{x}_k\big).
\end{align*}

\end{proof}
\subsection{Bounding Each Term in Decompositions}
\begin{lemma}\label{A1}
    For any $\gamma>0$, the following inequality holds:
    \begin{equation*}
        \mathrm{\uppercase\expandafter{\romannumeral1}} \leq \frac{\gamma R_x}{\sqrt{nm}}.
    \end{equation*}
\end{lemma}
\begin{proof}
    From Lemma \ref{lemma0}, we have 
    \begin{equation*}
        \mathbb{P}\big(X,Z\notin \mathcal{Z}^{x}_{\gamma}\big) \leq \frac{\gamma}{\sqrt{nm}}.
    \end{equation*}
Since $\ell(\phi(x_i),x_i)\geq 0$ for any $i\in[n]$, we thus obtain
    \begin{align*}
        \mathrm{\uppercase\expandafter{\romannumeral1}} = & \mathbb{P}\big(X^{(j)},\phi^{(j)}(X^{(j)})\notin \mathcal{Z}^{x}_{\gamma}\big)\Big( \mathbb{E}_{X\sim \mathcal{X},X^{(j)}\sim X} \big[\ell(\psi\circ\phi^{(j)}(X^{(j)}),X^{(j)})|X^{(j)},\phi^{(j)}(X^{(j)})\notin \mathcal{Z}^{x}_{\gamma}\big] \\
        & -\frac{1}{\vert \mathcal{U}\vert} \sum_{i,j\in \mathcal{U}} \ell(\psi\circ\phi^{(j)}(x^{(j)}_i),x^{(j)}_i)\Big), \\
        = & \mathbb{P}\big(X,Z\notin \mathcal{Z}^{x}_{\gamma}\big)\Big( \mathbb{E}_{X\sim \mathcal{X},X^{(j)}\sim X} \big[\ell(\psi\circ\phi^{(j)}(X^{(j)}),X^{(j)})|(X,Z)\notin \mathcal{Z}^{x}_{\gamma}\big] -\frac{1}{\vert \mathcal{U}\vert} \sum_{i,j\in \mathcal{U}} \ell(\psi\circ\phi^{(j)}(x^{(j)}_i),x^{(j)}_i)\Big) \\
        \leq &  \mathbb{P}\big(X,Z\notin \mathcal{Z}^{x}_{\gamma}\big) \mathbb{E}_{X\sim \mathcal{X},X^{(j)}\sim X} \big[\ell(\psi\circ\phi^{(j)}(X^{(j)}),X^{(j)})|(X,Z)\notin \mathcal{Z}^{x}_{\gamma}\big] \\
        \leq & \frac{\gamma}{\sqrt{nm}}  \mathbb{E}_{X\sim \mathcal{X},X^{(j)}\sim X} \big[\ell(\psi\circ\phi^{(j)}(X^{(j)}),X^{(j)})|(X,Z)\notin \mathcal{Z}^{x}_{\gamma}\big] \\
        \leq & \frac{\gamma R_x}{\sqrt{nm}}. 
    \end{align*}
\end{proof}

\begin{lemma}\label{BC1}
    For any $\gamma>0$, and $\delta>0$, with probability at least $1-\delta$, the following holds for all $\phi\in\Phi$
    \begin{align*}
        \mathrm{\uppercase\expandafter{\romannumeral2}} \leq & \frac{\sqrt{\mathbb{P}(X^{(j)},\phi^{(j)}(X^{(j)})\notin \mathcal{Z}^{x}_{\gamma})} \sum_{i,j\in\mathcal{U}}   \ell(\psi\circ\phi^{(j)}(x^{(j)}_i),x^{(j)}_i)}{\vert \mathcal{U}\vert}\sqrt{\frac{2\log(2\vert \Phi\vert/\delta)}{nm}},  \\
        \mathrm{\uppercase\expandafter{\romannumeral3}}  \leq & 2R_{x} \sqrt{\frac{2\Big(H(Z) +  c_{\phi} \sqrt{\frac{d\log(\sqrt{nm}/\gamma)}{2}} \Big) +  2\log(2\vert \Phi\vert/\delta)}{nm}}. 
    \end{align*}
\end{lemma}

\begin{proof}
     Define $p_k = \mathbb{P}\Big(X^{(j)}=a^{x}_k ,\phi(X^{(j)}) = (a^c_k, a^{u}_k)\Big)$, for $k\in[T]$, $p_{T+1}=\mathbb{P}(X^{(j)},\phi(X^{(j)})\notin  \mathcal{Z}^{x}_{\gamma})$, $b_k = \ell(\psi(a^c_k, a^{u}_k),a^x_k)$ for $k\in[T+1]$, and
\begin{equation*}
    \mathrm{\uppercase\expandafter{\romannumeral3}}_k = \sum_{t=1}^T\Big(p_t-\frac{\vert \mathcal{U}_t\vert}{nm}\Big)b_t-\Big(p_{k}-\frac{\vert \mathcal{U}_k\vert}{nm}\Big)b_k.
\end{equation*}
Applying Lemma \ref{kawaguchi} with 
\begin{align*}
    k = T+1,\quad X=(\vert \mathcal{U}_1\vert,\ldots,\vert \mathcal{U}_T\vert,\vert \mathcal{U}\vert), \quad p=(p_1,\ldots,p_{T+1}),\\
    m = nm, \quad \bar{a}_k=0, \quad \bar{a}_{T+1}=0, \quad \textrm{and} \quad \bar{a}_t = b_t \quad \textrm{for any } \quad t\neq k.
\end{align*}
When there exists $t\in[T]\backslash k$ such that $p_tb_t>0$, we have $\sum_{t=1}^{T+1}\bar{a}_tp_t \neq 0$, which satisfies the precondition of Lemma \ref{kawaguchi}. Then for any $\epsilon>0$ and $k\in[T]$,
\begin{equation*}
    \mathbb{P}(\mathrm{\uppercase\expandafter{\romannumeral3}}_k >\epsilon ) \geq \exp\Big(-\frac{nm\epsilon^2}{2(\sum_{t=1}^Tp_tb_t^2-p_kb_k^2)}\Big).
\end{equation*}
Setting $\delta = \exp\big(-\frac{nm\epsilon^2}{2(\sum_{t=1}^Tp_tb_t^2-p_kb_k^2)}\big)$ and solving $\epsilon$, we can get 
\begin{equation}\label{eq81}
    \mathbb{P}\bigg(\mathrm{\uppercase\expandafter{\romannumeral3}}_k > \sqrt{\sum_{t=1}^T p_tb_t^2-p_kb_k^2 }\sqrt{\frac{2\log(1/\delta)}{nm}}\bigg)\leq \delta,
\end{equation} 
for any $k\in[T]$. Similarly, by setting $\bar{a}_{T+1}=1$ and $\bar{a}_t=0$ for any $t\in[T]$, we have
\begin{equation*}
    \mathbb{P}\Big(p_{T+1}-\frac{\vert \mathcal{U}\vert}{nm}>\epsilon\Big) \leq \exp\Big(-\frac{nm\epsilon^2}{2p_{T+1}}\Big),
\end{equation*}
and further get
\begin{equation}\label{eq91}
    \mathbb{P}\Big(p_{T+1}-\frac{\vert \mathcal{U}\vert}{nm}>\sqrt{\frac{2p_{T+1}\log(1/\delta)}{nm}}\Big)\leq \delta.
\end{equation}
Note that for $p_ta_t=0$, $t\neq k$ or $p_{T+1}=0$, then $\mathrm{\uppercase\expandafter{\romannumeral3}}_k=0$ or $p_{T+1}-\frac{\vert \mathcal{U}\vert}{nm}=0$, and  (\ref{eq81}) and (\ref{eq91}) can be satisfied.

By substituting (\ref{eq91}) into $\mathrm{\uppercase\expandafter{\romannumeral2}}$, we have that for any $\delta>0$, with probability at least $1-\delta$,
\begin{align}
    \mathrm{\uppercase\expandafter{\romannumeral2}} = & \frac{1}{\vert \mathcal{U}\vert} \Big(\mathbb{P}\Big(X^{(j)},\phi^{(j)}(X^{(j)})\notin \mathcal{Z}^{x}_{\gamma}\Big) - \frac{\vert \mathcal{U}\vert}{nm} \Big)  \sum_{i,j\in\mathcal{U}}   \ell(\psi\circ\phi^{(j)}(x^{(j)}_i),x^{(j)}_i), \nonumber\\
    \leq & \frac{\sqrt{\mathbb{P}(X^{(j)},\phi^{(j)}(X^{(j)})\notin \mathcal{Z}^{x}_{\gamma})} \sum_{i,j\in\mathcal{U}}   \ell(\psi\circ\phi^{(j)}(x^{(j)}_i),x^{(j)}_i)}{\vert \mathcal{U}\vert}\sqrt{\frac{2\log(1/\delta)}{nm}}\label{BB}
\end{align}
Similarly, by using (\ref{eq81}), we have for any $\delta>0$ and $k\in[T]$, with probability at least $1-\delta$,
\begin{align*}
  \mathrm{\uppercase\expandafter{\romannumeral3}}_k  \leq & \sqrt{\sum_{t=1}^T  \mathbb{P}\big(X^{(j)}=a^{x}_t ,\phi(X^{(j)}) = (a^c_t, a^{u}_t)\big)b_t^2-\mathbb{P}\big(X^{(j)}=a^{x}_k ,\phi(X^{(j)}) = (a^c_k, a^{u}_k)\big)b_k^2 }\sqrt{\frac{2\log(1/\delta)}{nm}} \\
    \leq & R_{x} \sqrt{\sum_{t=1}^T  \mathbb{P}\big(X^{(j)}=a^{x}_t ,\phi(X^{(j)}) = (a^c_t, a^{u}_t)\big)-\mathbb{P}\big(X^{(j)}=a^{x}_k ,\phi(X^{(j)}) = \psi(a^c_k, a^{u}_k)\big)}\sqrt{\frac{2\log(1/\delta)}{nm}} \\
    = & R_{x} \sqrt{\mathbb{P}\Big((X,Z) \in\mathcal{Z}^{x}_{\gamma} \bigcap \big(X^{(j)},\phi(X^{(j)}) \big)\neq \big(a^x_k,\psi(a^c_k, a^{u}_k)\big)  \Big)}\sqrt{\frac{2\log(1/\delta)}{nm}}\\
    \leq & R_{x} \sqrt{\frac{2\log(1/\delta)}{nm}}.
\end{align*}
By taking union bound over every $k\in[T]$, we have for any $\delta>0$, with probability at least $1-\delta$, the following holds for all $k\in[T]$:
\begin{equation}\label{eq10}
    \mathrm{\uppercase\expandafter{\romannumeral3}}_k  \leq R_{x} \sqrt{\frac{2\log(T/\delta)}{n}}.
\end{equation}
Putting (\ref{eq10}) back into $\mathrm{\uppercase\expandafter{\romannumeral3}}$, we have for any $\delta>0$, with probability at least $1-\delta$, 
\begin{align*}
    \mathrm{\uppercase\expandafter{\romannumeral3}} = &\sum_{k=1}^{T} \Big(\mathbb{P}\big(  X^{(j)}=a^{x}_k , \phi(X^{(j)}) = (a^c_k, a^{u}_k)\big) - \frac{\vert \mathcal{U}_k\vert}{nm}  \Big) \ell\big(\psi(a^c_k, a^{u}_k),a^{x}_k\big) \\
    = & \frac{1}{T-1}\sum_{k=1}^{T} \mathrm{\uppercase\expandafter{\romannumeral3}}_k \\
    \leq &  \frac{1}{T-1} \sum_{k=1}^{T}  R_{x} \sqrt{\frac{2\log(T/\delta)}{nm}} \\
    = & \frac{T}{T-1} R_{x} \sqrt{\frac{2\log(T/\delta)}{nm}}.
\end{align*}
For the case of $T=1$, we prove that for any $\delta>0$, with probability at least $1-\delta$,
\begin{align*}
    \mathrm{\uppercase\expandafter{\romannumeral3}} = & \Big(\mathbb{P}\big(X^{(j)}=a^{x}_1 ,\phi(X^{(j)}) = (a^c_1, a^{u}_1)\big) - \frac{\vert \mathcal{U}_k\vert}{nm}  \Big) \ell\big(\psi(a^c_1, a^{u}_1),a^{x}_1\big) \\
    \leq & b_1 \sqrt{\mathbb{P}\big(X^{(j)}=a^{x}_1 ,\phi(X^{(j)}) = (a^c_1, a^{u}_1)\big)}  \sqrt{\frac{2\log(1/\delta)}{nm}} \\
    \leq &R_{x} \sqrt{\frac{2\log(T/\delta)}{nm}}.
\end{align*}
Therefore, for any $T\geq 1$, the following inequality holds 
\begin{equation}\label{eq12}
    \mathrm{\uppercase\expandafter{\romannumeral3}}  \leq 2R_{x} \sqrt{\frac{2\log(T/\delta)}{nm}}.
\end{equation}
Notice that the bounds (\ref{BB}) and (\ref{eq12}) naturally hold for a deterministic $\phi$. We now extend the results to the case of stochastic representation functions $\phi=\{\phi^{(j)}\}_{j=1}^m\in\Phi$. By taking union bounds with (\ref{BB}) and (\ref{eq12}), we have for any $\delta>0$, with probability at least $1-\delta$, the following holds for all $\phi\in\Phi$:
\begin{align}
    \mathrm{\uppercase\expandafter{\romannumeral2}} \leq & \frac{\sqrt{\mathbb{P}(X^{(j)},\phi^{(j)}(X^{(j)})\notin \mathcal{Z}^{x}_{\gamma})} \sum_{i,j\in\mathcal{U}}   \ell(\psi\circ\phi^{(j)}(x^{(j)}_i),x^{(j)}_i)}{\vert \mathcal{U}\vert}\sqrt{\frac{2\log(\vert \Phi\vert/\delta)}{nm}},  \label{13}\\
    \mathrm{\uppercase\expandafter{\romannumeral3}}  \leq& 2R_{x} \sqrt{\frac{2\log(T \vert \Phi\vert/\delta)}{nm}}. \label{14}
\end{align}
From Lemma \ref{lemma0}, we know that 
\begin{equation*}
   T = \vert  \mathcal{Z}^{x}_{\gamma}\vert \leq \exp( H(Z) + c_{\phi} \sqrt{\frac{d\log(\sqrt{nm}/\gamma)}{2}}).
\end{equation*}
Substituting the above inequality into (\ref{14}), we have 
\begin{equation}\label{CC}
    \mathrm{\uppercase\expandafter{\romannumeral3}}  \leq 2R_{x} \sqrt{\frac{2\Big(H(Z) +  c_{\phi} \sqrt{\frac{d\log(\sqrt{nm}/\gamma)}{2}} \Big) +  2\log(\vert \Phi\vert/\delta)}{nm}}.
\end{equation}
Taking union bounds over (\ref{13}) and (\ref{CC}), we have for any $\delta>0$, with probability at least $1-\delta$, the following inequalities hold:
\begin{align*}
    \mathrm{\uppercase\expandafter{\romannumeral2}} \leq & \frac{\sqrt{\mathbb{P}(X^{(j)},\phi^{(j)}(X^{(j)})\notin \mathcal{Z}^{x}_{\gamma})} \sum_{i,j\in\mathcal{U}}   \ell(\psi\circ\phi^{(j)}(x^{(j)}_i),x^{(j)}_i)}{\vert \mathcal{U}\vert}\sqrt{\frac{2\log(2\vert \Phi\vert/\delta)}{nm}},  \\
    \mathrm{\uppercase\expandafter{\romannumeral3}}  \leq & 2R_{x} \sqrt{\frac{2\Big(H(Z) +  c_{\phi} \sqrt{\frac{d\log(\sqrt{nm}/\gamma)}{2}} \Big) +  2\log(2\vert \Phi\vert/\delta)}{nm}}.
\end{align*}
\end{proof}

In the following lemma, we present a general upper bound on the generalization error for the multi-view reconstruction task.
\begin{lemma}\label{general1}
    For any $\gamma>0$, $\delta>0$, and all $\phi\in\Phi$, with probability at least $1-\delta$, the following inequality holds:
    \begin{align*}
        &\mathbb{E}_{X\sim \mathcal{X}} \big[\ell_{\mathrm{avg}}(X;\psi,\phi)\big] - \frac{1}{n}\sum^n_{i=1} \frac{1}{m}\sum_{j=1}^{m} \ell(\psi\circ\phi^{(j)}(x^{(j)}_i),x^{(j)}_i)    \nonumber \\
    \leq & \frac{\gamma R_x}{\sqrt{nm}} + R^s_x \frac{\sqrt{\gamma}}{(nm)^{1/4}} \sqrt{\frac{2\log(2 \vert \Phi\vert /\delta)}{nm}} + 2\sqrt{2} R_{x} \sqrt{\frac{H(C)+\sum_{j=1}^m H(U^{(j)})   +  c_{\phi} \sqrt{\frac{d\log(\sqrt{nm}/\gamma)}{2}}  +  \log(2\vert \Phi\vert/\delta)}{nm}}.
    \end{align*}
\end{lemma}

\begin{proof}

Applying Lemmas \ref{lemma0} and \ref{BC1}, we have for any $\gamma>0$ and $\delta>0$, with probability at least $1-\delta$,
\begin{align*}
    \mathrm{\uppercase\expandafter{\romannumeral2}} \leq & \frac{\sqrt{\mathbb{P}(X^{(j)},\phi^{(j)}(X^{(j)})\notin \mathcal{Z}^{x}_{\gamma})} \sum_{i,j\in\mathcal{U}}   \ell(\psi\circ\phi^{(j)}(x^{(j)}_i),x^{(j)}_i)}{\vert \mathcal{U}\vert}\sqrt{\frac{2\log(2\vert \Phi\vert/\delta)}{nm}}  \nonumber\\
    =  & \frac{\sqrt{\mathbb{P}(X,Z\notin \mathcal{Z}^{x}_{\gamma})} \sum_{i,j\in\mathcal{U}}   \ell(\psi\circ\phi^{(j)}(x^{(j)}_i),x^{(j)}_i)}{\vert \mathcal{U}\vert}\sqrt{\frac{2\log(2\vert \Phi\vert/\delta)}{nm}}  \nonumber\\
    \leq & \frac{\sqrt{\gamma}}{(nm)^{1/4}} \frac{1}{\vert \mathcal{U}\vert} \sum_{i\in\mathcal{I}}R^s_x \sqrt{\frac{2\log(2 \vert \Phi\vert/\delta)}{nm}} \nonumber\\
    = &  R^s_x \frac{\sqrt{\gamma}}{(nm)^{1/4}} \sqrt{\frac{2\log(2 \vert \Phi\vert/\delta)}{nm}}, \label{eq13}
\end{align*}
and 
\begin{align*}
    \mathrm{\uppercase\expandafter{\romannumeral3}}  \leq& 2\sqrt{2} R_{x} \sqrt{\frac{\Big(H(Z) +  c_{\phi} \sqrt{\frac{d\log(\sqrt{nm}/\gamma)}{2}} \Big) +  \log(2\vert \Phi\vert/\delta)}{nm}}.
\end{align*}
According to Lemma \ref{A1}, we obtain that 
\begin{equation*}
    \mathrm{\uppercase\expandafter{\romannumeral1}} \leq \frac{\gamma R_x}{\sqrt{nm}}.
\end{equation*}
Combining the estimation of $\mathrm{\uppercase\expandafter{\romannumeral1}}$, $\mathrm{\uppercase\expandafter{\romannumeral2}}$, and $\mathrm{\uppercase\expandafter{\romannumeral3}}$ with Lemma \ref{lemmaC.1}, we obtain
\begin{align}
    &\mathbb{E}_{X\sim \mathcal{X}} \big[\ell_{\mathrm{avg}}(X;\psi,\phi)\big] - \frac{1}{n}\sum^n_{i=1} \frac{1}{m}\sum_{j=1}^{m} \ell(\psi\circ\phi^{(j)}(x^{(j)}_i),x^{(j)}_i)    \nonumber \\
    \leq & \frac{\gamma R_x}{\sqrt{nm}} + R^s_x \frac{\sqrt{\gamma}}{(nm)^{1/4}} \sqrt{\frac{2\log(2 \vert \Phi\vert /\delta)}{nm}} + 2\sqrt{2} R_{x} \sqrt{\frac{ H(Z) +  c_{\phi} \sqrt{\frac{d\log(\sqrt{nm}/\gamma)}{2}} +  \log(2\vert \Phi\vert/\delta)}{nm}}. \label{b23}
\end{align}
Further, we have the following inequality holds:
\begin{align*}
    H(Z) = H(C,U^{(1)},\ldots,U^{(m)}) \leq H(C)+\sum_{j=1}^m H(U^{(j)}).
\end{align*}
Combining this with the inequality (\ref{b23}), we complete the proof.

\end{proof}

\subsection{Completing the Proof of Theorem \ref{theorem1}}\label{Proof-Thm1}

\begin{restatetheorem}{\ref{theorem1}}[Restate]
        For any $\gamma>0$ and $\delta>0$, with probability at least $1-\delta$:
        \begin{align*}  
            \overline{\mathrm{gen}}_{rec}  \leq \mathcal{K}_1\sqrt{\frac{H(C)+\sum_{j=1}^m H(U^{(j)})  +  \mathcal{K}_2}{nm} }   +\frac{\mathcal{K}_3}{\sqrt{nm}},
        \end{align*}
 where 
        \begin{align*}
           \mathcal{K}_1 = & 2\sqrt{2}R_{x}, \\
           \mathcal{K}_2 = &   c_{\phi} \sqrt{\frac{d\log(\sqrt{nm}/\gamma)}{2}}  + \log(2 /\delta), \\
           \mathcal{K}_3 = & \gamma R_x + R^s_x \frac{\sqrt{\gamma}}{(nm)^{1/4}}\sqrt{2\log(2/\delta)}.
        \end{align*}
\end{restatetheorem}

\begin{proof}
  For deterministic functions $\phi$, we have $\vert \Phi\vert=1$. Substituting $\vert \Phi\vert=1$ into Lemma \ref{general1}, this completes the proof.  
\end{proof}

We proceed to prove the generalization bound of stochastic representation functions $\phi$ by calculating the size $\vert \Phi \vert$ of the hypothesis space $\Phi$. We define $p_\phi(\check{\phi}) = \mathbb{P}(\phi=\check{\phi})$ and for any $\lambda>0$
\begin{equation*}
    C_\lambda = \frac{1}{e^{\lambda H(\phi)}}\sum_{\check{\phi}\in\Phi}p_\phi^{1-\lambda}(\check{\phi}).
\end{equation*}
For any $\epsilon>0$, define the typical subset of $\Phi$:
\begin{equation*}
    \Phi_\epsilon = \{\check{\phi}\in\Phi:-\log p_{\phi}(\check{\phi})-H(\phi)\leq \epsilon\}.
\end{equation*}
The following lemma provides the properties of the typical subset $\Phi_\epsilon$:

\begin{lemma}\label{typicalset}
    For any $\lambda>0$, if we take $\epsilon=\frac{1}{\lambda}\log(C_\lambda/\delta)$, Then
    \begin{equation*}
        \mathbb{P}(\phi\notin \Phi_\epsilon) \leq \delta,
    \end{equation*}
and 
\begin{equation*}
    \vert \Phi_\epsilon\vert \leq \exp\Big(H_{1-\lambda}(\phi)+\frac{1}{\lambda}\log(\frac{1}{\delta})\Big).
\end{equation*}
\end{lemma}
\begin{proof}
    By the definition of $\Phi_\epsilon$ and applying Markov's inequality, we have 
    \begin{align*}
        \mathbb{P}(\phi\notin \Phi_\epsilon)= & \mathbb{P}(-\log p_\phi(\phi)\geq H(\phi)+\epsilon) \\
        =& \mathbb{P}(-\lambda\log p_\phi(\phi)\geq \lambda H(\phi)+\lambda\epsilon) \\
        =& \mathbb{P}\big( p^{-\lambda}_\phi(\phi)\geq \exp(\lambda H(\phi)+\lambda\epsilon)\big) \\
        \leq& \frac{\mathbb{E}_{\phi}[p^{-\lambda}_\phi(\phi)]}{\exp(\lambda H(\phi)+\lambda\epsilon)} \\
        = & \frac{\sum_{\check{\phi}\in\Phi }p^{-\lambda}_\phi(\check{\phi})}{\exp(\lambda H(\phi)+\lambda\epsilon)} = \frac{C_\lambda}{e^{\lambda\epsilon}} := \delta.
    \end{align*}
We now compute the size of $\Phi_\epsilon$: 
\begin{align*}
    -\log p_\phi(\phi)-H(\phi)\leq &\epsilon \\
    -H(\phi)-\epsilon \leq & \log p_\phi(\phi) \\
    \exp(-H(\phi)-\epsilon) \leq & p_\phi(\phi),
\end{align*}
which implies that 
\begin{align*}
    1\geq \mathbb{P}(\phi\in\Phi_\epsilon)\geq \sum_{\check{\phi}\in\Phi_\epsilon} p_\phi(\check{\phi}) \geq \sum_{\check{\phi}\in\Phi_\epsilon} \exp(-H(\phi)-\epsilon) = \vert \Phi_\epsilon\vert \exp(-H(\phi)-\epsilon).
\end{align*}
By taking $C_\lambda/e^{\lambda\epsilon}= \delta $ and solving $\epsilon$, we have
\begin{align*}
    \vert \Phi_\epsilon\vert \leq & \exp\Big(H(\phi)+\frac{1}{\lambda}\log(\frac{C_\lambda}{\delta})\Big) \\
    = & \exp \Big(H(\phi) + \frac{1}{\lambda}\log(\frac{1}{\delta}) +\frac{1}{\lambda} \log\Big(\frac{1}{e^{\lambda H(\phi)}}\sum_{\check{\phi}\in\Phi}p_\phi^{1-\lambda}(\check{\phi})\Big) \Big) \\
    = &  \exp \Big(H(\phi) + \frac{1}{\lambda}\log(\frac{1}{\delta}) - \frac{1}{\lambda}\lambda H(\phi)  +\frac{1}{\lambda} \log\Big(\sum_{\check{\phi}\in\Phi}p_\phi^{1-\lambda}(\check{\phi})\Big) \Big) \\
    = & \exp \Big(H_{1-\lambda}(\phi)+ \frac{1}{\lambda}\log(\frac{1}{\delta}) \Big).
\end{align*}
This completes the proof.
\end{proof}

\subsection{Completing the Proof of Theorem \ref{theorem2}}\label{Proof-Thm2}

\begin{restatetheorem}{\ref{theorem2}}[Restate]
    For any $\gamma>0$ and $\delta>0$, with probability at least $1-\delta$ for all $\phi=\{\phi^{(j)}\}_{j=1}^m \in\Phi$:
     \begin{align*}
         \overline{\mathrm{gen}}_{rec}  \leq \mathcal{K}_1\sqrt{\frac{ \sum_{j=1}^m I(X^{(j)};C)+ I(X^{(j)};U^{(j)})  + H_{1-\lambda}(\phi) +\mathcal{K}_{2,\lambda}}{nm}  } +\frac{\mathcal{K}_{3,\phi}}{\sqrt{nm}}
     \end{align*}
     where 
     \begin{align*}
         \mathcal{K}_1 = & 2\sqrt{2}R_{x}, \\
         \mathcal{K}_{2,\lambda} = &   c_{\phi} \sqrt{\frac{d\log(\sqrt{nm}/\gamma)}{2}}  + \frac{1}{\lambda}\log(\frac{1}{\delta}) +  \log(\frac{4}{\delta}) + H(C|X)+\sum_{j=1}^{m} H(U^{(j)}|X^{(j)}), \\
         \mathcal{K}_{3,\phi} = & \gamma R_x + \sqrt{2}R^s_x \frac{\sqrt{\gamma}}{(nm)^{1/4}}\sqrt{H_{1-\lambda}(\phi) + \frac{1}{\lambda}\log(\frac{1}{\delta}) +  \log(\frac{4}{\delta})}.
     \end{align*}
\end{restatetheorem}

\begin{proof}
    Leveraging Lemma \ref{general1}, if $\phi\in\Phi_\epsilon$, then for any $\gamma>0$ and $\delta>0$, with probability at least $1-\delta$,
    \begin{align}
        &\mathbb{E}_{X\sim \mathcal{X}} \big[\ell_{\mathrm{avg}}(X;\psi,\phi)\big] - \frac{1}{n}\sum^n_{i=1} \frac{1}{m}\sum_{j=1}^{m} \ell(\psi\circ\phi^{(j)}(x^{(j)}_i),x^{(j)}_i)    \nonumber \\
    \leq & \frac{\gamma R_x}{\sqrt{nm}} + R^s_x \frac{\sqrt{\gamma}}{(nm)^{1/4}} \sqrt{\frac{2\log(2 \vert \Phi\vert /\delta)}{nm}} + 2\sqrt{2} R_{x} \sqrt{\frac{H(C)+\sum_{j=1}^m H(U^{(j)})   +  c_{\phi} \sqrt{\frac{d\log(\sqrt{nm}/\gamma)}{2}}  +  \log(2\vert \Phi\vert/\delta)}{nm}}.\label{b24}
    \end{align}
From Lemma \ref{typicalset}, we know that for any $\delta>0$,
\begin{equation*}
    \mathbb{P}(\phi\notin \Phi_\epsilon) \leq \delta,
\end{equation*}
\begin{equation}\label{b26}
\vert \Phi_\epsilon\vert \leq \exp\Big(H_{1-\lambda}(\phi)+\frac{1}{\lambda}\log(\frac{1}{\delta})\Big).
\end{equation}
We thus obtain
\begin{align*}
    \mathbb{P}(\textrm{inequality (\ref{b24}) holds}) \geq &\mathbb{P}(\phi\in\Phi_\epsilon\bigcap\textrm{inequality (\ref{b24}) holds}) \\
    =& \mathbb{P}(\phi\in\Phi_\epsilon) \mathbb{P}(\textrm{inequality (\ref{b24}) holds}|\phi\in\Phi_\epsilon) \\\
    \geq &\mathbb{P}(\phi\in\Phi_\epsilon) (1-\delta) \\
    \geq& (1-\delta)(1-\delta) = 1-2\delta+\delta^2 \geq 1-2\delta.
\end{align*}
By taking $\delta=\delta'/2$, we have for any $\delta'>0$
\begin{equation*}
    \mathbb{P}(\textrm{inequality (\ref{b24}) holds}) \geq 1-\delta'.
\end{equation*}
In other words, for any $\gamma>0$ and $\delta>0$, with probability at least $1-\delta$, the following inequality holds for all $\phi\in\Phi$:
\begin{align}
    &\mathbb{E}_{X\sim \mathcal{X}} \big[\ell_{\mathrm{avg}}(X;\psi,\phi)\big] - \frac{1}{n}\sum^n_{i=1} \frac{1}{m}\sum_{j=1}^{m} \ell(\psi\circ\phi^{(j)}(x^{(j)}_i),x^{(j)}_i)    \nonumber \\
    \leq & \frac{\gamma R_x}{\sqrt{nm}} + R^s_x \frac{\sqrt{\gamma}}{(nm)^{1/4}} \sqrt{\frac{2\log(4 \vert \Phi\vert /\delta)}{nm}} + 2\sqrt{2} R_{x} \sqrt{\frac{H(C)+\sum_{j=1}^m H(U^{(j)})   +  c_{\phi} \sqrt{\frac{d\log(\sqrt{nm}/\gamma)}{2}}  +  \log(4\vert \Phi\vert/\delta)}{nm}}. \label{b25}
\end{align}
According to the chain rule, we get that $H(U^{(j)}) = I(X^{(j)}; U^{(j)}) + H(U^{(j)}|X^{(j)})$ and 
\begin{align}
    H(C) = & I(X;C) + H(C|X) =   I(\{X^{(j)}\}_{j=1}^m;C)  + H(C|X) \nonumber \\
    =& I(X^{(1)};C)+I(\{X^{(j)}\}_{j=2}^m;C|X^{(1)})+ H(C|X) \nonumber \\
    = &I(X^{(1)};C)+ I(\{X^{(j)}\}_{j=2}^m;C) + I(\{X^{(j)}\}_{j=2}^m;X^{(1)} | C) - I(\{X^{(j)}\}_{j=2}^m;X^{(1)}) + H(C|X)\nonumber \\
    = &I(X^{(1)};Z)+ I(\{X^{(j)}\}_{j=2}^m;C)  + H(C|X) \nonumber \\
    = & \ldots \nonumber \\
    = & \sum_{j=1}^m I(X^{(j)};C) + H(C|X). \label{b27}
\end{align}
where the fourth equation is due to the fact that given the common information $C$, $X^{(j)}$ and $X^{(i)}$ are independent of each other for arbitrary $i\neq j$. Substituting (\ref{b26}) and (\ref{b27}) into (\ref{b25}), this completes the proof.
\end{proof}

\section{Proof of Theorems \ref{theorem3}\&\ref{theorem4} [Generalization Bound for Multi-view Classification]} 
Similarly, we define the set of all possible multi-view data and its representation for each class by
\begin{equation*}
    \mathcal{Z}^{x}_y = \{\theta_y(v), \phi\circ\theta_y(v): v\in\mathcal{V} \}.
\end{equation*}
For any $\gamma>0$,  define the typical subset of $\mathcal{Z}^{x}_y$ by 
\begin{equation}\label{subset}
    \mathcal{Z}^{x}_{y,\gamma}  = \Big\{x=(x^{(1)},\ldots,x^{(m)}),z=(c,u^{(1)},\ldots,u^{(m)})\in \mathcal{Z}^{x}_y: -\log p_{z|y}(z) - H(Z_y)\leq c^y_{\phi} \sqrt{\frac{d\log(\sqrt{nm}/\gamma)}{2}} \Big\},
\end{equation}
where $p_{z|y}(z)=\mathbb{P}(Z=z|Y=y)= \mathbb{P}(Z_y=z)$.

\subsection{Properties of the Typical Subset}
\begin{lemma}\label{lemma1}
    For any $\gamma>0$, we have 
    \begin{equation*}
        \mathbb{P}(X_y,Z_y\notin  \mathcal{Z}^{x}_{y,\gamma})\leq \frac{\gamma}{\sqrt{nm}},
    \end{equation*}
and 
   \begin{equation*}
    \vert  \mathcal{Z}^{x}_{y,\gamma}\vert \leq \exp\Big(H(Z_y) +  c^y_{\phi} \sqrt{\frac{d\log(\sqrt{nm}/\gamma)}{2}} \Big).
   \end{equation*}
\end{lemma}
\begin{proof}
    Consider the function $f_y(v) = -\log p_{z|y}(h_y(v))$, where $h_y(v) =\phi\circ\theta_y(v)$. Let $p_v(v)=\mathbb{P}(V = v)$ and $h_y^{-1}(z)=\{v\in\mathcal{V}: h_y(v) = z\}$, we have 
    \begin{align*}
        \mathbb{E}_V [f_y(V)] = & -\sum_{v\in\mathcal{V}} p_v(v)\log p_{z|y}(h_y(v)) \\
         =& -\sum_{z\in \mathcal{Z}^{x}_y}\sum_{v\in h_y^{-1}(z)} p_v(v)\log p_{z|y}(h_y(v))  \\
         = & -\sum_{z\in \mathcal{Z}^{x}_y} \Big(\sum_{v\in h_y^{-1}(z)} p_v(v) \Big)\log p_{z|y}(z)  \\
         = & - \sum_{z\in{\mathcal{Z}_{y}}} p_{z|y}(z)  \log p_{z|y}(z) \\
         =& H(Z_y).
    \end{align*}
By applying McDiarmid's inequality on $f(V) = -\log p_{z|y}(Z)$, we have 
\begin{equation*}
    \mathbb{P}(-\log p_{z|y}(Z)-H(Z_y)\geq \epsilon) \leq \exp\Big(-\frac{2\epsilon^2}{d (c^y_{\phi})^2}\Big).
\end{equation*}
By taking $\delta = \exp\Big(-\frac{2\epsilon^2}{d (c^y_{\phi})^2}\Big)$, we have 
\begin{equation*}
    \epsilon = c^y_{\phi} \sqrt{\frac{d\log(1/\delta)}{2}}.
\end{equation*}
Combining with (\ref{subset}), we select $\delta = \gamma/\sqrt{nm}$ and 
\begin{align*}
    \mathbb{P} (X_y,Z_y\notin  \mathcal{Z}^{x}_{y,\gamma} ) \leq \delta = \frac{\gamma}{\sqrt{nm}}.
\end{align*}
Similar to the proof of Lemma \ref{lemma0}, we can prove that 
\begin{align*}
    \vert  \mathcal{Z}^{x}_{y,\gamma}\vert \exp(-H(Z_y) - \epsilon) \leq 1,
\end{align*}
which implies 
\begin{equation*}
    \vert  \mathcal{Z}^{x}_{y,\gamma}\vert \leq \exp\bigg( H(Z_y) + c^y_{\phi} \sqrt{\frac{d\log(\sqrt{nm}/\gamma)}{2}}\bigg).
\end{equation*}

\end{proof}

\subsection{Decompose the Generalization Error}
Recall the definition of the generalization error of the learning hypothesis $\phi$ for multi-view classification:
\begin{equation}\label{gen2}
    \overline{\mathrm{gen}}_{cls}= \mathbb{E}_{(X,Y)\sim\mathcal{X}\times\mathcal{Y}} [\hat{\ell}_{\mathrm{avg}}((X,Y) ; \hat{\psi},\phi)] - \frac{1}{n}\sum_{i=1}^{n}\frac{1}{m}\sum_{j=1}^{m}\hat{\ell}(\hat{\psi}\circ\phi^{(j)}(x_i^{(j)}), y_i).  
\end{equation}
We further define $T^y = \vert  \mathcal{Z}^{x}_{y,\gamma}\vert$, the elements of the typical subset as $ \mathcal{Z}^{x}_{y,\gamma}=\{(a^{x,y}_1,a_1^{c,y},a^{u,y}_1),\ldots,(a^{x,y}_T,a_T^{c,y},a^{u,y}_T)\}$, and 
\begin{align*}
    \mathcal{U}^y = & \{i\in[n],j\in[m]: x^{(j)}_i,\phi^{(j)}(x^{(j)}_i) \notin \mathcal{Z}^{x}_{y,\gamma}, y_i = y\}, \\
    \mathcal{U}_k^y =& \{i\in[n],j\in[m]: x^{(j)}_i=a^{x,y}_k ,\phi^{(j)}(x^{(j)}_i) = (a^{c,y}_k, a^{u,y}_k), y_i = y\},
\end{align*}
which are dependent on the training set $s$. By using the above notations, we decompose the generalization error as follows.
\begin{lemma}\label{gap2}
    The generalization error satisfies:
    \begin{equation*}
        \mathbb{E}_{(X,Y)\sim\mathcal{X}\times\mathcal{Y}} [\hat{\ell}_{\mathrm{avg}}((X,Y) ; \hat{\psi},\phi)] - \frac{1}{n}\sum_{i=1}^{n}\frac{1}{m}\sum_{j=1}^{m}\hat{\ell}(\hat{\psi}\circ\phi^{(j)}(x_i^{(j)}), y_i)  = \widetilde{\mathrm{\uppercase\expandafter{\romannumeral1}}} + \widetilde{\mathrm{\uppercase\expandafter{\romannumeral2}}}+ \widetilde{\mathrm{\uppercase\expandafter{\romannumeral3}}},
    \end{equation*}
where 
\begin{align*}
    \widetilde{\mathrm{\uppercase\expandafter{\romannumeral1}}} = &  \sum_{y\in\mathcal{Y}}\mathbb{P}\big(Y=y, (X^{(j)},\phi^{(j)}(X^{(j)}))\notin  \mathcal{Z}^{x}_{y,\gamma} \big)  \Big( \mathbb{E}_{(X,Y)\sim\mathcal{X}\times\mathcal{Y},X^{(j)}\sim X}[\hat{\ell}(\hat{\psi}\circ\phi^{(j)}(X^{(j)}), Y)|Y=y,  (X^{(j)},\phi^{(j)}(X^{(j)}))\notin  \mathcal{Z}^{x}_{y,\gamma}]  \\
    & -\frac{1}{\vert \mathcal{U}^y\vert} \sum_{i,j\in \mathcal{U}^y}\hat{\ell}(\hat{\psi}\circ\phi^{(j)}(x^{(j)}_i),y) \Big) \\
    \widetilde{\mathrm{\uppercase\expandafter{\romannumeral2}}} = &   \sum_{y\in\mathcal{Y}}  \frac{1}{\vert \mathcal{U}^y\vert} \Big(\mathbb{P}\big(Y=y, (X^{(j)},\phi^{(j)}(X^{(j)}))\notin  \mathcal{Z}^{x}_{y,\gamma} \big)  -  \frac{\vert \mathcal{U}^y\vert}{nm}\Big)\sum_{i,j\in\mathcal{U}^y}\hat{\ell}(\hat{\psi}\circ\phi^{(j)}(x^{(j)}_i),y)\\
    \widetilde{\mathrm{\uppercase\expandafter{\romannumeral3}}} = &\sum_{y\in\mathcal{Y}} \sum_{k=1}^{T^y} \Big( \mathbb{P}\big(Y=y, X^{(j)}_i=a^{x,y}_k ,\phi^{(j)}(X^{(j)}_i) = (a^{c,y}_k, a^{u,y}_k)\big) - \frac{\vert \mathcal{U}_k^y\vert}{nm} \Big)\hat{\ell}(\hat{\psi}(a^{c,y}_k, a^{u,y}_k),y).
\end{align*}
\end{lemma}

\begin{proof}
     We decompose the population risk as:
    \begin{align}
        &\mathbb{E}_{(X,Y)\sim\mathcal{X}\times\mathcal{Y}} [\hat{\ell}_{\mathrm{avg}}((X,Y) ; \hat{\psi},\phi)]  \nonumber\\
        = &  \sum_{y\in\mathcal{Y}} \mathbb{P}(Y=y)  \mathbb{E}_{(X,Y)\sim\mathcal{X}\times\mathcal{Y}} [\hat{\ell}_{\mathrm{avg}}((X,Y) ; \hat{\psi},\phi)|Y=y] \nonumber\\
        = & \sum_{y\in\mathcal{Y}} \mathbb{P}(Y=y, (X,Z)\notin  \mathcal{Z}^{x}_{y,\gamma})  \mathbb{E}_{(X,Y)\sim\mathcal{X}\times\mathcal{Y}}\big[\hat{\ell}_{\mathrm{avg}}((X,Y) ; \hat{\psi},\phi)|Y=y,(X,Z)\notin  \mathcal{Z}^{x}_{y,\gamma}\big]  \nonumber\\
         & + \sum_{y\in\mathcal{Y}} \mathbb{P}(Y=y, (X,Z)\in  \mathcal{Z}^{x}_{y,\gamma})  \mathbb{E}_{(X,Y)\sim\mathcal{X}\times\mathcal{Y}}\big[\hat{\ell}_{\mathrm{avg}}((X,Y) ; \hat{\psi},\phi)|Y=y,(X,Z)\in  \mathcal{Z}^{x}_{y,\gamma}\big]  \nonumber\\
        = & \sum_{y\in\mathcal{Y}}\mathbb{P}\big(Y=y, (X^{(j)},\phi^{(j)}(X^{(j)}))\notin  \mathcal{Z}^{x}_{y,\gamma} \big) \mathbb{E}_{(X,Y)\sim\mathcal{X}\times\mathcal{Y},X^{(j)}\sim X}[\hat{\ell}(\hat{\psi}\circ\phi^{(j)}(X^{(j)}), Y)|Y=y, (X^{(j)},\phi^{(j)}(X^{(j)}))\notin  \mathcal{Z}^{x}_{y,\gamma}]   \nonumber\\
         & +  \sum_{y\in\mathcal{Y}} \sum_{k=1}^{T^y} \mathbb{P}\big(Y=y, X^{(j)}=a^{x,y}_k ,\phi^{(j)}(X^{(j)}) = (a^{c,y}_k, a^{u,y}_k)\big)\hat{\ell}(\hat{\psi}(a^{c,y}_k, a^{u,y}_k),y). \label{decom1}
    \end{align}
Similarly, the empirical risk can be decomposed as 
\begin{align}
    \frac{1}{n}\sum_{i=1}^{n}\frac{1}{m}\sum_{j=1}^{m}\hat{\ell}(\hat{\psi}\circ\phi^{(j)}(x_i^{(j)}), y_i)  = &  \frac{1}{nm}\sum_{y\in\mathcal{Y}}\Big(\sum_{i,j\in\mathcal{U}^y}\hat{\ell}(\hat{\psi}\circ\phi^{(j)}(x^{(j)}_i),y) + \sum_{k=1}^{T^y}\sum_{i,j\in\mathcal{U}_k^y}\hat{\ell}(\hat{\psi}\circ\phi^{(j)}(x^{(j)}_i),y) \Big) \nonumber\\
    =  &\frac{1}{nm}\sum_{y\in\mathcal{Y}} \sum_{i,j\in\mathcal{U}^y}\hat{\ell}(\hat{\psi}\circ\phi^{(j)}(x^{(j)}_i),y) + \frac{1}{nm} \sum_{y\in\mathcal{Y}}\sum_{k=1}^{T^y}\sum_{i,j\in\mathcal{U}_k^y} \hat{\ell}(\hat{\psi}(a^{c,y}_k, a^{u,y}_k),y)\nonumber\\
    =  &\frac{1}{nm}\sum_{y\in\mathcal{Y}} \sum_{i,j\in\mathcal{U}^y}\hat{\ell}(\hat{\psi}\circ\phi^{(j)}(x^{(j)}_i),y) +  \sum_{y\in\mathcal{Y}}\sum_{k=1}^{T^y}\frac{\vert \mathcal{U}_k^y \vert}{nm}\hat{\ell}(\hat{\psi}(a^{c,y}_k, a^{u,y}_k),y).  \label{decom2}
\end{align}
Putting (\ref{decom1}) and (\ref{decom2}) back into (\ref{gen2}), we can get 
\begin{align*}
    & \mathbb{E}_{(X,Y)\sim\mathcal{X}\times\mathcal{Y}} [\hat{\ell}_{\mathrm{avg}}((X,Y) ; \hat{\psi},\phi)] - \frac{1}{n}\sum_{i=1}^{n}\frac{1}{m}\sum_{j=1}^{m}\hat{\ell}(\hat{\psi}\circ\phi^{(j)}(x_i^{(j)}), y_i)  \\
    = &   \sum_{y\in\mathcal{Y}}\mathbb{P}\big(Y=y, (X^{(j)},\phi^{(j)}(X^{(j)}))\notin  \mathcal{Z}^{x}_{y,\gamma} \big) \mathbb{E}_{(X,Y)\sim\mathcal{X}\times\mathcal{Y},X^{(j)}\sim X}[\hat{\ell}(\hat{\psi}\circ\phi^{(j)}(X^{(j)}), Y)|Y=y, (X^{(j)},\phi^{(j)}(X^{(j)}))\notin  \mathcal{Z}^{x}_{y,\gamma}]   \\
    & - \sum_{y\in\mathcal{Y}} \mathbb{P}\big(Y=y, (X^{(j)},\phi^{(j)}(X^{(j)}))\notin  \mathcal{Z}^{x}_{y,\gamma} \big) \frac{1}{\vert \mathcal{U}^y\vert} \sum_{i,j\in \mathcal{U}^y}\hat{\ell}(\hat{\psi}\circ\phi^{(j)}(x^{(j)}_i),y) \\
    & + \sum_{y\in\mathcal{Y}} \mathbb{P}\big(Y=y, (X^{(j)},\phi^{(j)}(X^{(j)}))\notin  \mathcal{Z}^{x}_{y,\gamma} \big) \frac{1}{\vert \mathcal{U}^y\vert} \sum_{i,j\in \mathcal{U}^y}\hat{\ell}(\hat{\psi}\circ\phi^{(j)}(x^{(j)}_i),y)  - \frac{1}{nm}\sum_{y\in\mathcal{Y}} \sum_{i,j\in\mathcal{U}^y}\hat{\ell}(\hat{\psi}\circ\phi^{(j)}(x^{(j)}_i),y) \\
    & + \sum_{y\in\mathcal{Y}} \sum_{k=1}^{T^y} \mathbb{P}\big(Y=y, X^{(j)}=a^{x,y}_k ,\phi^{(j)}(X^{(j)}) = (a^{c,y}_k, a^{u,y}_k)\big)\hat{\ell}(\hat{\psi}(a^{c,y}_k, a^{u,y}_k),y)- \sum_{y\in\mathcal{Y}}\sum_{k=1}^{T^y}\frac{\vert \mathcal{U}_k^y \vert}{nm}\hat{\ell}(\hat{\psi}(a^{c,y}_k, a^{u,y}_k),y)\\
    = & \sum_{y\in\mathcal{Y}}\mathbb{P}\big(Y=y, (X^{(j)},\phi^{(j)}(X^{(j)}))\notin  \mathcal{Z}^{x}_{y,\gamma} \big)  \Big( \mathbb{E}_{(X,Y)\sim\mathcal{X}\times\mathcal{Y},X^{(j)}\sim X}[\hat{\ell}(\hat{\psi}\circ\phi^{(j)}(X^{(j)}), Y)|Y=y, (X^{(j)},\phi^{(j)}(X^{(j)}))\notin  \mathcal{Z}^{x}_{y,\gamma}]  \\
    & -\frac{1}{\vert \mathcal{U}^y\vert} \sum_{i,j\in \mathcal{U}^y}\hat{\ell}(\hat{\psi}\circ\phi^{(j)}(x^{(j)}_i),y) \Big) \\
    &+  \sum_{y\in\mathcal{Y}}  \frac{1}{\vert \mathcal{U}^y\vert} \Big(\mathbb{P}\big(Y=y, (X^{(j)},\phi^{(j)}(X^{(j)}))\notin  \mathcal{Z}^{x}_{y,\gamma} \big)  -  \frac{\vert \mathcal{U}^y\vert}{nm}\Big)\sum_{i,j\in\mathcal{U}^y}\hat{\ell}(\hat{\psi}\circ\phi^{(j)}(x^{(j)}_i),y)\\
    &+ \sum_{y\in\mathcal{Y}} \sum_{k=1}^{T^y} \Big( \mathbb{P}\big(Y=y, X^{(j)}_i=a^{x,y}_k ,\phi^{(j)}(X^{(j)}_i) = (a^{c,y}_k, a^{u,y}_k)\big) - \frac{\vert \mathcal{U}_k^y\vert}{nm} \Big)\hat{\ell}(\hat{\psi}(a^{c,y}_k, a^{u,y}_k),y),
\end{align*}
which completes the proof.
\end{proof}

\subsection{ Bounding Each Term in Decompositions}
\begin{lemma}\label{AAA}
    For any $\gamma>0$, the following inequality holds:
    \begin{equation*}
        \widetilde{\mathrm{\uppercase\expandafter{\romannumeral1}}} \leq \frac{\gamma R_{x,y}}{\sqrt{nm}}.
    \end{equation*}
Additionally, if $\phi\in\Phi$, for any $\delta>0$, with probability at least $1-\delta$, the following inequalities holds:
\begin{align*}
    \widetilde{\mathrm{\uppercase\expandafter{\romannumeral2}}}\leq &\sum_{y\in\mathcal{Y}} \sqrt{\mathbb{P}(Y=y, (X^{(j)},\phi^{(j)}(X^{(j)}))\notin  \mathcal{Z}^{x}_{y,\gamma})} \frac{\sum_{i,j\in\mathcal{U}^y}\hat{\ell}(\hat{\psi}\circ\phi^{(j)}(x^{(j)}_i),y)}{\vert \mathcal{U}^y\vert} \sqrt{\frac{2\log(2\vert\mathcal{Y}\vert \vert \Phi\vert/\delta)}{nm}}, \\
       \widetilde{\mathrm{\uppercase\expandafter{\romannumeral3}}}\leq & 2R_{x,y}  \sum_{y\in\mathcal{Y}}\sqrt{\mathbb{P}(Y=y)}\sqrt{\frac{2\Big(H(Z_y) +  c^y_{\phi} \sqrt{\frac{d\log(\sqrt{nm}/\gamma)}{2}} \Big) +  2\log(2\vert\mathcal{Y}\vert\vert \Phi\vert/\delta)}{nm}}. 
   \end{align*}
\end{lemma}

\begin{proof}
    From Lemma \ref{lemma1}, we have 
    \begin{equation*}
        \mathbb{P}(X_y,Z_y\notin  \mathcal{Z}^{x}_{y,\gamma})\leq \frac{\gamma}{\sqrt{nm}}.
    \end{equation*}
By applying the above inequality, we get that
\begin{align*}
    \widetilde{\mathrm{\uppercase\expandafter{\romannumeral1}}} = & \sum_{y\in\mathcal{Y}}\mathbb{P}\big(Y=y, (X^{(j)},\phi^{(j)}(X^{(j)}))\notin  \mathcal{Z}^{x}_{y,\gamma} \big)  \Big( \mathbb{E}_{(X,Y)\sim\mathcal{X}\times\mathcal{Y},X^{(j)}\sim X}[\hat{\ell}(\hat{\psi}\circ\phi^{(j)}(X^{(j)}), Y)|Y=y, (X^{(j)},\phi^{(j)}(X^{(j)}))\notin  \mathcal{Z}^{x}_{y,\gamma}]  \\
    & -\frac{1}{\vert \mathcal{U}^y\vert} \sum_{i,j\in \mathcal{U}^y}\hat{\ell}(\hat{\psi}\circ\phi^{(j)}(x^{(j)}_i),y) \Big) \\
    \leq & \sum_{y\in\mathcal{Y}}\mathbb{P}\big(Y=y, X^{(j)},\phi^{(j)}(X^{(j)})\notin  \mathcal{Z}^{x}_{y,\gamma} \big) \mathbb{E}_{(X,Y)\sim\mathcal{X}\times\mathcal{Y},X^{(j)}\sim X}[\hat{\ell}(\hat{\psi}\circ\phi^{(j)}(X^{(j)}), Y)|Y=y, (X^{(j)},\phi^{(j)}(X^{(j)}))\notin  \mathcal{Z}^{x}_{y,\gamma}] \\
    = & \sum_{y\in\mathcal{Y}}\mathbb{P}\big(Y=y, (X,Z)\notin  \mathcal{Z}^{x}_{y,\gamma} \big) \mathbb{E}_{(X,Y)\sim\mathcal{X}\times\mathcal{Y}}[ \hat{\ell}_{\mathrm{avg}}((X,Y);\hat{\psi},\phi)|Y=y, (X,Z)\notin  \mathcal{Z}^{x}_{y,\gamma}] \\
    \leq & \sum_{y\in\mathcal{Y}}\mathbb{P}(Y=y) \frac{\gamma}{\sqrt{nm}} R_{x,y} = \frac{\gamma}{\sqrt{nm}} R_{x,y} .
\end{align*}
Define $p^y_k = \mathbb{P}\big(Y=y, X^{(j)}=a^{x,y}_k ,\phi^{(j)}(X^{(j)}) = (a^{c,y}_k, a^{u,y}_k)\big)$ for $k\in[T^y]$, $p_{T^y+1}=\mathbb{P}(Y=y, (X^{(j)},\phi^{(j)}(X^{(j)}))\notin  \mathcal{Z}^{x}_{y,\gamma})$, and $b^y_k= \hat{\ell}(\hat{\psi}(a^{c,y}_k, a^{u,y}_k),y)$ for $k\in[T^y+1]$, we then have
\begin{equation*}
    \widetilde{\mathrm{\uppercase\expandafter{\romannumeral3}}}_k = \sum_{t=1}^{T^y}\Big(p^y_t-\frac{\vert \mathcal{U}_t^y\vert}{nm}\Big)b^y_t-\Big(p^y_{k}-\frac{\vert \mathcal{U}_k^y\vert}{nm}\Big)b^y_k.
\end{equation*}
Applying Lemma \ref{kawaguchi} with 
\begin{align*}
    k = T^y+1,\quad X=(\vert \mathcal{U}_1^y\vert,\ldots,\vert \mathcal{U}_{T^y}^y\vert,\vert \mathcal{U}^y\vert), \quad p=(p_1,\ldots,p_{T^y+1}),\\
    m = nm, \quad \bar{a}_k=0, \quad \bar{a}_{T^y+1}=0, \quad\textrm{and} \quad \bar{a}_i = b_i \quad \textrm{for any}\quad i\neq k.
\end{align*}
For any $\epsilon>0$ and $k\in[T^y]$, we have
\begin{equation}\label{eq16} 
    \mathbb{P}(\widetilde{\mathrm{\uppercase\expandafter{\romannumeral3}}}_k  \geq \epsilon ) \leq \exp\Big(-\frac{nm\epsilon^2}{2(\sum_{t=1}^{T^y}p^y_t (b^y_t)^2-p^y_k(b^y_k)^2)}\Big).
\end{equation}
Similarly, we obtain
\begin{equation}\label{eq17} 
    \mathbb{P}\Big(p^y_{T+1}-\frac{\vert \mathcal{U}^y\vert}{nm}\geq\epsilon\Big) \leq \exp\Big(-\frac{nm\epsilon^2}{2p^y_{T^y+1}}\Big).
\end{equation}
Take $\delta$ as the right-hand side of (\ref{eq16}) and (\ref{eq17}) respectively, we have 
\begin{equation}\label{eq8}
    \mathbb{P}\bigg(\widetilde{\mathrm{\uppercase\expandafter{\romannumeral3}}}_k \geq \sqrt{\sum_{t=1}^{T^y} p^y_t(b^y_t)^2-p^y_k(b^y_k)^2 }\sqrt{\frac{2\log(1/\delta)}{nm}}\bigg)\leq \delta
\end{equation} 
for any $k\in[T^y]$, and 
\begin{equation}\label{eq9}
    \mathbb{P}\Big(p^y_{T^y+1}-\frac{\vert \mathcal{U}^y\vert}{nm}\geq\sqrt{\frac{2p^y_{T^y+1}\log(1/\delta)}{nm}}\Big)\leq \delta.
\end{equation}
 Take union bounds (\ref{eq8}), (\ref{eq9}) over all $y\in\mathcal{Y}$ and $\phi\in\Phi$, we have for any $\delta>0$, with probability at least $1-\delta$, the following inequalities hold for all $y\in\mathcal{Y}$ and $\phi\in\Phi$:
 \begin{equation}\label{re2}
    \widetilde{\mathrm{\uppercase\expandafter{\romannumeral3}}}_k \leq \sqrt{\sum_{t=1}^{T^y} p^y_t (b^y_t)^2-p^y_k(b^y_k)^2 }\sqrt{\frac{2\log(\vert\mathcal{Y}\vert \vert\Phi\vert/\delta)}{nm}},
\end{equation} 
for any $k\in[T]$, and
\begin{equation}\label{re1}
    \mathbb{P}\big(Y=y, (X^{(j)},\phi^{(j)}(X^{(j)}))\notin  \mathcal{Z}^{x}_{y,\gamma} \big)  -  \frac{\vert \mathcal{U}^y\vert}{nm}\leq \sqrt{\frac{2\mathbb{P}(Y=y, (X^{(j)},\phi^{(j)}(X^{(j)}))\notin  \mathcal{Z}^{x}_{y,\gamma})\log(\vert\mathcal{Y}\vert \vert\Phi\vert/\delta)}{nm}}.
\end{equation}

Substitute (\ref{re1}) into $\widetilde{\mathrm{\uppercase\expandafter{\romannumeral2}}}$ in Lemma \ref{gap2}, we have for any $\delta>0$, with probability at least $1-\delta$,
\begin{align}\label{192}
    \widetilde{\mathrm{\uppercase\expandafter{\romannumeral2}}} = & \sum_{y\in\mathcal{Y}}  \frac{1}{\vert \mathcal{U}^y\vert} \Big(\mathbb{P}\big(Y=y,(X^{(j)},\phi^{(j)}(X^{(j)}))\notin  \mathcal{Z}^{x}_{y,\gamma} \big)  -  \frac{\vert \mathcal{U}^y\vert}{nm}\Big)\sum_{i,j\in\mathcal{U}^y}\hat{\ell}(\hat{\psi}\circ\phi^{(j)}(x^{(j)}_i),y) \nonumber\\
   \leq & \sum_{y\in\mathcal{Y}} \sqrt{\mathbb{P}(Y=y, (X^{(j)},\phi^{(j)}(X^{(j)}))\notin  \mathcal{Z}^{x}_{y,\gamma})} \frac{\sum_{i,j\in\mathcal{U}^y}\hat{\ell}(\hat{\psi}\circ\phi^{(j)}(x^{(j)}_i),y)}{\vert \mathcal{U}^y\vert} \sqrt{\frac{2\log(\vert\mathcal{Y}\vert \vert \Phi\vert/\delta)}{nm}}.
\end{align}
Similarly, by applying the inequality (\ref{re2}), we then have for any $\delta>0$ and $k\in[T^y]$, with probability at least $1-\delta$,
\begin{align*}
    \widetilde{\mathrm{\uppercase\expandafter{\romannumeral3}}}_k \leq & \sqrt{\sum_{t=1}^{T^y} p^y_t (b^y_t)^2-p^y_k(b^y_k)^2 }\sqrt{\frac{2\log(\vert\mathcal{Y}\vert \vert\Phi\vert/\delta)}{nm}} \\
    \leq & R_{x,y} \sqrt{\sum_{t=1}^{T^y} p^y_t-p^y_k }\sqrt{\frac{2\log(\vert\mathcal{Y}\vert\vert \Phi\vert/\delta)}{nm}} \\
    = & R_{x,y} \sqrt{\mathbb{P}\Big(Y=y \bigcap (X,Z)\in  \mathcal{Z}^{x}_{y,\gamma} \bigcap (X^{(j)},\phi^{(j)}(X^{(j)}))\neq (a^{x,y}_k,a^{c,y}_k,a^{u,y}_k) \Big)}\sqrt{\frac{2\log(\vert\mathcal{Y}\vert\vert \Phi\vert/\delta)}{nm}}\\
    \leq & R_{x,y} \sqrt{\mathbb{P}(Y=y)}\sqrt{\frac{2\log(\vert\mathcal{Y}\vert\vert \Phi\vert/\delta)}{nm}}.
\end{align*}
Taking the union bound over all $k\in[T^y]$, for any $\delta>0$, with probability at least $1-\delta$, 
\begin{equation}\label{eq18}
    \widetilde{\mathrm{\uppercase\expandafter{\romannumeral3}}}_k \leq  R_{x,y} \sqrt{\mathbb{P}(Y=y)}\sqrt{\frac{2\log(T^y\vert\mathcal{Y}\vert\vert \Phi\vert/\delta)}{nm}}.
\end{equation}
Putting (\ref{eq18}) back into $\widetilde{\mathrm{\uppercase\expandafter{\romannumeral3}}}$ in Lemma \ref{gap2}, we have that for any $\delta>0$, with probability at least $1-\delta$,
\begin{align}
   \widetilde{\mathrm{\uppercase\expandafter{\romannumeral3}}} = &\sum_{y\in\mathcal{Y}} \sum_{k=1}^{T^y} \Big( \mathbb{P}\big(Y=y, X^{(j)}=a^{x,y}_k ,\phi^{(j)}(X^{(j)}) = (a^{c,y}_k, a^{u,y}_k)\big) - \frac{\vert \mathcal{U}_k^y\vert}{nm} \Big)\hat{\ell}(\hat{\psi}(a^{c,y}_k, a^{u,y}_k),y) \nonumber\\
    =&\sum_{y\in\mathcal{Y}}  \frac{1}{T^y-1} \sum_{k=1}^{T^y} \widetilde{\mathrm{\uppercase\expandafter{\romannumeral3}}}_k  \nonumber\\
    \leq & \sum_{y\in\mathcal{Y}}  \frac{1}{T^y-1} \sum_{k=1}^{T^y} R_{x,y} \sqrt{\mathbb{P}(Y=y)}\sqrt{\frac{2\log(T^y\vert\mathcal{Y}\vert\vert \Phi\vert/\delta)}{nm}} \nonumber\\
    \leq & \sum_{y\in\mathcal{Y}}  \frac{T^y}{T^y-1} R_{x,y} \sqrt{\mathbb{P}(Y=y)}\sqrt{\frac{2\log(T^y\vert\mathcal{Y}\vert\vert \Phi\vert/\delta)}{nm}}  \nonumber\\
    \leq & 2 \sum_{y\in\mathcal{Y}} R_{x,y} \sqrt{\mathbb{P}(Y=y)}\sqrt{\frac{2\log(T^y\vert\mathcal{Y}\vert\vert \Phi\vert/\delta)}{nm}}. \label{b40}
\end{align}
From Lemma \ref{lemma1}, we get that
\begin{equation*}
    T^y =  \vert  \mathcal{Z}^{x}_{y,\gamma}\vert \leq \exp\Big(H(Z_y) +  c^y_{\phi} \sqrt{\frac{d\log(\sqrt{nm}/\gamma)}{2}} \Big).
   \end{equation*}
Combining the above with (\ref{b40}), we can get 
\begin{align}\label{eq191}
 \widetilde{\mathrm{\uppercase\expandafter{\romannumeral3}}}\leq 2 R_{x,y}  \sum_{y\in\mathcal{Y}}\sqrt{\mathbb{P}(Y=y)}\sqrt{\frac{2\Big(H(Z_y) +  c^y_{\phi} \sqrt{\frac{d\log(\sqrt{nm}/\gamma)}{2}} \Big) +  2\log(\vert\mathcal{Y}\vert\vert \Phi\vert/\delta)}{nm}}.  
\end{align}
Finally, taking the union bound over (\ref{192}) and (\ref{eq191}), we have that for any $\delta>0$, with probability at least $1-\delta$, the following inequalities hold:
\begin{align*}
 \widetilde{\mathrm{\uppercase\expandafter{\romannumeral2}}}\leq &\sum_{y\in\mathcal{Y}} \sqrt{\mathbb{P}(Y=y, X^{(j)},\phi^{(j)}(X^{(j)})\notin  \mathcal{Z}^{x}_{y,\gamma})} \frac{\sum_{i,j\in\mathcal{U}^y}\hat{\ell}(\hat{\psi}\circ\phi^{(j)}(x^{(j)}_i),y)}{\vert \mathcal{U}^y\vert} \sqrt{\frac{2\log(2\vert\mathcal{Y}\vert \vert \Phi\vert/\delta)}{nm}}, \\
    \widetilde{\mathrm{\uppercase\expandafter{\romannumeral3}}}\leq & 2R_{x,y}  \sum_{y\in\mathcal{Y}}\sqrt{\mathbb{P}(Y=y)}\sqrt{\frac{2\Big(H(Z_y) +  c^y_{\phi} \sqrt{\frac{d\log(\sqrt{nm}/\gamma)}{2}} \Big) +  2\log(2\vert\mathcal{Y}\vert\vert \Phi\vert/\delta)}{nm}}. 
\end{align*}
This completes the proof.
\end{proof}

In the following lemma, we present the general upper bound for multi-view classification tasks.
\begin{lemma}\label{general2}
    For  any $\gamma>0$, $\delta>0$, and all $\phi\in\Phi$, with probability at least $1-\delta$, the following inequality holds:
    \begin{align*}
        &\mathbb{E}_{(X,Y)\sim\mathcal{X}\times\mathcal{Y}}\big[\hat{\ell}_{\mathrm{avg}}((X,Y) ; \hat{\psi},\phi)\big] - \frac{1}{n}\sum_{i=1}^{n}\frac{1}{m}\sum_{j=1}^{m} \hat{\ell}(\hat{\psi}(\phi^{(j)}(x_i^{(j)})), y_i) \\
       \leq  & \frac{\gamma R_{x,y}}{\sqrt{nm}} +  R^s_{x,y} \frac{\sqrt{\gamma \vert \mathcal{Y}\vert}}{(nm)^{1/4}} \sqrt{\frac{2\log(2\vert\mathcal{Y}\vert \vert \Phi\vert/\delta)}{nm}}\\
       & + 2\sqrt{2} R_{x,y} \sqrt{\vert\mathcal{Y}\vert} \sqrt{\frac{\sum_{j=1}^m I(X^{(j)};C,U^{(j)}|Y)  + H(Z|Y,X^{(1)}) +  c_{\phi} \sqrt{\frac{d\log(\sqrt{nm}/\gamma)}{2}} +  \log(2\vert\mathcal{Y}\vert\vert \Phi\vert/\delta)}{nm}}.
    \end{align*}
\end{lemma}

\begin{proof}
    From Lemma \ref{lemma1}, we know that 
    \begin{equation*}
        \mathbb{P}(X_y,Z_y\notin  \mathcal{Z}^{x}_{y,\gamma}) \leq \frac{\gamma}{\sqrt{nm}}.
    \end{equation*}
 Applying Lemma \ref{AAA} and using Jensen's inequality, we have that for any $\gamma>0$ and $\delta>0$, with probability at least $1-\delta$,
    \begin{align*}
        \widetilde{\mathrm{\uppercase\expandafter{\romannumeral2}}}\leq &  \sum_{y\in\mathcal{Y}} \sqrt{\mathbb{P}(Y=y, (X^{(j)},\phi^{(j)}(X^{(j)}))\notin  \mathcal{Z}^{x}_{y,\gamma})} \frac{\sum_{i,j\in\mathcal{U}^y}\hat{\ell}(\hat{\psi}\circ\phi^{(j)}(x^{(j)}_i),y)}{\vert \mathcal{U}^y\vert} \sqrt{\frac{2\log(2\vert\mathcal{Y}\vert \vert \Phi\vert/\delta)}{nm}}\\
        \leq & \sum_{y\in\mathcal{Y}} \sqrt{\mathbb{P}(Y=y)} \sqrt{\mathbb{P}(X^{(j)},\phi^{(j)}(X^{(j)})\notin  \mathcal{Z}^{x}_{y,\gamma})}  \frac{\sum_{i,j\in\mathcal{U}^y}\hat{\ell}(\hat{\psi}\circ\phi^{(j)}(x^{(j)}_i),y)}{\vert \mathcal{U}^y\vert} \sqrt{\frac{2\log(2\vert\mathcal{Y}\vert \vert \Phi\vert/\delta)}{nm}} \\
        \leq & \frac{\sqrt{\gamma}}{(nm)^{1/4}} \frac{\sum_{i,j\in\mathcal{U}^y}\hat{\ell}(\hat{\psi}\circ\phi^{(j)}(x^{(j)}_i),y)}{\vert \mathcal{U}^y\vert} \sqrt{\vert \mathcal{Y}\vert}\sqrt{ \sum_{y\in\mathcal{Y}} \mathbb{P}(Y=y)} \sqrt{\frac{2\log(2\vert\mathcal{Y}\vert \vert \Phi\vert/\delta)}{nm}} \\
        \leq & R^s_{x,y} \frac{\sqrt{\gamma \vert \mathcal{Y}\vert}}{(nm)^{1/4}} \sqrt{\frac{2\log(2\vert\mathcal{Y}\vert \vert \Phi\vert/\delta)}{nm}},
    \end{align*}
    and
    \begin{align*}
        \widetilde{\mathrm{\uppercase\expandafter{\romannumeral3}}} \leq   & 2R_{x,y}  \sum_{y\in\mathcal{Y}}\sqrt{\mathbb{P}(Y=y)}\sqrt{\frac{2\Big(H(Z_y) +  c^y_{\phi} \sqrt{\frac{d\log(\sqrt{nm}/\gamma)}{2}} \Big) +  2\log(2\vert\mathcal{Y}\vert\vert \Phi\vert/\delta)}{nm}} \\
         \leq & 2 R_{x,y} \sqrt{\vert\mathcal{Y}\vert} \sqrt{\sum_{y\in\mathcal{Y}}\mathbb{P}(Y=y)\frac{2\Big(H(Z_y) +  c^y_{\phi} \sqrt{\frac{d\log(\sqrt{nm}/\gamma)}{2}} \Big) +  2\log(2\vert\mathcal{Y}\vert\vert \Phi\vert/\delta)}{nm}} \\
         = & 2\sqrt{2} R_{x,y} \sqrt{\vert\mathcal{Y}\vert} \sqrt{\frac{H(Z|Y) +  c_{\phi} \sqrt{\frac{d\log(\sqrt{nm}/\gamma)}{2}} +  \log(2\vert\mathcal{Y}\vert\vert \Phi\vert/\delta)}{nm}}.
    \end{align*}
    By applying Lemma \ref{gap2}, we can get 
    \begin{align}
            &\mathbb{E}_{(X,Y)\sim\mathcal{X}\times\mathcal{Y}} [\hat{\ell}_{\mathrm{avg}}((X,Y) ; \hat{\psi},\phi)] - \frac{1}{n}\sum_{i=1}^{n}\frac{1}{m}\sum_{j=1}^{m}\hat{\ell}(\hat{\psi}\circ\phi^{(j)}(x_i^{(j)}), y_i) = \widetilde{\mathrm{\uppercase\expandafter{\romannumeral1}}} + \widetilde{\mathrm{\uppercase\expandafter{\romannumeral2}}}+ \widetilde{\mathrm{\uppercase\expandafter{\romannumeral3}}} \nonumber\\
           \leq  & \frac{\gamma R_{x,y}}{\sqrt{nm}} +  R^s_{x,y} \frac{\sqrt{\gamma \vert \mathcal{Y}\vert}}{(nm)^{1/4}} \sqrt{\frac{2\log(2\vert\mathcal{Y}\vert \vert \Phi\vert/\delta)}{nm}} + 2\sqrt{2} R_{x,y} \sqrt{\vert\mathcal{Y}\vert} \sqrt{\frac{H(Z|Y) +  c_{\phi} \sqrt{\frac{d\log(\sqrt{nm}/\gamma)}{2}} +  \log(2\vert\mathcal{Y}\vert\vert \Phi\vert/\delta)}{nm}}. \label{b42}
        \end{align}
Further using the chain rule, we have 
\begin{align*}
     H(Z|Y) =  &  I(X^{(1)};Z|Y) + H(Z|Y,X^{(1)}) =  I(X^{(1)};C,U^{(1)},\ldots,U^{(m)}|Y) + H(Z|Y,X^{(1)}) \\
     = & I(X^{(1)};C,U^{(1)}|Y) + I(X^{(1)};\{U^{(j)}\}_{j=2}^m|Y,C,U^{(1)}) + H(Z|Y,X^{(1)}) \\
     = & I(X^{(1)};C,U^{(1)}|Y) + H(Z|Y,X^{(1)}) \\
     \leq & \sum_{j=1}^m I(X^{(j)};C,U^{(j)}|Y)  + H(Z|Y,X^{(1)}), 
\end{align*}
Putting the above estimation back into (\ref{b42}), this completes the proof.

    \end{proof}

\subsection{Completing the Proof of Theorem \ref{theorem3}} \label{Proof-Thm3}
\begin{restatetheorem}{\ref{theorem3}}[Restate]
    For any $\gamma>0$ and $\delta>0$, with probability at least $1-\delta$, we have 
    \begin{align*}  
        \overline{\mathrm{gen}}_{cls} \leq  \widetilde{\mathcal{K}}_1 \sqrt{\frac{\sum_{j=1}^m I(X^{(j)};C,U^{(j)}|Y) +  \widetilde{\mathcal{K}}_2}{nm}}+\frac{ \widetilde{\mathcal{K}}_3}{\sqrt{nm}},
    \end{align*}
    where 
    \begin{align*}
    &\widetilde{\mathcal{K}}_1 =  2\sqrt{2} R_{x,y} \sqrt{\vert\mathcal{Y}\vert},\\
    &\widetilde{\mathcal{K}}_2 =  c_{\phi} \sqrt{\frac{d\log(\sqrt{nm}/\gamma)}{2}} +  \log(2\vert\mathcal{Y}\vert/\delta) + H(Z|Y,X^{(1)}), \\
    &\widetilde{\mathcal{K}}_3 = \gamma R_{x,y} + R^s_{x,y} \frac{\sqrt{\gamma \vert \mathcal{Y}\vert}}{(nm)^{1/4}}\sqrt{2\log(2\vert\mathcal{Y}\vert /\delta)} .
\end{align*}
\end{restatetheorem}
\begin{proof}
Since the function $\phi$ is deterministic and independent of the training data $S$, we have $\vert\Phi\vert = 1$. Combining this with Lemma \ref{general2}, this completes the proof.
\end{proof}

\subsection{Completing the Proof of Theorem \ref{theorem4}}\label{Proof-Thm4}

\begin{restatetheorem}{\ref{theorem4}}[Restate]
    For any $\gamma>0$ and $\delta>0$, with probability at least $1-\delta$, we have for all $\phi=\{\phi^{(j)}\}_{j=1}^m \in\Phi$:
    \begin{align*}  
        &\overline{\mathrm{gen}}_{cls} \leq   \widetilde{\mathcal{K}}_1 \sqrt{\frac{\sum_{j=1}^m I(X^{(j)};C,U^{(j)}|Y) + H_{1-\lambda}(\phi) +  \widetilde{\mathcal{K}}_{2,\lambda}}{nm}}+\frac{ \widetilde{\mathcal{K}}_{3,\phi}}{\sqrt{nm}}, 
    \end{align*}
    where
    \begin{align*}
    &\widetilde{\mathcal{K}}_1 =  2\sqrt{2} R_{x,y} \sqrt{\vert\mathcal{Y}\vert},\\
    &\widetilde{\mathcal{K}}_{2,\lambda} =  c_{\phi} \sqrt{\frac{d\log(\sqrt{nm}/\gamma)}{2}} +  \frac{1}{\lambda}\log(\frac{1}{\delta}) +  \log\Big(\frac{4\vert\mathcal{Y}\vert}{\delta}\Big) + H(Z|Y,X^{(1)}), \\
    &\widetilde{\mathcal{K}}_{3,\phi} = \gamma R_{x,y} +\sqrt{2} R^{\tilde{s}}_{x,y} \frac{\sqrt{\gamma \vert \mathcal{Y}\vert}}{(nm)^{1/4}}\sqrt{H_{1-\lambda}(\phi^{\tilde{s}}) + \frac{1}{\lambda}\log(\frac{1}{\delta}) + \log\Big(\frac{4\vert\mathcal{Y}\vert}{\delta}\Big)}.
\end{align*}
\end{restatetheorem}
\begin{proof}
    From Lemma \ref{general2}, if $\phi\in\Phi_\epsilon$, for any $\gamma>0$ and $\delta>0$, with probability at least $1-\delta$, the following inequality holds:
    \begin{align}\label{eq23}
        \overline{\mathrm{gen}}_{cls} \leq  & \frac{\gamma R_{x,y}}{\sqrt{nm}} +  R^s_{x,y} \frac{\sqrt{\gamma \vert \mathcal{Y}\vert}}{(nm)^{1/4}} \sqrt{\frac{2\log(2\vert\mathcal{Y}\vert \vert \Phi_\epsilon\vert/\delta)}{nm}} \nonumber\\
        & + 2\sqrt{2} R_{x,y} \sqrt{\vert\mathcal{Y}\vert} \sqrt{\frac{\sum_{j=1}^m I(X^{(j)};C,U^{(j)}|Y)  + H(Z|Y,X^{(1)}) +  c_{\phi} \sqrt{\frac{d\log(\sqrt{nm}/\gamma)}{2}} +  \log(2\vert\mathcal{Y}\vert\vert \Phi_\epsilon\vert/\delta)}{nm}}. 
    \end{align}
Analyzing analogously to the proof of Theorem \ref{theorem2}, we apply Lemma \ref{typicalset} and obtain
\begin{equation*}
    \mathbb{P}(\phi\notin \Phi_\epsilon) \leq \delta,
\end{equation*}
\begin{equation*}
    \mathbb{P}(\textrm{inequality (\ref{eq23}) holds}) \geq \mathbb{P}(\phi\in\Phi_\epsilon\bigcap\textrm{inequality (\ref{eq23}) holds})  \geq 1-2\delta.
\end{equation*}
We select $\delta=\delta'/2$, and then for any $\gamma>0$ and $\delta>0$, with probability at least $1-\delta$, the following inequality holds:
\begin{align}\label{eq24}
    \overline{\mathrm{gen}}_{cls} \leq  & \frac{\gamma R_{x,y}}{\sqrt{nm}} +  R^{\tilde{s}}_{x,y} \frac{\sqrt{\gamma \vert \mathcal{Y}\vert}}{(nm)^{1/4}} \sqrt{\frac{2\log(4\vert\mathcal{Y}\vert \vert \Phi_\epsilon\vert/\delta)}{nm}} \nonumber\\
        & + 2\sqrt{2} R_{x,y} \sqrt{\vert\mathcal{Y}\vert} \sqrt{\frac{\sum_{j=1}^m I(X^{(j)};C,U^{(j)}|Y)  + H(Z|Y,X^{(1)}) +  c_{\phi} \sqrt{\frac{d\log(\sqrt{nm}/\gamma)}{2}} +  \log(4\vert\mathcal{Y}\vert\vert \Phi_\epsilon\vert/\delta)}{nm}}. 
\end{align}
Using Lemma \ref{typicalset}, we have
\begin{align}
    \log( 4\vert \Phi_\epsilon\vert \vert\mathcal{Y}\vert/\delta)=  &  \log(\vert \Phi_\epsilon\vert) +  \log(4\vert\mathcal{Y}\vert/\delta)\leq  H_{1-\lambda}(\phi)+\frac{1}{\lambda}\log(\frac{1}{\delta}) +  \log(\frac{4\vert\mathcal{Y}\vert}{\delta}). \label{eq25}
\end{align}
Plugging (\ref{eq25}) back into (\ref{eq24}), this completes the proof.
\end{proof}

\section{Proof of Theorems \ref{theorem5}\&\ref{theorem6} [Data-dependent Generalization Bounds]}

\subsection{LOO Settings}\label{Proof-Thm5}

\begin{restatetheorem}{\ref{theorem5}}[Restate]
    If $\lambda \rightarrow 0$, for any $\delta>0$ and all $\phi=\{\phi^{(j)}\}_{j=1}^m \in\Phi$, with probability at least $1-\delta$, we have
    \begin{align*}
        \Delta_{\mathrm{loo}}  \leq \mathcal{K}^{u}_{1} \sqrt{\sum_{j=1}^{m} I(X^{(j)};C,U^{(j)}|Y)  +   I(\phi; U) + \mathcal{K}^{u}_{2,\lambda}},
    \end{align*}
    where 
\begin{align*}
    \mathcal{K}^{u}_{1} = &\sqrt{2}\sigma_u, \\
    \mathcal{K}^{u}_{2,\lambda} = &c_{\phi} \sqrt{\frac{d\log(\sqrt{nm}/\gamma)}{2}} + \frac{1}{\lambda}\log(\frac{1}{\delta}) + \log(\frac{4\vert \mathcal{Y}\vert}{\delta}) + H(Z|X^{(1)},Y)+H(\phi|U) ,
\end{align*}
by assuming that $ \Delta_{\mathrm{loo}}$ is $\sigma_u$-subgaussian w.r.t $\sigma_u \in[0,R^{\tilde{s}}_{x,y}]$.
\end{restatetheorem}
\begin{proof}
Assume that $(X,Z)=(x,z)$ for some $(x,z)=\{(x_i,z_i)=(x_i,\phi(x_i))\}_{i=1}^{n+1}\in\mathcal{Z}^{x}_{y,\gamma}$ and $\phi\in\Phi_\epsilon$, we then have 
 \begin{align*}
    \Delta_{\mathrm{loo}}= &\hat{\ell}_{\mathrm{avg}}\big( (x_U,y_U) ; \hat{\psi},\phi \big)-\frac{1}{n}\sum_{i\neq U}\hat{\ell}_{\mathrm{avg}}\big( (x_i,y_i) ; \hat{\psi},\phi \big) \\
    =& \hat{\ell}_{\mathrm{avg}}\big( (x_U,y_U) ; \hat{\psi},\phi \big) -\frac{1}{n}\sum_{i=1}^{n+1}\hat{\ell}_{\mathrm{avg}}\big( (x_i,y_i) ; \hat{\psi},\phi \big)+ \frac{1}{n} \hat{\ell}_{\mathrm{avg}}\big( (x_U,y_U) ; \hat{\psi},\phi \big) \\
    =& \frac{n+1}{n}\hat{\ell}_{\mathrm{avg}}\big( (x_U,y_U) ; \hat{\psi},\phi \big) -  \frac{n+1}{n} \frac{1}{n+1} \sum_{i=1}^{n+1} \hat{\ell}_{\mathrm{avg}}\big( (x_i,y_i) ; \hat{\psi},\phi \big)  \\
    =& \frac{n+1}{n} \Big(\hat{\ell}_{\mathrm{avg}}\big( (x_U,y_U) ; \hat{\psi},\phi \big) -\frac{1}{n+1} \sum_{i=1}^{n+1}\hat{\ell}_{\mathrm{avg}}\big( (x_i,y_i) ; \hat{\psi},\phi \big)\Big).
 \end{align*} 
It is easy to prove $\Delta_{\mathrm{loo}}=0$. When $\hat{\ell}_{\mathrm{avg}}\big( (x_U,y_U) ; \hat{\psi},\phi \big)=\sup_{i\in[n+1]}\hat{\ell}_{\mathrm{avg}}\big( (x_i,y_i) ; \hat{\psi},\phi \big)$ and $\hat{\ell}_{\mathrm{avg}}\big( (x_i,y_i) ; \hat{\psi},\phi \big)=0$ for any $i\neq u$, $\Delta_{\mathrm{loo}}$ takes the maximum value $R^{\tilde{s}}_{x,y}$. Similarly, one can prove that $\Delta_{\mathrm{loo}}\geq -R^{\tilde{s}}_{x,y} $. This implies that $\Delta_{\mathrm{loo}}$ is $R^{\tilde{s}}_{x,y}$-subgaussian. Assume that  $(X,Z)=(x,z)$ and let $\Delta_{\mathrm{loo}}$ be $\sigma_u$-subgaussian w.r.t $U$, where $\sigma_u\in[0,R^{\tilde{s}}_{x,y}]$, then for any $\epsilon>0$
 \begin{equation*}
    \mathbb{P}_U(\Delta_{\mathrm{loo}} \geq \epsilon)\leq \exp\Big(-\frac{\epsilon^2}{2\sigma_u^2}\Big).
 \end{equation*}
That is, for any $\delta>0$ and $(x,z)\in \mathcal{Z}^{x}_{y,\gamma}$, if $(X,Z)=(x,z)$, then with probability at least $1-\delta$, 
\begin{equation}\label{eq26}
    \Delta_{\mathrm{loo}} \leq \sigma_u\sqrt{2\log(1/\delta)}.
\end{equation}
From Lemma \ref{lemma0} and \ref{typicalset},  we know that for any $\delta>0$
\begin{align}
     \mathbb{P}(X,Z\notin  \mathcal{Z}^{x}_{y,\gamma})\leq \frac{\gamma}{\sqrt{nm}}, &\quad  \vert  \mathcal{Z}^{x}_{y,\gamma}\vert \leq \exp\Big(H(Z_y) +  c_{\phi} \sqrt{\frac{d\log(\sqrt{nm}/\gamma)}{2}} \Big), \label{unio1}
    \\
    \mathbb{P}(\phi\notin \Phi_\epsilon) \leq \delta, &\quad \vert \Phi_\epsilon\vert \leq \exp\Big(H_{1-\lambda}(\phi)+\frac{1}{\lambda}\log(\frac{1}{\delta})\Big). \label{unio2}
\end{align}
Taking the union bound of (\ref{eq26}) over all $(x,z)\in \mathcal{Z}^{x}_{y,\gamma}$ and $\phi\in\Phi_\epsilon$, we have for any $\delta>0$, with probability at least $1-\delta$, the following inequality holds simultaneously if $(X,Z)=(x,z)$:
\begin{equation}\label{unio3}
    \Delta_{\mathrm{loo}}  \leq \sigma_u\sqrt{2\log(\vert \Phi_\epsilon\vert \vert  \mathcal{Z}^{x}_{y,\gamma}\vert /\delta)}.
\end{equation}
Again, take the union bound over $y\in\mathcal{Y}$, (\ref{unio1}), (\ref{unio2}), and (\ref{unio3}), then for any $\delta>0$, with probability at least $1-\delta$,
\begin{align}
    \Delta_{\mathrm{loo}}  \leq & \sigma_u\sqrt{2\log(4\vert \Phi_\epsilon\vert \vert  \mathcal{Z}^{x}_{y,\gamma}\vert \vert \mathcal{Y}\vert /\delta)} \nonumber\\
    \leq & \sqrt{2}\sigma_u\sqrt{H(Z_y) +   H_{1-\lambda}(\phi)+\frac{1}{\lambda}\log(\frac{1}{\delta}) + c_{\phi} \sqrt{\frac{d\log(\sqrt{nm}/\gamma)}{2}} + \log(\frac{4\vert \mathcal{Y}\vert}{\delta})} \nonumber\\
    \leq & \sqrt{2}\sigma_u\sqrt{\sum_{y\in\mathcal{Y}}\mathbb{P}(Y=y) \Big( H(Z_y) +   H_{1-\lambda}(\phi)+\frac{1}{\lambda}\log(\frac{1}{\delta}) + c_{\phi} \sqrt{\frac{d\log(\sqrt{nm}/\gamma)}{2}} + \log(\frac{4\vert \mathcal{Y}\vert}{\delta})\Big)}\nonumber \\
    = & \sqrt{2}\sigma_u\sqrt{ H(Z|Y) +   H_{1-\lambda}(\phi)+\frac{1}{\lambda}\log(\frac{1}{\delta}) + c_{\phi} \sqrt{\frac{d\log(\sqrt{nm}/\gamma)}{2}} + \log(\frac{4\vert \mathcal{Y}\vert}{\delta})}.\label{eq32}
\end{align}
Similarly, we can obtain that
\begin{align}
    H(Z|Y) = & I(X^{(1)};Z|Y) + H(Z|X^{(1)},Y) = I(X^{(1)};C,U^{(1)},\ldots,U^{(m)}|Y) + H(Z|X^{(1)},Y) \nonumber\\
    = & I(X^{(1)};C,U^{(1)}|Y) + I(X^{(1)};\{U^{(j)}\}_{j=2}^m|C,U^{(1)},Y)+ H(Z|X^{(1)},Y) \nonumber\\
    = & I(X^{(1)};C,U^{(1)}|Y) + H(Z|X^{(1)},Y) \nonumber\\
    \leq & \sum_{j=1}^{m} I(X^{(j)};C,U^{(j)}|Y) + H(Z|X^{(1)},Y), \label{b51}
\end{align}
and when $\lambda\rightarrow 0$, we have 
\begin{equation}\label{b52}
    H_{1-\lambda}(\phi) \approx H(\phi) = I(\phi;U) + H(\phi|U) 
\end{equation}
Combining the inequalities (\ref{b51}), (\ref{b52}) with (\ref{eq32}), we complete the proof.
\end{proof}

\subsection{Supersample Settings}\label{Proof-Thm6}

\begin{restatetheorem}{\ref{theorem6}}[Restate]
    For any $\lambda\in(0,1)$, $\delta>0$, and all $\phi=\{\phi^{(j)}\}_{j=1}^m \in\Phi$, with probability at least $1-\delta$, we have
    \begin{align*}
        \Delta_{\mathrm{sup}} \leq & \mathcal{K}^{\tilde{u}}_{1}\sqrt{\frac{\sum_{j=1}^{m} I(X^{(j)};C,U^{(j)}|Y)  +   I(\phi; \tilde{U}) + \mathcal{K}^{\tilde{u}}_{2,\lambda}}{nm}},
\end{align*}
where $\Delta\hat{\ell}^{i,j}_{\psi,\phi}=\hat{\ell}(\hat{\psi}\circ\phi^{(j)}(x^{(j)}_{i,1}),y_{i,1})-\hat{\ell}(\hat{\psi}\circ\phi^{(j)}(x^{(j)}_{i,0}),y_{i,0})$,
\begin{align*}
    \mathcal{K}^{\tilde{u}}_{1}  = &\sqrt{\frac{1}{nm}\sum_{i,j=1}^{n,m}(\Delta\hat{\ell}^{i,j}_{\psi,\phi})^2}, \\
    \mathcal{K}^{\tilde{u}}_{2,\lambda} = &   c_{\phi} \sqrt{\frac{d\log(\sqrt{nm}/\gamma)}{2}} + \frac{1}{\lambda}\log(\frac{1}{\delta}) + \log(\frac{4\vert \mathcal{Y}\vert}{\delta}) + H(Z|X^{(1)},Y) + H(\phi|\tilde{U}).
\end{align*}

\end{restatetheorem}

\begin{proof}
For some $(x_u,z_u)=\{(x_{i,0},\phi(x_{i,0})), (x_{i,1},\phi(x_{i,1}))\}_{i=1}^{n}\in \mathcal{Z}^{x}_{y,\gamma}$, and $\phi\in\Phi_\epsilon$, we have 
    \begin{align*}
        \Delta_{\mathrm{sup}}=&\frac{1}{n}\sum_{i=1}^n\hat{\ell}_{\mathrm{avg}}\big( (x_{i,1-\tilde{U}_i},y_{i,1-\tilde{U}_i}) ; \hat{\psi},\phi \big)-\frac{1}{n}\sum_{i=1}^n \hat{\ell}_{\mathrm{avg}}\big( (x_{i,\tilde{U}_i},y_{i,\tilde{U}_i}) ; \hat{\psi},\phi \big) \\
         =& \frac{1}{n}\sum_{i=1}^{n} \frac{1}{m}\sum_{j=1}^{m} \hat{\ell}(\hat{\psi}\circ\phi^{(j)}(x^{(j)}_{i,1-\tilde{U}_i}),y_{i,1-\tilde{U}_{i}}) - \frac{1}{n}\sum_{i=1}^{n}\frac{1}{m}\sum_{j=1}^{m}\hat{\ell}\big(\hat{\psi}\circ\phi^{(j)}(x^{(j)}_{i,\tilde{U}_i}),y_{i,\tilde{U}_i}\big) \\
         =& \frac{1}{n}\sum_{i=1}^{n} \frac{1}{m}\sum_{j=1}^{m} (-1)^{\tilde{U}_i}\Delta\hat{\ell}^{i,j}_{\psi,\phi},
    \end{align*}
where $\Delta\hat{\ell}^{i,j}_{\psi,\phi}=\hat{\ell}(\hat{\psi}\circ\phi^{(j)}(x^{(j)}_{i,1}),y_{i,1})-\hat{\ell}(\hat{\psi}\circ\phi^{(j)}(x^{(j)}_{i,0}),y_{i,0})$. Note that $\mathbb{E}_{\tilde{U}_i}[(-1)^{\tilde{U}_i}]=0$, by using McDiarmid's inequality with $f(\tilde{U})=\Delta_{\mathrm{sup}}$, we have for any $\epsilon>0$,
\begin{equation*}
    \mathbb{P}_{\tilde{U}}(\Delta_{\mathrm{sup}} \geq \epsilon)\leq \exp\Big(-\frac{2\epsilon^2}{\sum_{i,j=1}^{n,m}(2\Delta\hat{\ell}^{i,j}_{\psi,\phi}/nm)^2}\Big).
\end{equation*}
Therefore, for any $\delta>0$, $(x_u,z_u) \in \mathcal{Z}^{x}_{y,\gamma}$, and $\phi\in\Phi_\epsilon$, with probability at least $1-\delta$,
\begin{equation*}
    \Delta_{\mathrm{sup}} \leq \sqrt{\frac{1}{nm}\sum_{i,j=1}^{n,m}(\Delta\hat{\ell}^{i,j}_{\psi,\phi})^2}\sqrt{\frac{2\log(1/\delta)}{nm}}.
\end{equation*}
According to  Lemmas \ref{lemma1} and \ref{typicalset}, we know that for any $\delta>0$
\begin{align}
    \mathbb{P}((X,Z) \notin  \mathcal{Z}^{x}_{y,\gamma})\leq \frac{\gamma}{\sqrt{nm}}, &\quad  \vert  \mathcal{Z}^{x}_{y,\gamma}\vert \leq \exp\Big(H(Z_y) +  c_{\phi} \sqrt{\frac{d\log(\sqrt{nm}/\gamma)}{2}} \Big), \label{unio5}
    \\
    \mathbb{P}(\phi \notin \Phi_\epsilon) \leq \delta, &\quad \vert \Phi_\epsilon\vert \leq \exp\Big(H_{1-\lambda}(\phi)+\frac{1}{\lambda}\log(\frac{1}{\delta})\Big). \label{unio6}
\end{align}
Taking union bounds over all $(x_u,z_u)\in   \mathcal{Z}^{x}_{y,\gamma}$ and $\phi \in \Phi_\epsilon$, we have for any $\delta>0$, with probability at least $1-\delta$, the following holds simultaneously
\begin{align}\label{unio7}
    \Delta_{\mathrm{sup}} \leq & \sqrt{\frac{1}{nm}\sum_{i,j=1}^{n,m}(\Delta\hat{\ell}^{i,j}_{\psi,\phi})^2}\sqrt{\frac{2\log(\vert  \mathcal{Z}^{x}_{\gamma}\vert \vert \Phi_\epsilon\vert /\delta)}{nm}}. 
\end{align}
Again, take the union bound over $y\in\mathcal{Y}$, (\ref{unio5}), (\ref{unio6}), and (\ref{unio7}), we have for any $\delta>0$, with probability at least $1-\delta$,
\begin{align}
    \Delta_{\mathrm{sup}}  \leq& \sqrt{\frac{1}{nm}\sum_{i,j=1}^{n,m}(\Delta\hat{\ell}^{i,j}_{\psi,\phi})^2}\sqrt{\frac{2\log(4\vert  \mathcal{Z}^{x}_{\gamma}\vert \vert \Phi_\epsilon\vert  \vert \mathcal{Y}\vert/\delta)}{nm}} \nonumber\\
    \leq &   \sqrt{\frac{2}{nm}\sum_{i,j=1}^{n,m}(\Delta\hat{\ell}^{i,j}_{\psi,\phi})^2} \sqrt{\frac{H(Z_y) +   H_{1-\lambda}(\phi)+\frac{1}{\lambda}\log(\frac{1}{\delta})+ \log(\frac{4\vert \mathcal{Y}\vert}{\delta}) + c_{\phi} \sqrt{\frac{d\log(\sqrt{nm}/\gamma)}{2}} }{nm}} \nonumber\\
    \leq & \sqrt{\frac{2}{nm}\sum_{i,j=1}^{n,m}(\Delta\hat{\ell}^{i,j}_{\psi,\phi})^2} \sqrt{ \sum_{y\in\mathcal{Y}}\mathbb{P}(Y=y)\frac{H(Z_y) +   H_{1-\lambda}(\phi)+\frac{1}{\lambda}\log(\frac{1}{\delta})+ \log(\frac{4\vert \mathcal{Y}\vert}{\delta}) + c_{\phi} \sqrt{\frac{d\log(\sqrt{nm}/\gamma)}{2}} }{nm}}\nonumber \\
    = &   \sqrt{\frac{2}{nm}\sum_{i,j=1}^{n,m}(\Delta\hat{\ell}^{i,j}_{\psi,\phi})^2} \sqrt{\frac{H(Z|Y) +   H_{1-\lambda}(\phi)+\frac{1}{\lambda}\log(\frac{1}{\delta})+ \log(\frac{4\vert \mathcal{Y}\vert}{\delta}) + c_{\phi} \sqrt{\frac{d\log(\sqrt{nm}/\gamma)}{2}} }{nm}}. \label{b50}
\end{align}
Similarly, if $\lambda\rightarrow 0$, we have 
\begin{equation}\label{b54}
    H_{1-\lambda}(\phi) \approx H(\phi) = I(\phi;\tilde{U}) + H(\phi|\tilde{U}).
\end{equation}
Putting the inequality (\ref{b51}) and (\ref{b54}) back into (\ref{b50}), this completes the proof.
\end{proof}
\section{Proof of Theorem \ref{theorem7} [Fast-rate Generalization Bound]}\label{Proof-Thm7}

\begin{restatetheorem}{\ref{theorem7}}[Restate]
    For any $\lambda\in(0,1)$, $\delta>0$, $\beta\in(0,\log 2)$, and $ \xi\geq\frac{\log (2-e^{2\beta R^{\tilde{s}}_{x,y}})}{2\beta R^{\tilde{s}}_{x,y}}-1$, with probability at least $1-\delta$, for all $\phi=\{\phi^{(j)}\}_{j=1}^m \in\Phi$, we have
    \begin{align*}
        \overline{\mathrm{gen}}_{cls}\leq  \xi \widehat{L}_{cls}   + \frac{\sum_{j=1}^{m} I(X^{(j)};C,U^{(j)}|Y)   + H_{1-\lambda}(\phi) +\hat{\mathcal{K}}}{nm\beta},
    \end{align*}
where $\hat{\mathcal{K}} = c_{\phi} \sqrt{\frac{d\log(\sqrt{nm}/\gamma)}{2}} +\frac{1}{\lambda}\log(\frac{1}{\delta}) +\log (\frac{4 \vert \mathcal{Y}\vert}{\delta}) + H(Z|Y,X^{(1)})$. In the interpolating setting, i.e., $\widehat{L}_{cls} = 0$, we further have 
\begin{equation*}
    \overline{\mathrm{gen}}_{cls}  \leq \frac{\sum_{j=1}^{m} I(X^{(j)};C,U^{(j)}|Y)   + H_{1-\lambda}(\phi) +\hat{\mathcal{K}}}{nm\beta}.
\end{equation*}
\end{restatetheorem}

\begin{proof}
   Assume that $(X,Z)=(x,z)$. For some $(x,z)=\{(x_{i,0},\phi(x_{i,0})), (x_{i,1},\phi(x_{i,1}))\}_{i=1}^{n}\in \mathcal{Z}^{x}_{y,\gamma}$, and $\phi\in\Phi_\epsilon$, 
    \begin{align}
        &  \mathbb{P}\Bigg(\frac{1}{n}\sum_{i=1}^{n} \frac{1}{m}\sum_{j=1}^{m} \Big(\hat{\ell}(\hat{\psi}\circ\phi^{(j)}(x^{(j)}_{i,1-\tilde{U}_i}),y_{i,1-\tilde{U}_{i}}) - (1+\xi)\hat{\ell}\big(\hat{\psi}\circ\phi^{(j)}(x^{(j)}_{i,\tilde{U}_i}),y_{i,\tilde{U}_i}\big)\Big) \geq t \Bigg) \nonumber\\
        = & \mathbb{P}\Bigg(\frac{1}{n}\sum_{i=1}^{n} \frac{1}{m}\sum_{j=1}^{m} \Big(1+\frac{\xi}{2}\Big) \Big(\hat{\ell}\big(\hat{\psi}\circ\phi^{(j)}(x^{(j)}_{i,1-\tilde{U}_i}),y_{i,1-\tilde{U}_{i}}\big) - \hat{\ell}\big(\hat{\psi}\circ\phi^{(j)}(x^{(j)}_{i,\tilde{U}_i}),y_{i,\tilde{U}_i}\big)\Big) \nonumber\\
        &- \frac{\xi}{2} \hat{\ell}(\hat{\psi}\circ\phi^{(j)}(x^{(j)}_{i,1-\tilde{U}_i}),y_{i,1-\tilde{U}_{i}}) - \frac{\xi}{2} \hat{\ell}\big(\hat{\psi}\circ\phi^{(j)}(x^{(j)}_{i,\tilde{U}_i}),y_{i,\tilde{U}_i}\big) \geq t \Bigg) \nonumber\\
        = & \mathbb{P}\Bigg(\frac{1}{2n}\sum_{i=1}^{n} \frac{1}{m}\sum_{j=1}^{m} \Big( (-1)^{\tilde{U}_{i}} (2 + \xi) \hat{\ell}\big(\hat{\psi}\circ\phi^{(j)}(x^{(j)}_{i,1}),y_{i,1}\big) - \xi\hat{\ell}\big(\hat{\psi}\circ\phi^{(j)}(x^{(j)}_{i,1}),y_{i,1}\big)  \Big)\nonumber\\
        &+\frac{1}{2n}\sum_{i=1}^{n} \frac{1}{m}\sum_{j=1}^{m} \Big( -(-1)^{\tilde{U}_{i}} (2 + \xi) \hat{\ell}\big(\hat{\psi}\circ\phi^{(j)}(x^{(j)}_{i,0}),y_{i,0}\big) - \xi\hat{\ell}\big(\hat{\psi}\circ\phi^{(j)}(x^{(j)}_{i,0}),y_{i,0}\big) \Big) \geq t \Bigg) \nonumber\\
        = & \mathbb{P}\Bigg( \sup_{I\in\{0,1\}}\Bigg\{\frac{1}{2n}\sum_{i=1}^{n} \frac{1}{m}\sum_{j=1}^{m} \Big( (-1)^{\tilde{U}_{i}} (2 + \xi) - \xi \Big) \hat{\ell}\big(\hat{\psi}\circ\phi^{(j)}(x^{(j)}_{i,I}),y_{i,I}\big)\Bigg\} \nonumber\\
        &+ \sup_{I\in\{0,1\}}\Bigg\{\frac{1}{2n}\sum_{i=1}^{n} \frac{1}{m}\sum_{j=1}^{m} \Big( -(-1)^{\tilde{U}_{i}} (2 + \xi) - \xi \Big) \hat{\ell}\big(\hat{\psi}\circ\phi^{(j)}(x^{(j)}_{i,1-I}),y_{i,1-I}\big) \Bigg\} \geq t \Bigg) \nonumber\\
        \leq &  \inf_{\gamma\in(0,1)}\mathbb{P}\Bigg( \sup_{I\in\{0,1\}}\Bigg\{\frac{1}{2n}\sum_{i=1}^{n} \frac{1}{m}\sum_{j=1}^{m} \Big( (-1)^{\tilde{U}_{i}} (2 + \xi) - \xi \Big) \hat{\ell}\big(\hat{\psi}\circ\phi^{(j)}(x^{(j)}_{i,I}),y_{i,I}\big) \Bigg\} \geq \gamma t \Bigg)\nonumber\\
        &+ \mathbb{P}\Bigg(  \sup_{I\in\{0,1\}}\Bigg\{\frac{1}{2n}\sum_{i=1}^{n} \frac{1}{m}\sum_{j=1}^{m} \Big( -(-1)^{\tilde{U}_{i}} (2 + \xi) - \xi \Big) \hat{\ell}\big(\hat{\psi}\circ\phi^{(j)}(x^{(j)}_{i,1-I}),y_{i,1-I}\big) \Bigg\} \geq (1-\gamma)t \Bigg). \label{b56}
    \end{align}
For any $I\in\{0,1\}$, $t>0$, and $\beta>0$, by using Markov's inequality, we then have 
\begin{align*}
    & \mathbb{P}\Bigg(\frac{1}{2n}\sum_{i=1}^{n} \frac{1}{m}\sum_{j=1}^{m} \Big( (-1)^{\tilde{U}_{i}} (2 + \xi) - \xi \Big) \hat{\ell}\big(\hat{\psi}\circ\phi^{(j)}(x^{(j)}_{i,I}),y_{i,I}\big)  \geq  t \Bigg) \\
    = & \mathbb{P}\Bigg(\exp\Bigg(\beta\sum_{i=1}^{n} \sum_{j=1}^{m} \Big( (-1)^{\tilde{U}_{i}} (2 + \xi) - \xi \Big) \hat{\ell}\big(\hat{\psi}\circ\phi^{(j)}(x^{(j)}_{i,I}),y_{i,I}\big) \Bigg)  \geq e^{2\beta nmt} \Bigg) \\
    \leq & e^{-2\beta nmt} \mathbb{E}_{\tilde{U}} \Bigg[\exp\Bigg(\beta\sum_{i=1}^{n} \sum_{j=1}^{m} \Big( (-1)^{\tilde{U}_{i}} (2 + \xi) - \xi \Big) \hat{\ell}\big(\hat{\psi}\circ\phi^{(j)}(x^{(j)}_{i,I}),y_{i,I}\big) \Bigg)  \Bigg] \\
    = & e^{-2\beta nmt} \mathbb{E}_{\tilde{U}} \prod_{i,j=1}^{n,m}\Big[\exp\Big(\beta  \Big( (-1)^{\tilde{U}_{i}} (2 + \xi) - \xi \Big) \hat{\ell}\big(\hat{\psi}\circ\phi^{(j)}(x^{(j)}_{i,I}),y_{i,I}\big) \Big)\Big] \\
    = & e^{-2\beta nmt} \prod_{i,j=1}^{n,m} \frac{e^{-2\beta(1+\xi)\hat{\ell}\big(\hat{\psi}\circ\phi^{(j)}(x^{(j)}_{i,I}),y_{i,I}\big)}+ e^{2\beta\hat{\ell}\big(\hat{\psi}\circ\phi^{(j)}(x^{(j)}_{i,I}),y_{i,I}\big)}}{2}.
\end{align*}
We intend to select the values of $\xi$ and $\beta$, such that $e^{-2\beta(1+\xi)\hat{\ell}\big(\hat{\psi}\circ\phi^{(j)}(x^{(j)}_{i,I}),y_{i,I}\big)}+ e^{2\beta\hat{\ell}\big(\hat{\psi}\circ\phi^{(j)}(x^{(j)}_{i,I}),y_{i,I}\big)}\leq 2$ for any $i\in[n]$, $j\in[m]$, and $I\in\{0,1\}$. Notice that $e^{2\beta\hat{\ell}\big(\hat{\psi}\circ\phi^{(j)}(x^{(j)}_{i,I}),y_{i,I}\big)}\leq 2 $ implies that $\beta\leq \frac{\log 2}{2 \hat{\ell}\big(\hat{\psi}\circ\phi^{(j)}(x^{(j)}_{i,I}),y_{i,I}\big)}<1$. Furthermore, it is sufficient to select a large enough $\xi$ that satisfies:
\begin{equation*}
    e^{-2\beta(1+\xi)\hat{\ell}\big(\hat{\psi}\circ\phi^{(j)}(x^{(j)}_{i,I}),y_{i,I}\big)}+ e^{2\beta\hat{\ell}\big(\hat{\psi}\circ\phi^{(j)}(x^{(j)}_{i,I}),y_{i,I}\big)}\leq 2.
\end{equation*}
Solving the above inequality yields $\xi\geq\frac{\log (2-e^{2\beta\hat{\ell}(\hat{\psi}\circ\phi^{(j)}(x^{(j)}_{i,I}),y_{i,I})})}{2\beta\hat{\ell}\big(\hat{\psi}\circ\phi^{(j)}(x^{(j)}_{i,I}),y_{i,I}\big)}-1$. It is easy to validate that this lower bound increases monotonically with the increase of $\hat{\ell}(\hat{\psi}\circ\phi^{(j)}(x^{(j)}_{i,I}),y_{i,I})$. Therefore, we choose the value of $\xi$ by
\begin{equation*}
    \xi \geq\frac{\log (2-e^{2\beta R^{\tilde{s}}_{x,y}})}{2\beta R^{\tilde{s}}_{x,y}}-1,
\end{equation*}
such that $\frac{e^{-2\beta(1+\xi)\hat{\ell}\big(\hat{\psi}\circ\phi^{(j)}(x^{(j)}_{i,I}),y_{i,I}\big)}+ e^{2\beta\hat{\ell}\big(\hat{\psi}\circ\phi^{(j)}(x^{(j)}_{i,I}),y_{i,I}\big)}}{2} \leq 1$.
For any $i\in[n]$, $j\in[m]$, and the following inequality can hold:
\begin{equation}\label{b57}
    \mathbb{P}\Bigg(\frac{1}{2n}\sum_{i=1}^{n} \frac{1}{m}\sum_{j=1}^{m} \Big( (-1)^{\tilde{U}_{i}} (2 + \xi) - \xi \Big) \hat{\ell}\big(\hat{\psi}\circ\phi^{(j)}(x^{(j)}_{i,I}),y_{i,I}\big)  \geq t \Bigg)  \leq  e^{-2\beta nmt}.
\end{equation}
Taking the union bound of (\ref{b57}) over $I\in\{0,1\}$, we have 
\begin{equation}\label{b58}
    \mathbb{P}\Bigg(\sup_{I\in\{0,1\}} \Bigg\{ \frac{1}{2n}\sum_{i=1}^{n} \frac{1}{m}\sum_{j=1}^{m} \Big( (-1)^{\tilde{U}_{i}} (2 + \xi) - \xi \Big) \hat{\ell}\big(\hat{\psi}\circ\phi^{(j)}(x^{(j)}_{i,I}),y_{i,I}\big) \Bigg\}  \geq t \Bigg)  \leq  2e^{-2\beta nmt}.
\end{equation}
Putting (\ref{b58}) back into (\ref{b56}), we get that 
\begin{align}
    &\mathbb{P}\Bigg(\frac{1}{n}\sum_{i=1}^{n} \frac{1}{m}\sum_{j=1}^{m} \Big(\hat{\ell}(\hat{\psi}\circ\phi^{(j)}(x^{(j)}_{i,1-\tilde{U}_i}),y_{i,1-\tilde{U}_{i}}) - (1+\xi)\hat{\ell}\big(\hat{\psi}\circ\phi^{(j)}(x^{(j)}_{i,\tilde{U}_i}),y_{i,\tilde{U}_i}\big)\Big) \geq t \Bigg) \nonumber\\
    \leq &\inf_{\gamma\in(0,1)}  2e^{-2\beta \gamma nmt} + 2e^{-2\beta (1-\gamma)nmt} \nonumber\\
    = & 2e^{-2\beta nmt} + 2e^{-2\beta nmt} = 4e^{-2\beta nmt}. \label{b59}
\end{align}
By taking $\delta$ as the RHS of the above inequality, we have for any $\delta>0$, with probability at least $1-\delta$,
\begin{equation}
    \frac{1}{n}\sum_{i=1}^{n} \frac{1}{m}\sum_{j=1}^{m} \Big(\hat{\ell}(\hat{\psi}\circ\phi^{(j)}(x^{(j)}_{i,1-\tilde{U}_i}),y_{i,1-\tilde{U}_{i}}) - (1+\xi)\hat{\ell}\big(\hat{\psi}\circ\phi^{(j)}(x^{(j)}_{i,\tilde{U}_i}),y_{i,\tilde{U}_i}\big)\Big) \leq \frac{\log (4/\delta)}{nm\beta}.
\end{equation}
From Lemma \ref{lemma0} and \ref{typicalset},  we know that for any $\delta>0$
\begin{align}
     \mathbb{P}((X,Z)\notin  \mathcal{Z}^{x}_{y,\gamma})\leq \frac{\gamma}{\sqrt{nm}}, &\quad  \vert  \mathcal{Z}^{x}_{y,\gamma}\vert \leq \exp\Big(H(Z_y) +  c_{\phi} \sqrt{\frac{d\log(\sqrt{nm}/\gamma)}{2}} \Big), \label{unio11}
    \\
    \mathbb{P}(\phi\notin \Phi_\epsilon) \leq \delta, &\quad \vert \Phi_\epsilon\vert \leq \exp\Big(H_{1-\lambda}(\phi)+\frac{1}{\lambda}\log(\frac{1}{\delta})\Big). \label{unio22}
\end{align}
Taking the union bound over every $y\in\mathcal{Y}$, $(x,z)\in \mathcal{Z}^{x}_{\gamma}$ and $\phi\in\Phi_\epsilon$, we then have for any $\delta>0$, with probability at least $1-\delta$, the following holds
\begin{equation}
    \frac{1}{n}\sum_{i=1}^{n} \frac{1}{m}\sum_{j=1}^{m} \Big(\hat{\ell}(\hat{\psi}\circ\phi^{(j)}(x^{(j)}_{i,1-\tilde{U}_i}),y_{i,1-\tilde{U}_{i}}) - (1+\xi)\hat{\ell}\big(\hat{\psi}\circ\phi^{(j)}(x^{(j)}_{i,\tilde{U}_i}),y_{i,\tilde{U}_i}\big)\Big) \leq \frac{\log (4 \vert \mathcal{Z}^{x}_{\gamma}\vert \vert \Phi_\epsilon\vert \vert \mathcal{Y}\vert/\delta)}{nm\beta}.
\end{equation}
By substituting (\ref{unio11}) and (\ref{unio22}) into the above inequality, we have for any $\delta>0$, with probability at least $1-\delta$,
\begin{align*}
    &\frac{1}{n}\sum_{i=1}^{n} \frac{1}{m}\sum_{j=1}^{m} \Big(\hat{\ell}(\hat{\psi}\circ\phi^{(j)}(x^{(j)}_{i,1-\tilde{U}_i}),y_{i,1-\tilde{U}_{i}}) - (1+\xi)\hat{\ell}\big(\hat{\psi}\circ\phi^{(j)}(x^{(j)}_{i,\tilde{U}_i}),y_{i,\tilde{U}_i}\big)\Big) \\
    \leq & \frac{\log(\vert \mathcal{Z}^{x}_{\gamma}\vert )+ \log(\vert \Phi_\epsilon\vert) +\log (4 \vert \mathcal{Y}\vert/ \delta)}{nm\beta} \\
    \leq & \frac{H(Z_y) +  c_{\phi} \sqrt{\frac{d\log(\sqrt{nm}/\gamma)}{2}} + H_{1-\lambda}(\phi)+\frac{1}{\lambda}\log(\frac{1}{\delta}) +\log (\frac{4 \vert \mathcal{Y}\vert}{\delta})}{nm\beta}\\
    \leq & \sum_{y\in\mathcal{Y}}\mathbb{P}(Y=y)\frac{H(Z_y) +  c_{\phi} \sqrt{\frac{d\log(\sqrt{nm}/\gamma)}{2}} + H_{1-\lambda}(\phi)+\frac{1}{\lambda}\log(\frac{1}{\delta}) +\log (\frac{4 \vert \mathcal{Y}\vert}{\delta})}{nm\beta}\\
    \leq & \frac{H(Z|Y) +  c_{\phi} \sqrt{\frac{d\log(\sqrt{nm}/\gamma)}{2}} + H_{1-\lambda}(\phi)+\frac{1}{\lambda}\log(\frac{1}{\delta}) +\log (\frac{4 \vert \mathcal{Y}\vert}{\delta})}{nm\beta}
\end{align*}
Putting (\ref{b51}) back into the above inequality, we have 
\begin{align*}
    &\frac{1}{n}\sum_{i=1}^{n} \frac{1}{m}\sum_{j=1}^{m} \Big(\hat{\ell}(\hat{\psi}\circ\phi^{(j)}(x^{(j)}_{i,1-\tilde{U}_i}),y_{i,1-\tilde{U}_{i}}) - (1+\xi)\hat{\ell}\big(\hat{\psi}\circ\phi^{(j)}(x^{(j)}_{i,\tilde{U}_i}),y_{i,\tilde{U}_i}\big)\Big) \\
    \leq & \frac{\sum_{j=1}^{m} I(X^{(j)};C,U^{(j)}|Y) +  c_{\phi} \sqrt{\frac{d\log(\sqrt{nm}/\gamma)}{2}} + H_{1-\lambda}(\phi)+\frac{1}{\lambda}\log(\frac{1}{\delta}) +\log (\frac{4 \vert \mathcal{Y}\vert}{\delta}) + H(Z|X^{(1)},Y)}{nm\beta},
\end{align*}
which implies that 
\begin{align*}
    \overline{\mathrm{gen}}_{cls} =
    &\frac{1}{n}\sum_{i=1}^{n} \frac{1}{m}\sum_{j=1}^{m} \Big(\hat{\ell}(\hat{\psi}\circ\phi^{(j)}(x^{(j)}_{i,1-\tilde{U}_i}),y_{i,1-\tilde{U}_{i}}) - (1+\xi)\hat{\ell}\big(\hat{\psi}\circ\phi^{(j)}(x^{(j)}_{i,\tilde{U}_i}),y_{i,\tilde{U}_i}\big) + \xi\hat{\ell}\big(\hat{\psi}\circ\phi^{(j)}(x^{(j)}_{i,\tilde{U}_i}),y_{i,\tilde{U}_i}\big) \Big) \\
    \leq & \xi \frac{1}{n}\sum_{i=1}^{n} \frac{1}{m}\sum_{j=1}^{m}\hat{\ell}\big(\hat{\psi}\circ\phi^{(j)}(x^{(j)}_{i,\tilde{U}_i}),y_{i,\tilde{U}_i}\big) \\
    &+  \frac{\sum_{j=1}^{m} I(X^{(j)};C,U^{(j)}|Y) +  c_{\phi} \sqrt{\frac{d\log(\sqrt{nm}/\gamma)}{2}} + H_{1-\lambda}(\phi)+\frac{1}{\lambda}\log(\frac{1}{\delta}) +\log (\frac{4 \vert \mathcal{Y}\vert}{\delta}) + H(Z|X^{(1)},Y)}{nm\beta}\\
    = & \xi \widehat{L}_{cls}+  \frac{\sum_{j=1}^{m} I(X^{(j)};C,U^{(j)}|Y) +  c_{\phi} \sqrt{\frac{d\log(\sqrt{nm}/\gamma)}{2}} + H_{1-\lambda}(\phi)+\frac{1}{\lambda}\log(\frac{1}{\delta}) +\log (\frac{4 \vert \mathcal{Y}\vert}{\delta}) + H(Z|X^{(1)},Y)}{nm\beta}.
\end{align*}
This completes the proof.

\end{proof}


\end{document}